\documentclass[Afour,sageh,times]{sagej}
\usepackage[bookmarks=true,breaklinks]{hyperref}
\usepackage{cleveref}
\usepackage{rigidmath}
\usepackage{amsmath}
\usepackage{siunitx}
\usepackage[dvipsnames]{xcolor}
\usepackage{graphicx}
\PassOptionsToPackage{hyphens}{url}
\usepackage[hyphens]{url}
\usepackage{comment}
\usepackage{ragged2e}
\usepackage{caption}
\usepackage{subcaption}
\usepackage[ruled,vlined,linesnumbered]{algorithm2e}
\usepackage{braket}
\usepackage{tabularx,booktabs}
\usepackage{float}

\begin{document}

\title{Set-Valued Rigid-Body Dynamics for Simultaneous, Inelastic, Frictional Impacts}

\author{Mathew Halm\affilnum{1} and Michael Posa\affilnum{1}}

\runninghead{Halm and Posa}

\affiliation{\affilnum{1}GRASP Laboratory, University of Pennsylvania}

\corrauth{Mathew Halm \\
GRASP Laboratory \\
University of Pennsylvania \\
Philadelphia, PA, 19104
}

\email{mhalm@seas.upenn.edu}


\begin{abstract}
Robotic manipulation and locomotion often entail nearly-simultaneous collisions---such as heel and toe strikes during a foot step---with outcomes that are extremely sensitive to the order in which impacts occur.
Robotic simulators and state estimation commonly lack the fidelity and accuracy to predict this ordering, and instead pick one with a heuristic.
This discrepancy degrades performance when model-based controllers and policies learned in simulation are placed on a real robot.
We reconcile this issue with a \textit{set-valued} rigid-body model which generates a broad set of outcomes to simultaneous frictional impacts with any impact ordering.
We first extend Routh's impact model to multiple impacts by reformulating it as a differential inclusion (DI), and show that any solution will resolve all impacts in finite time.
By considering time as a state, we embed this model into another DI which captures the continuous-time evolution of rigid body dynamics, and guarantee existence of solutions.
We finally cast simulation of simultaneous impacts as a linear complementarity problem (LCP), and develop an algorithm for tight approximation of the post-impact velocity set with probabilistic guarantees.
We demonstrate our approach on several examples drawn from manipulation and legged locomotion, and compare the predictions to other models of rigid and compliant collisions.
\end{abstract}

\maketitle


\section{Introduction}
\label{section:introduction}
\begin{figure*}[h]
    \center
     \begin{subfigure}[b]{.16\textwidth}
        \includegraphics[width=.98\hsize]{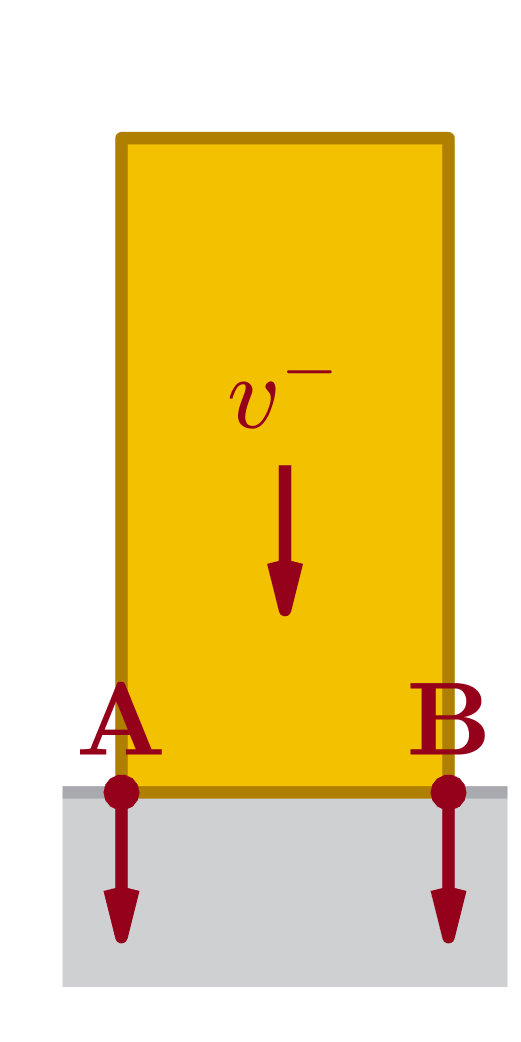}
        \centering
        \caption{\label{fig:phone_cartoon}Initial condition\\ (pre-impact)}
    \end{subfigure}
    \begin{subfigure}[b]{.16\textwidth}
        \includegraphics[width=.98\hsize]{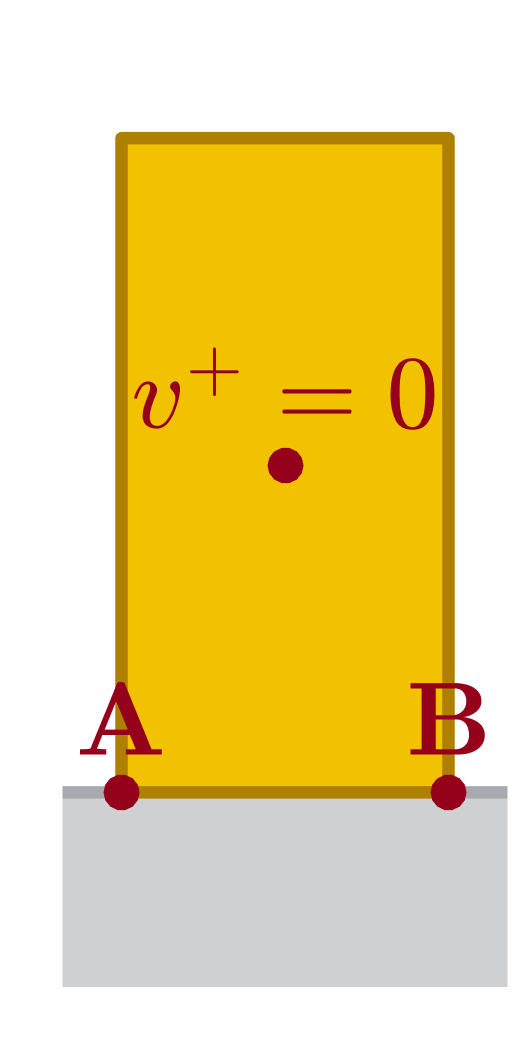}
        \centering
        \caption{\label{fig:phone_example_1}Symmetric impact \\ (post-impact)}
    \end{subfigure}
    \begin{subfigure}[b]{.32\textwidth}
        \includegraphics[width=.49\hsize]{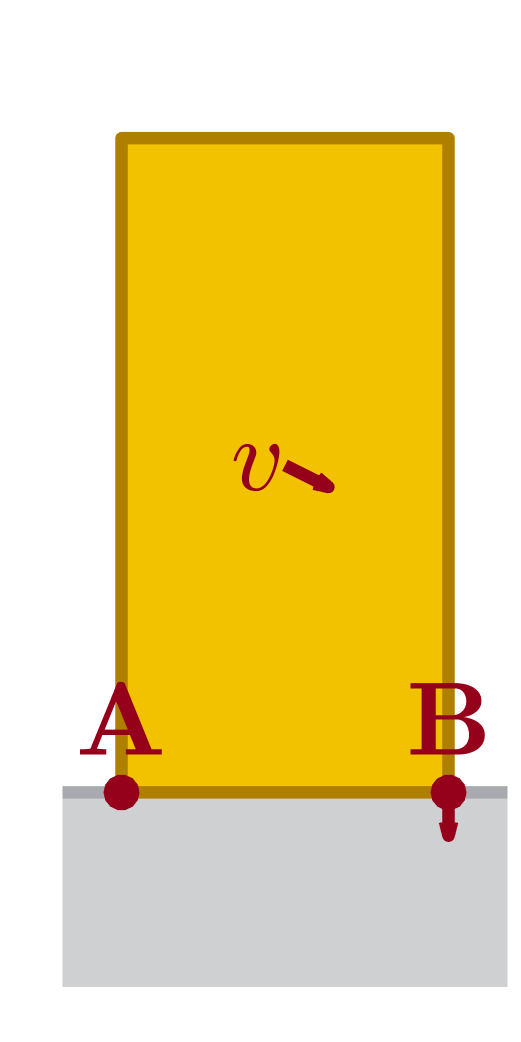}
        \includegraphics[width=.49\hsize]{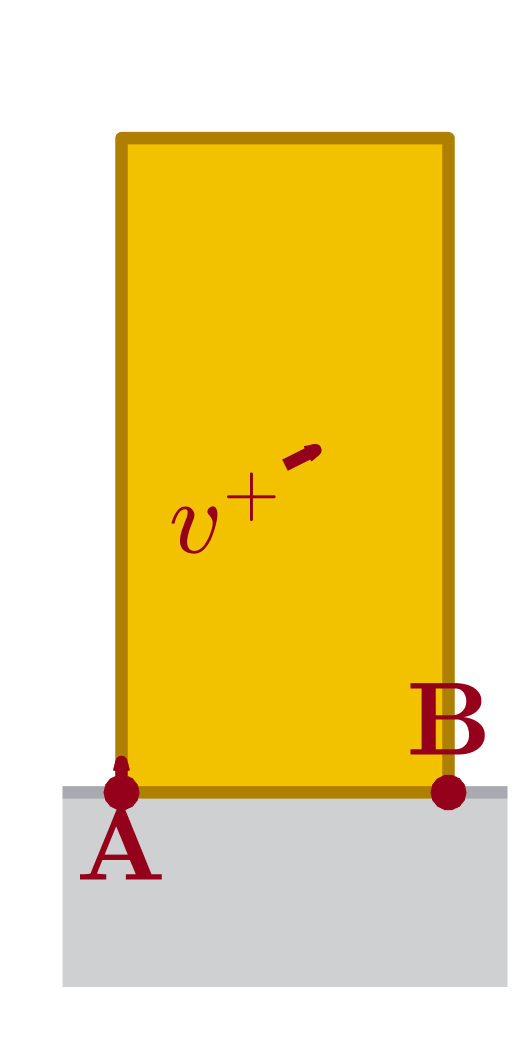}
        \centering
        \caption{\label{fig:phone_example_2}A-then-B sequential impacts\\ (mid-, then post-impact)}
    \end{subfigure}
    \begin{subfigure}[b]{.32\textwidth}
        \includegraphics[width=.49\hsize]{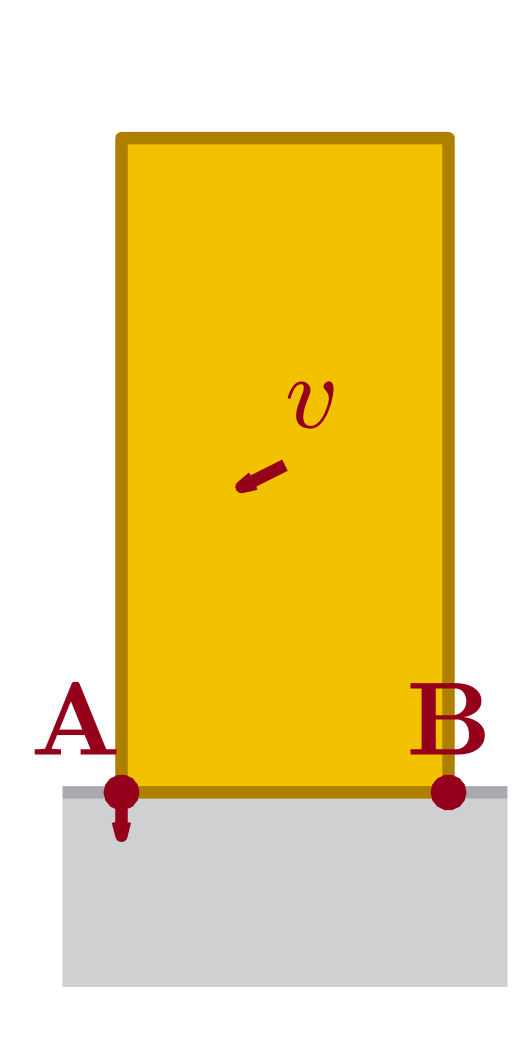}
        \includegraphics[width=.49\hsize]{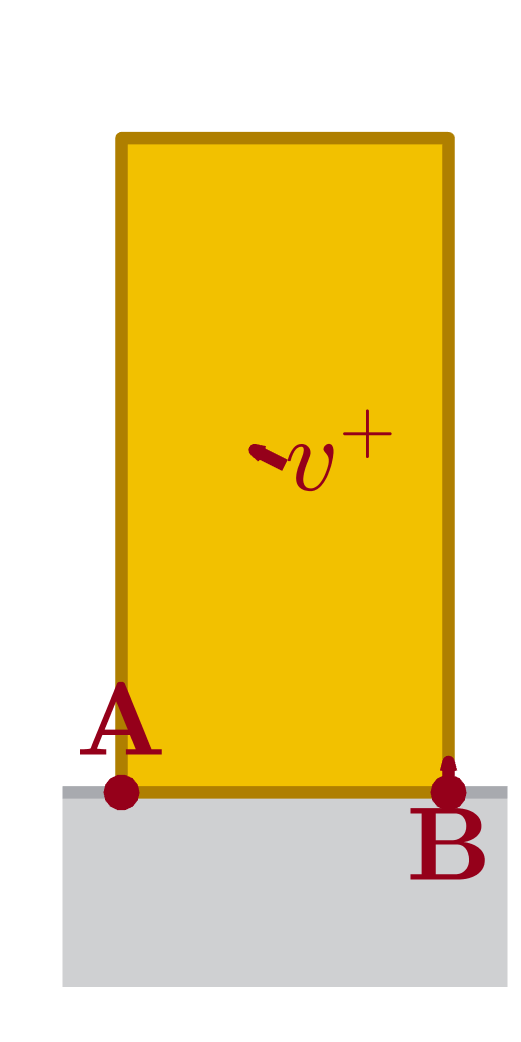}
        \centering
        \caption{\label{fig:phone_example_3}B-then-A sequential impacts\\ (mid-, then post-impact)}
    \end{subfigure}
    
    \caption{Differing outcomes under differing impulse orderings in instantaneous, inelastic impact models are considered for a rocking block (yellow), which collides with flat ground (gray) at 2 corners $\vect A, \vect B$(\subref*{fig:phone_cartoon}); additional details provided in \Cref{subsec:impactbackground} and \Cref{adx:exampledetails}.
        (\subref*{fig:phone_example_1}) One perspective is that impacts at $\vect A$ and $\vect B$ should be resolved simultaneously. Applying the inelastic variant of \citet{Anitescu97} (see \Cref{eq:AnitestcuDiscreteFormulation}) for example results in the block coming to rest.
        (\subref*{fig:phone_example_2}) Other models process impacts one-at-a-time, e.g. \citet{Ivanov1995}.
        With sufficient friction, any impact at a single corner can stick.
        If $\vect A$'s impact is processed first, the block pivots counter-clockwise, necessitating a second impact at $\vect B$, which then causes the block to pivot about $\vect B$ with $\vect A$ lifting off the ground.
        If the impacts are instead ordered $\vect B$-then-$\vect A$, then by symmetry is $\vect A$ pivots and $\vect B$ lifts off. }
    \label{fig:phone}
\end{figure*}
Imperfect but useful physical models have long enabled improvements in planning and control of robotic locomotion and manipulation.
However, the shift from slow, simple motion in tightly-controlled laboratories to dynamic, complex, real-world tasks has dramatically increased accuracy requirements and decreased calibration data availability for these models.
As a result, model inaccuracy 
has become a common bottleneck in developing modern machine learning and mechanics-based methods alike;
in particular, inaccurate prediction of collisions among robots and their surroundings is a longstanding failure of robotics models, especially when multiple impacts happen simultaneously or in quick succession \citep{Ibarz2021,Wensing2023}.

From a mechanical perspective, these failures arise in part from inherent unpredictability of simultaneous collisions.
While some robotics systems and environments are intentionally soft (e.g. cloth manipulation), many locomotion and manipulation tasks inherently involve contact between nearly-rigid components of robots and their environment \citep{Wieber2016,Kemp2007}.
When such objects collide, materials deform on an imperceptibly-small spatial and temporal scale to prevent interpenetration, inducing extreme sensitivity in their motion.
Even small changes in initial conditions and material properties generate large changes in real-world outcomes; accordingly, small errors in state estimation and identification produce large prediction error \citep{Ibarz2021,Chatterjee1997}.
A familiar occurrence of this sensitivity is the unpredictability of billiards breaks \citep{Wang2015} and dice rolls, though even a simple rectangular block impacting flat ground (Figure \ref{fig:phone}) is difficult to model \citep{Housner1963,Zhang2001,Lygeros2003,Yilmaz2009}.
Unfortunately, sensitive, simultaneous impacts regularly occur in robotics (see Section \ref{subsection:motivatingexamples}).

This sensitivity is highly dependent the rapid ordering or sequencing of impact forces between the various colliding bodies \citep{Wang1992, Hurmuzlu1994, Chatterjee1999,  Ivanov1995, Smith2012, Uchida2015}.
In reality, this ordering emerges from material properties and deformation dynamics \citep{Chatterjee1999}, which are generally not tractable to fully identify or simulate in real-world robotics scenarios.
Instead, robotics models typically make a rigid-body assumption, a tractable but coarse approximation of contact mechanics in which objects do not deform.
Such models inherently do not fully capture true impact physics; instead, they typically select a single outcome according to some physically-principled constraints.
There is broad agreement that it is important that solutions to such models should exist over arbitrary time horizons; and that collisions not inject energy into the system \citep{Stewart2000,Stronge90}.
Most models also add additional and seemingly well-motivated constraints in the pursuit of uniqueness of solutions, such as maximum dissipation \citep{Drumwright2010}, minimum potential energy \citep{Uchida2015}, symmetry \citep{Smith2012}, and velocity-based complementarity \citep{Anitescu97}.
Additionally some models have a handful of non-unique solutions, but rely on a numerical solvers which may be biased toward a particular solution \citep{Anitescu97,Stewart1996a,Remy2017}.
However, differing constraints inevitably lead to disagreeing or unrealistic predictions \citep{Remy2017,Fazeli2017a}, and unique outcomes do not reflect the large uncertainty generated from the practically-unknowable sequencing of impacts.
Under restrictions on the systems and mechanics involved, such as massless limbs and no kinetic friction, such models may lead to useful, accurate modeling of robotic systems and tasks \citep{Johnson2016a,Burden2016}.
As both simple and complex robotics systems violate these assumptions \citep{Remy2017,Fazeli2017a}, it is still important to investigate principled modeling approaches that faithfully represent such systems.

In the examples we discuss in \Cref{sec:examples}, we find that the discrepancies between and within models can be significantly large.
This may be particularly problematic for model-based controllers which have built around and are fragile to deviations from a single, expected behavior \citep{Wensing2023}, such as learned policies trained on a single set of settings in a single simulator \citep{Peng2018}, or tracking a dynamically-feasible trajectory of a particular, approximate model \citep{Yang2023}.
This work takes a fundamentally different perspective, in which we propose the development of \textit{set-valued} rigid-body models that attempt to generate all physically-reasonable outcomes, particularly by capturing the effects of arbitrary ordering of impacts.
Though some predictions from such a model may not ultimately occur, controllers guaranteed to stabilize the model---and learned policies trained on the model's predictions---are well-positioned to perform reliably in the real world.

While non-unique predictions through randomly-sequenced individual impacts has existed conceptually for decades \citep{Ivanov1995}, such methods do not capture the subtleties of partially-concurrent impacts \citep{Chatterjee1998TwoInterpretations}, and feasible computation of the entire set of possibilities has remained an open problem \citep{Stewart2000}.
In the domain of inelastic impacts, we tackle both of these issues by developing a differential impact model which allows impacts to resolve at arbitrary relative rates, first conceptually explored in \citet{Posa16}.
This construction is similar mathematically to other methods, including Darboux-Keller \citep{Keller1986} and LZB \citep{Nguyen2018} approaches, in that it extends Routh's original method for inelastic impact \citep{Routh91}; such extensions have so far however been focused on producing a single outcome when well-identified material properties are available \citep{Nguyen2018}.
We also find that intentionally permitting many different behaviors enables proofs of existence of well-behaved solutions under exceptionally few assumptions---ones which terminate a single impact in both continuous and discrete domains; and continuous-time solutions incorporating both movement under sustained contact as well as instantaneous impacts.
In particular, in the latter case we guarantee solutions through well-known pathological scenrios of rigid-body motion, such as Painlev\'e paradoxes \citep{Stewart2000} and Zeno behaviors \citep{Ames2006}.
We will pair these theoretical advances with practical algorithms for approximation of the set of outcomes to individual impacts, which can be readily integrated with event-based simulation schemes.

This work extends our previous work \citep{Halm2019}, in which we first extended Routh's impact method to set-valued simultaneous frictional impacts. This paper supplements the scope of this work with the following:
\begin{itemize}
	\item In \Cref{section:model}, we provide a simplified theoretical analysis of our set-valued impact model (\Cref{eq:multicontact}). We prove that solutions to this model always exist (\Cref{lem:NonEmptyClosure}), and that each solution is physically reasonable in that it dissipates kinetic energy (\Cref{lem:dissipate}) and terminates the impact process over a finite duration (\Cref{thm:nondegeneratedissipation}). We include new motivating examples highlighting the inconsistencies between existing models of simultaneous impact in \Cref{subsection:motivatingexamples}.
	\item In \Cref{section:continuous_time_model}, we unify set-valued impacts and continuous-time evolution into a single model (\Cref{eq:ContinuousTimeModel}).
	We prove that solutions to this model as well always exist (\Cref{thm:ContinuousTimeSolutions}) over arbitrary time horizons (\Cref{thm:ContinuousTimeMinimumAdvancement} and \Cref{coro:AggregateAdvancement}). We illuminate via example how pathological scenarios including Painlev\'e paradoxes and Zeno behaviors are captured by the model.
	\item In Section \ref{section:simulation}, we formulate an implicit numerical integration scheme for impact model, encoded as a linear complementarity problem (LCP) (\Cref{eq:SimulationLCP}).
	We demonstrate that each integration step LCP is solvable (\Cref{thm:TimesteppingSolutionExistence}) and dissipates kinetic energy (\Cref{thm:TimesteppingDissipation}).
We provide algorithms with probabilistic bounds on computation time for both sampling from (Algorithm \ref{alg:ImpactSimulation} and \Cref{thm:ExponentialSimulationDecay}) and global approximation of (Algorithm \ref{alg:Approximate} and \Cref{thm:VelocitySetApproximation}) the feasible post-impact velocity set of a simultaneous impact event. In Section \ref{sec:examples}, we apply our model to several examples from robotic locomotion and manipulation.
\end{itemize}
\section{Background}

We now introduce notation (summarized in \Cref{table:GenericMathNotation,table:ContactDynamicsTerms}) for and review the mathematics underpinning continuous-time rigid-body dynamics with contact. Well-versed readers may skip to Section \ref{section:model} and use this section and the appendix as required.
 We use several set-, matrix-, and vector-valued operations and constants, the most common of which are listed in \Cref{table:GenericMathNotation}.

 We begin with mathematical foundations: sampling-based set approximation (\Cref{subsec:setapproximationviasampling}), set-valued maps (\Cref{section:SetValuedMaps}), differential inclusions (\Cref{subsec:differentialinclusions}), and linear complementarity problems (\Cref{subsec:lcps}).

 We conclude with an overview of rigid-body dynamics under sustained contact (\Cref{subsec:continuoustimedynamics}), 
 impacts (\Cref{subsec:impactbackground}); and initial value problems that combone both of these behaviors (\Cref{subsubsec:IVPs}); a listing of the associated system terms is in \Cref{table:ContactDynamicsTerms}.
 
 For notational brevity, we frequently write a singleton set $\Braces{a}$ without braces (e.g. $a + B$ is the Minkowski sum of $\Braces{a}$ and $B$) and suppress dynamics terms' inputs whenever clear (i.e. we write $\Mass$ instead of $\Mass(\Configuration)$).

\begin{table}
	\caption{Frequently-used constants and operations on sets $A, A_i, A', B$, scalars $c$, vectors $\vect v, \vect w$, matrices $\vect M, \vect N$, and functions $f : A \to B$, $\vect g(t) : \Real \to \Real^n$, $D: A \to \PowerSet{B}$. For brevity, we frequently write a singleton $\Braces{a}$ without braces.}
	\label{table:GenericMathNotation}
	\centering
	\setlength\tabcolsep{0.7mm}
	\begin{tabular}{c l}
		\toprule
		Expression & Meaning \\

		\midrule
		$A^c$ & complement of $A$ \\
		$\Interior(A)$ & interior of $A$ \\
		$\Closure(A)$ & closure of $A$ \\ 
		$\Hull(A)$ & convex hull of $A$ \\ 
		$\PowerSet{A}$ & power set $\Braces{A' : A' \subseteq A}$ \\
		$f : A \to B$ & $f$ maps $a \in A$ to $f(a) \in B$\\
		$D : A \to \PowerSet{B}$ & $D$ maps $a \in A$ to $D(a) \subseteq B$\\
		$f(A')$ & image of $A'$, $\Braces{f(a') : a' \in A'}$\\
		$D(A')$ & image of $A'$, $\cup_{a' \in A'} D(a')$\\
		$\vect MA$ & scaled set, $\Braces{ \vect M a: a \in A}$ \\
		$-A$ & $(-1) A$ \\
		$A + B$ & Minkowski sum $\Braces{a+b : a\in A, b\in B }$ \\
		$A - B$ & Minkowski sum of $A$ and $-B$ \\
		$[A_1;\; \dots \; A_k]$\hspace{-3mm} & Cartesian product $A_1 \times \dots \times A_k$\hspace{-2mm}\\
		$\dot{\vect g}(s)$ & total Lebesgue derivative $\TotalDiff{}{s}\vect g$\\
		$\vect v_i$ & $i$th element of $\vect v$\\
		$\sigma_{max}(\vect M)$ & maximum singular value of $\vect M$\\
		$\sigma_{min}(\vect M)$ & minimum singular value of $\vect M$\\
		$\vect M \succ \vect N$ & $\vect M - \vect N$ is positive definite\\
		$\vect M \succeq \vect N$ & $\vect M - \vect N$ is pos. semi-definite\\
		$\vect v > \vect w$ & $\vect v_i > \vect w_i$ for each $i$\\
		$\vect v \geq \vect w$ & $\vect v_i \geq \vect w_i$ for each $i$\\
		$A > 0$ & each element of $A$ is positive \\
		$A \geq 0$ & each element of $A$ is non-negative\hspace{-5mm} \\
		$\Norm{\vect A}_{F}$ & Frobenius norm of $A$\\
		$\Norm{\vect v}_p$ & $l_p$ norm of $\vect v$, $\Parentheses{\sum_i |\vect v_i|^p}^{\frac{1}{p}}$\\
		$\Norm{\vect v}_{\vect M}$ & $\vect M$-norm $\sqrt{\Velocity^T \vect M \Velocity}$, $\vect M \succeq \ZeroVector$\\
		$\Direction{\vect v}$ & unit direction, $\frac{\vect v}{\TwoNorm{\vect v}}$, of $\vect v \neq \ZeroVector$\\
		$\Ball[c]$ & $c$-radius ball $\Braces{ \vect v: \TwoNorm{\vect v} < c}$\\
		$\OneVector$ & matrix/vector of all $1$'s\\
		$\ZeroVector$ & matrix/vector of all $0$'s\\
		$\Real^{n+}$ & $\Braces{\vect v \in \Real^n, \vect v \geq \ZeroVector}$\\
		\bottomrule
	\end{tabular}
\end{table}
\begin{table}
	\centering
	\caption{Dynamics terms for rigid bodies and frictional contact. Some terms are written with the dependence on their inputs suppressed.}
	\label{table:ContactDynamicsTerms}
	\centering
	\setlength\tabcolsep{0.7mm}
	\begin{tabular}{c c l}
		\toprule
		Term & Space & Meaning \\

		\midrule
		$\Configurations$ & $\Natural$ & number of configuration variables \\
		$\Velocities $ & $\Natural$  & number of generalized velocities \\
		$\States$ & $\Natural$  & number of states $\Configurations + \Velocities$ \\
		$\Contacts $ & $\Natural$  & number of contacts \\
		$t$ & $\Real$ & time \\
		$\Configuration$ & $\ConfigurationSpace$ & robot/environment configuration \\
		$\Velocity$ & $\VelocitySpace$ & robot/environment velocity \\
		$\State$ & $\StateSpace$ & robot/environment state \\
		$\bar \State$ & $\StateTimeSpace$ & time-augmented state \eqref{eq:AugmentedStateDefinition} \\
		$\Input$ & $\Real^{n_u}$ & robot/environment input forces \\
		$\GeneralizedVelocityJacobian(\Configuration)$ & $\Real^{\Configurations\times\Velocities}$ & generalized velocity Jacobian \eqref{eq:ConfigurationEvolution}\\
		$\Mass(\Configuration)$ &$\Real^{\Velocities\times\Velocities}$& generalized mass-inertia matrix\\
		$\NetForce[s](\State,\Input)$ & $\VelocitySpace$ & non-contact forces \eqref{eq:ManipulatorEquations}\\
		$\KineticEnergy(\Configuration, \Velocity)$ &$\Real$& total kinetic energy \eqref{eq:KineticEnergyDefinition} \\
		$\Jn(\Configuration)$ &$\Real^{\Contacts\times\Velocities}$& normal velocity Jacobian \\
		$\Jt(\Configuration)$ &$\Real^{2\Contacts\times\Velocities}$& tangent velocity Jacobian \\
		$\J(\Configuration)$ &$\Real^{3\Contacts\times\Velocities}$& full contact velocity Jacobian \eqref{eq:BlockJacobianDefinition} \\
		$\NormalForce$ &$\Real^{\Contacts}$& normal forces vector\\
		$\FrictionForce$ &$\Real^{2\Contacts}$& frictional contact forces vector\\
		$\Force$ &$\Real^{3\Contacts}$& full contact forces vector \eqref{eq:BlockForceDefinition} \\
		$\FrictionCoeff[i]$ &$\Real$& $i$th contact Coulomb friction coeff.\\
		$\FrictionCone \Parentheses{\Configuration}$ &$\PowerSet{\VelocitySpace}$& Coulomb friction cone at $\Configuration$ \eqref{eq:FrictionCone} \\
		$\JD$ &$\Real^{\ConeBases\Contacts\times\Velocities}$& linear tangent vel. Jacobian \eqref{eq:ForceDJD} \\
		$\FrictionBasisForce$ &$\Real^{\ConeBases\Contacts}$& linear friction forces vector \eqref{eq:ForceDJD}  \\
		$\bar\J$ &\hspace{-2mm}$\Real^{(\ConeBases+1)\Contacts\times\Velocities}$\hspace{-2mm}& linear velocity Jacobian \eqref{eq:barJbarForce} \\
		$\bar\Force$ &$\Real^{(\ConeBases+1)\Contacts}$& linear contact forces vector \eqref{eq:barJbarForce} \\
		$\LinearFrictionCone \Parentheses{\Configuration}$ &$\PowerSet{\VelocitySpace}$& linear friction cone at $\Configuration$ \eqref{eq:LinearFrictionCone} \\
		$\ContactsSpace$ &$\PowerSet{\Natural}$& set of all contacts \\
		$\ContactSet_A(\Configuration)$ &$\PowerSet{\ContactsSpace}$& active/touching contact set at $\Configuration$ \eqref{eq:ActiveContactsDefinition} \\
		$\ContactSet_P(\Configuration)$ &$\PowerSet{\ContactsSpace}$& penetrating contact set at $\Configuration$ \eqref{eq:PenetratingContactsDefinition} \\
		$\ConfigurationSet_A$ &$\PowerSet{\ConfigurationSpace}$& set of active-contact configurations\\
		$\ConfigurationSet_P$ &$\PowerSet{\ConfigurationSpace}$& set of penetrating configurations\\
		$\bar\StateSet_A$ &$\PowerSet{\StateTimeSpace}$& set of active-contact states\\
		$\bar\StateSet_P$ &$\PowerSet{\StateTimeSpace}$& set of penetrating states\\
		$\ActiveSet (\Configuration) $ &$\PowerSet{\VelocitySpace}$& set of colliding velocities \eqref{eq:ImpactingVelocitiesDefinition} \\
		$\InactiveSet (\Configuration)$ &$\PowerSet{\VelocitySpace}$& set of separating velocities \eqref{eq:SeparatingVelocitiesDefinition} \\
		\bottomrule
	\end{tabular}
\end{table}
\subsection{Mathematical Foundations}
The total derivative of an absolutely continuous function $\vect f(t)$ is denoted $\dot{\vect f}(t)$.
$f : A \to B$ is \textit{Lipschitz continuous} with constant $L$ if for all $a_1$, $a_2$ in $A$, $\TwoNorm{f(a_1) - f(a_2)} \leq L\TwoNorm{a_1 - a_2}$.
An absolutely continuous ${\vect f}(t)$ has this property if $\TwoNorm{\dot{\vect f}(t)} \leq L$ almost everywhere (a.e.).
Furthermore, (partial) compositions of Lipschitz functions are also Lipschitz with constant no more than the product of the composed functions. That is, if $f,g: A \times B \to A$ are two Lipschitz functions with constants $L_f$ and $L_g$, $h(a,b_1,b_2) = f(g(a,b_1),b_2)$ is Lipschitz with constant no more than $L_fL_g$.

We say a function $\DissipationRate(s):\Domain \rightarrow \Closure \Real^+$, is positive definite if it is positive on $\Domain \setminus \Braces{0}$ and $\DissipationRate(0) = 0$.

\subsubsection{Set Approximation via Sampling}\label{subsec:setapproximationviasampling}
Problems in robotics can often be approximately solved with arbitrary-close approximation (up to limitations stemming from machine precision) via stochastic sampling (e.g. planning with RRT* \citep{Karaman2011}).
In Section \ref{section:simulation}, we will use sampling to approximate the set of post-impact velocities corresponding to a pre-impact state with an \textit{$\varepsilon$-net}:
\begin{definition}
	For $\varepsilon \geq 0$, an $\varepsilon$-net of a set $\StateSet$ is a set $\StateSet' \subseteq \StateSet$ such that for each $x \in \StateSet$, $\exists x' \in \StateSet'$ with $\TwoNorm{x-x'} \leq \varepsilon$.
\end{definition}
In the spirit of probabilistic completeness, we will show that, with sufficient samples and ignoring limitations on machine precision, our simulation scheme can approximate this set to arbitrary $\varepsilon$ with arbitrary confidence.
The essential goal is to show that a sufficient quantity of independent and identically distributed samples of a set tends to yield an $\varepsilon$-net of the set with low $\varepsilon$.
In particular, we will be approximating the image of a box under a Lipschitz continuous function via uniform sampling on the input space:
\begin{lemma}[Dense Sampling (Appendix \ref{adx:epsilonnetproof})]
	Let $g(x): \Real^n \to \Real^m$ be Lipschitz with constant $L$. Consider a set of $N$ uniform i.i.d. samples $\StateSet = \Braces{x_1,\dots,x_N}$ from $[0,h]^n$.
	Then $g(\StateSet)$ is an $\varepsilon$-net of $g([0,h]^n)$ with probability at least 
	\begin{align}
		&1 - \frac{(1-\Omega)^N}{\Omega}\,, & \Omega &= \left\lceil \frac{hL\sqrt{n}}{\varepsilon} \right\rceil^{-n}\,.
	\end{align}
	\label{prop:epsilonnet}
\end{lemma}

\subsubsection{Set-Valued Maps}\label{section:SetValuedMaps}
Our mathematical constructions and theoretical results will frequently make use of \textit{set-valued} maps $\DerivativeMap(a): A \to \PowerSet{B}$, which take as input an element $a \in A$ an output a subset of $B'$ of some output space $B$.
As complex operations on the sets involved in such maps are essential to our analysis, some abbreviated notation is required for the readability of our constructions and derivations.
We list these abbreviations as part of \Cref{table:GenericMathNotation}.
Set-valued maps may exhibit properties reminiscent of continuity for single-valued functions. We in particular will make frequent use of an \textit{upper semi-continuity} (u.s.c.) property:
\begin{definition}
	A function $\DerivativeMap : A \to \PowerSet{B}$, where $A \subseteq \Real^{n_A}$, $B \subseteq \Real^{n_B}$, is \textit{upper semi-continuous} if for any input $a$ and neighborhood $B'$ of $ 
	\DerivativeMap (a)$, there exists a neighborhood $A'$ of $a$ with $B' \subseteq \DerivativeMap (A')$. Equivalently, if $B$ is compact, for all convergent sequences $\Sequence{a}{i}$ and $\Sequence{b}{i}$,
	$$b_i \in \DerivativeMap (a_i)\,, \forall i \implies \lim b_i \in \DerivativeMap (\lim a_i)\,.$$
\end{definition}

Similar to continuous functions, there are several useful compositional rules which preserve upper semicontinuity; finite combination of u.s.c. functions by cartesian product, convex hull, composition, union, and addition are all u.s.c. \citep{Aubin1984}.
\subsubsection{Differential Inclusions}\label{subsec:differentialinclusions}
\begin{figure}
		\centering
    \begin{subfigure}{.49 \hsize}
	   \includegraphics[width=\hsize,trim={1mm 0mm 14mm 0mm},clip]{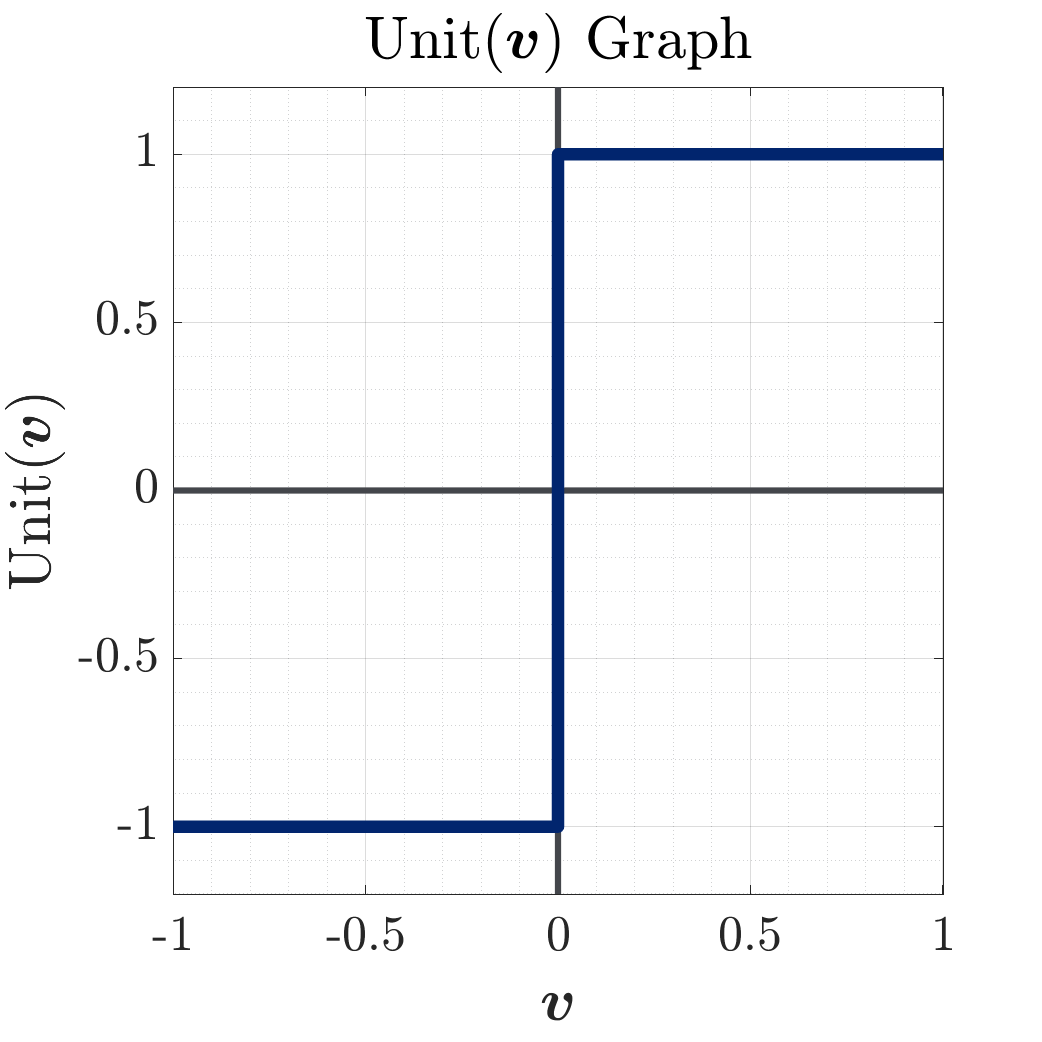}
       \caption{\label{figure_Ua}}
    \end{subfigure}
    \hfill
    \begin{subfigure}{.49 \hsize}
        \includegraphics[width=\hsize,trim={1mm 0mm 14mm 0mm},clip]{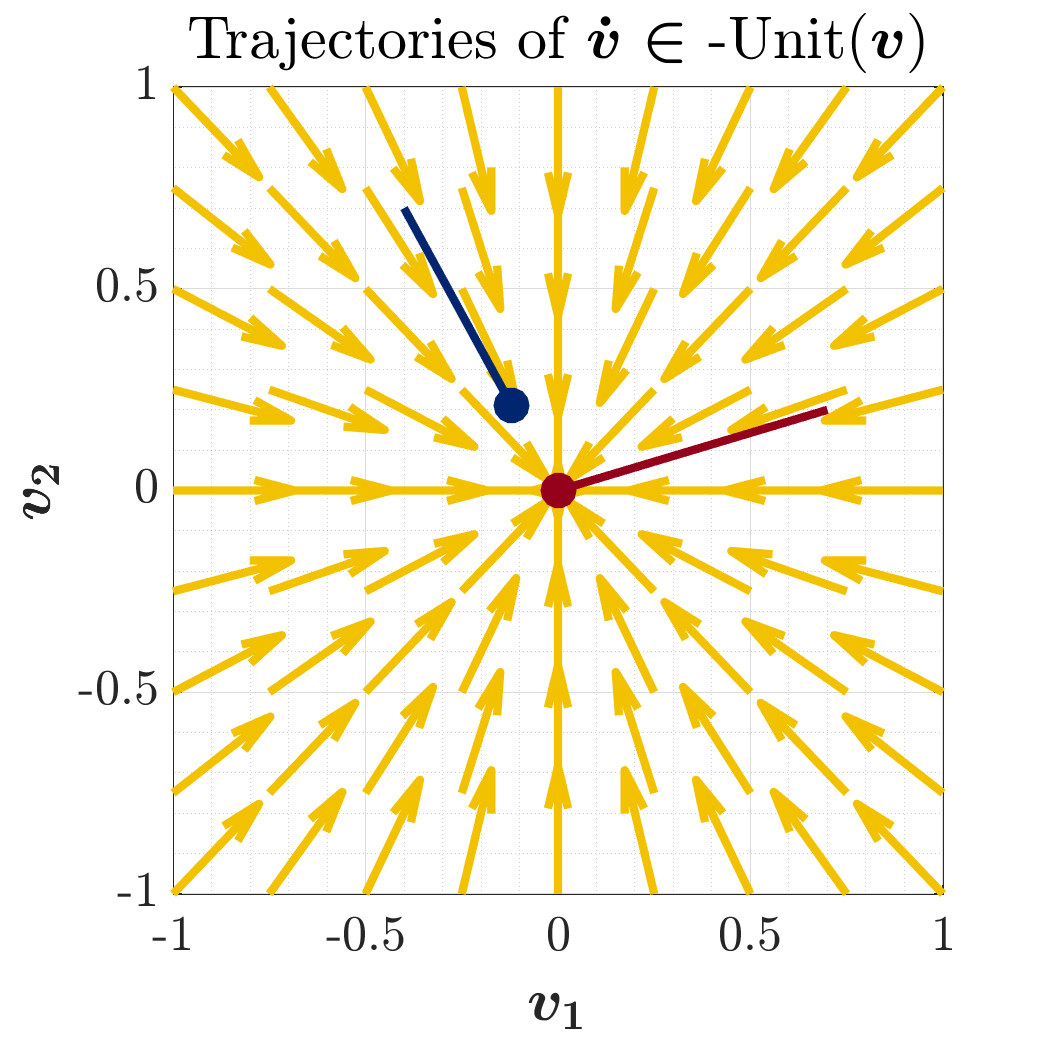}
        \caption{\label{figure_Ub}}
    \end{subfigure}
	\caption{Illustration of a set-valued function and corresponding differential inclusion. (\subref*{figure_Ua}) Graph of $\Unit(\Velocity)$, the set-valued unit direction for  dimension $n=1$. $\Unit(\Velocity)$ is continuous at $\Velocity \neq 0$. At $\ZeroVector$, $\Unit$ takes the value $\Brackets{-1,1}$, which contains a continuous extension of $\Direction{\Velocity}$ from both the left $(-1)$ and the right $(+1)$, so that $\Unit$ is u.s.c.. (\subref*{figure_Ub}) Flow field (yellow) of the solutions (blue, red) to the inclusion $\dot \Velocity \in -\Unit(\Velocity)$ for the $2$-dimensional unit direction.}
	\label{fig:unit}
\end{figure}
We will later see that in continuous time, the dynamics of rigid bodies under frictional contact present complexities that Ordinary Differential Equation (ODE) formulations cannot capture, as multiple outcomes that obey the constituent laws of contact may exist (\textit{non-unique} behaviors) \citep{Stewart2000}.
It is then useful to define an object that, unlike ODEs, allows for the derivative at each state to lie in a set of possible values
\begin{equation}
\dot \State \in \DerivativeMap (\State)\,.
\label{eq:differentialinclusion}
\end{equation}
As the set-valued map $\DerivativeMap (\State)$ associated with friction may not be continuous, conditions for a function $\State(t)$ to solve this \textit{differential inclusion} (DI) are weaker from those for an ODE:
\begin{definition}\label{def:inclusionsolution}
	For a compact interval $[a,b]$, $\State(t) : [a,b] \to \Real^n$ is a solution to the differential inclusion $\dot \State \in \DerivativeMap (\State)$ if $\State(t)$ is absolutely continuous and $\dot \State(t) \in \DerivativeMap (\State(t))$ a.e. on $[a,b]$. Denote the set of such solutions as $\SolutionSet{\DerivativeMap}[[a,b]]$.
\end{definition}
\noindent Solutions to initial value problems for (\ref{eq:differentialinclusion}) are defined similarly:
\begin{definition}
	The set of solutions to $\dot \State(t) \in \DerivativeMap (\State(t))$ with initial condition $\State(a) = \State_0$ over the interval $t \in [a, b]$ are denoted as $\IVP{\DerivativeMap}{\State_0}{[a,b]}$.
\end{definition}
\noindent In Figure \ref{fig:unit}, we consider an example DI
\begin{equation}
	\dot \Velocity \in -\Unit\Parentheses{\Velocity}\,,
\end{equation}
where $\Unit\Parentheses{\State}$ is the set-valued unit direction function
\begin{equation}
	\Unit\Parentheses{\Velocity} = \begin{cases}
		\Braces{\Direction{\Velocity}} & \Velocity \neq \ZeroVector\,, \\
		\Braces{\Velocity' : \TwoNorm{\Velocity'} \leq 1}  & \Velocity = \ZeroVector\,.
	\end{cases}
\end{equation}
The unique solution to the initial value problem starting from $\Velocity(0) = \Velocity_0$ has the form
\begin{equation}
	\Velocity(t) = \max\Parentheses{\TwoNorm{\Velocity_0} - t, 0}\Direction{\Velocity}_0 \,.
\end{equation}
This solution is non-differentiable at $t = \TwoNorm{\Velocity_0}$ and thus is not a solution of any ODE.
In general, non-emptiness and regularity of the initial value problem depends on the structure of $\DerivativeMap (\State)$; fortunately, we will later show that solution sets for frictional dynamics are well-behaved due to their upper semi-continuous (u.s.c.) structure:
\begin{proposition}[\citet{Aubin1984}]\label{prop:closure}
	Let $\State_0 \in \Real^n$ and $[a,b]$ be a compact interval. Suppose $\DerivativeMap (\State)$ is uniformly bounded (i.e. $\DerivativeMap (\State) \subseteq \Ball[c]$ for some $c>0$). If $\DerivativeMap (\State)$ is u.s.c., closed, convex, and non-empty at all $\State$, then $\IVP{\DerivativeMap}{\State_0}{[a,b]}$ is non-empty and u.s.c. in $\State_0$ under uniform convergence.
\end{proposition}
U.s.c. functions have the useful property that they map compact sets to closed sets, and \Cref{prop:closure} immediately and crucially implies that $\SolutionSet{\DerivativeMap}[[a,b]]$ and $\IVP{\DerivativeMap}{\State_0}{[a,b]}$ are non-empty and closed under uniform convergence. The DI in Figure \ref{fig:unit} for example exhibits this structure.

\subsubsection{Linear Complementarity Problems}\label{subsec:lcps}
We will formulate multi-impact simulation as a sequence of linear complementarity problems (LCP's), which have been widely used for frictional contact simulation \citep{Anitescu97,Stewart1996a}. We refer the reader to \citet{Cottle09} for a complete description.

\begin{definition}\label{def:LCPDef}
The linear complementarity problem with parameters $\LCPMatrix \in \Real^{n \times n}$ and $\LCPVector \in \Real^n$ is the constraint satisfaction problem
\begin{align}
& {\text{find}}
& & \LCPVariables \in \Real^n\,,\\
& \text{subject to}
& & \LCPVariables^T \Parentheses{\LCPMatrix \LCPVariables + \LCPVector} = \ZeroVector \,, \label{eq:lcpcompbeg}\\
& & & \LCPVariables, \LCPMatrix \LCPVariables + \LCPVector \geq \ZeroVector \,, \label{eq:lcpcompend}
\end{align}
for which the set of solutions is denoted $\LCP{\LCPMatrix }{\LCPVector}$. \eqref{eq:lcpcompbeg}--\eqref{eq:lcpcompend} are often abbreviated as $\Complementary{\LCPVariables}{\LCPMatrix \LCPVariables + \LCPVector}$. 
\end{definition}

For LCPs related to frictional behavior, $\LCPMatrix$ is often copositive (i.e. $\vect{x}^T \LCPMatrix \vect{x} \geq 0$ for all $\vect{x} \geq \ZeroVector$). This property provides a sufficient condition for LCP feasibility and computability:
\begin{proposition}[\citep{Cottle09}]\label{prop:LCPCopExist}
Let $\LCPVector \in \Real^n$, and let $\LCPMatrix \in \Real^{n \times n}$ be copositive. If $\LCPVector^T\LCP{\LCPMatrix}{\ZeroVector} \geq 0$, then $\LCP{\LCPMatrix}{\LCPVector}$ contains a solution which can be computed in finite time.
\end{proposition}
While solution uniqueness is not guaranteed, if mapping the solution through a matrix $\vect A$ produces uniqueness, it also produces Lipschitz continuity:
\begin{proposition}[\citep{Facchinei2003}]\label{prop:AffineLCPLipschitz}
For all matrices $\LCPMatrix \in \Real^{n \times n}$, $\vect A \in \Real^{m \times n}$, if the function $\vect f(\LCPVector) = \vect A\LCP{\LCPMatrix}{\LCPVector}$ is unique over a convex domain $\Omega \subseteq \Real^n$, it is also Lipschitz on $\Omega$.
\end{proposition}

\subsection{Rigid-Body Dynamics with Friction}
We now describe the mathematics and assumptions of rigid body modeling of multiple articulated-body systems which undergo Coulomb friction and inelastic impacts; notation is summarized in \Cref{table:ContactDynamicsTerms}.

As discussed in \Cref{section:introduction}, both the suitability of rigid-body modeling and the motion that results is dependent on the properties of the materials involved.
While the following sections will specifically outline some narrow, technical assumptions, we first establish three high-level modeling decisions which inform the scope of applicability of our models; our derivations; and our comparisons to the surrounding literature.
\begin{itemize}
	\item \textbf{All bodies are rigid}. We assume that every body deforms negligibly, i.e. bodies' stiffnesses are high enough that the energy input to the system is much lower than the potential energy required to compress objects significantly.
	In this setting, continuous-time evolution under sustained contact can be tracked with a state containing the position, orientation, linear velocity, and angular velocity of a nominal frame affixed to each body; and impacts can be reasonably modeled as instantaneous.
	There are multiple, nuanced interpretations of what can be considered ``negligible'' deformation, especially when concurrent impacts are involved; we refer the reader to \citet{Chatterjee1998TwoInterpretations} for a detailed discussion.
	\item \textbf{Contact forces are dominated by dry friction}, specifically Coulomb's law \citep{Popova2015} described in \Cref{subsec:continuoustimedynamics}. This law is often appropriate e.g. for manipulation of clean objects or locomotion over dry terrain, rather than interaction with viscous or adhesive substances.
	\item \textbf{Impacts are completely inelastic}, in that they dissipate kinetic energy as much as possible.
	Such assumptions are appropriate e.g. for materials which plastically deform under impact; have viscous deformation behavior; or for which the energy is lost to elastic vibrations \citep{Stoianovici1996,Stewart2000}.
	Inelastic impact models been employed effectively in robotics simulation, planning, and control \citep{Wieber2016,Wensing2023}.
	For a single impact, this property characterized by the bodies having no separating velocity post-impact, though there is in general no single accepted rule for sumultaneous impacts \citep{Stewart2000}.
\end{itemize}

\subsubsection{Continuous-time evolution without impacts}\label{subsec:continuoustimedynamics}
Rigid robots contacting rigid objects and environment can be modeled with inputs $\Input$ (e.g. motor torques) and states $\State = [\Configuration; \Velocity] \in \StateSpace$, where $\Configuration \in \ConfigurationSpace$ represents the robot's configuration and object poses.
Though $\Velocity \in \VelocitySpace$ is simply $\TimeDiff{\Configuration}{}$ for some systems, others (e.g. those relating angular velocities and quaternion derivatives) obey
\begin{equation}\label{eq:ConfigurationEvolution}
	\Differential \Configuration = \GeneralizedVelocityJacobian (\Configuration ) \Velocity \Differential t \,,
\end{equation}
for some smooth, bounded, full-column-rank $\GeneralizedVelocityJacobian (\Configuration) \in \Real^{\Configurations \times \Velocities}$ \citep{Tedrake, Castro2020}.
Contact between these bodies is modeled as occurring at up to $\Contacts \in \Natural$ point pairs (for a thorough introduction, see  \citet{Brogliato99} and \citet{Stewart2000}) referred to as the \textit{contacts} $\ContactsSpace = \Braces{1,\dots,\Contacts}$.
Impactless evolution of the system is governed by
\begin{equation}\label{eq:ManipulatorEquations}
	\Mass(\Configuration)\TimeDiff{\Velocity} =\NetForce[s](\State,\Input) + \sum_{i \in \ContactsSpace} \J[i](\Configuration)^T\Force[i]\,.
\end{equation}
Here, the continuous function $\Mass(\Configuration) \succ 0, \Mass \in \Real^{\Velocities \times \Velocities}$ is the generalized inertial matrix, related to the kinetic energy $\KineticEnergy(\Configuration,\Velocity) \in \Real$ by
\begin{equation}\label{eq:KineticEnergyDefinition}
	\KineticEnergy(\Configuration,\Velocity) = \frac{1}{2}\Norm{\Velocity}_{\Mass(\Configuration)}^2 = \frac{1}{2}\Velocity^T\Mass(\Configuration )\Velocity\,.
\end{equation}
By assumption, there exist global $c_1,c_2 > 0$ such that $c_1\Identity \succeq \Mass \succeq c_2\Identity$.
$\NetForce[s]$ aggregates smooth, non-contact forces (e.g. potential, gyroscopic, and input forces as well as Coriolis and centrifugal effects).
For each $i$, $\J[i]^T\Force[i] \in \Real^{\Velocities}$ is the net (generalized) force due to the $i$th contact.
$\J[i]= [\Jn[i]; \Jt[i]] \in \Real^{3 \times \Velocities}$ is the contact Jacobian which maps generalized velocities into Euclidean velocities in the $i$th contact frame normal ($\Jn[i] \in \Real^{\Velocities}$) and tangential ($\Jt[i] \in \Real^{2 \times \Velocities}$) directions.
$\Force[i] = [\NormalForce[i]; \FrictionForce[i]] \in \Real^3$ are the contact-frame normal forces $\NormalForce[i] \in \Real$ and frictional forces $\FrictionForce[i] \in \Real^2$, which are typically dictated by two essential physical laws:
\begin{itemize}
	\item \textbf{Normal complementarity}: The signed distance $\Gap (\Configuration)\in \Real^\Contacts$ captures object geometry as inter-body distances.
Normal forces push bodies apart, and neither penetration nor force-at-a-distance are possible; that is, for each $i$,
\begin{align}\label{eq:NormalComplementarity}
	\Jn[i] = \PartialDiff{\Gap_i}{\Configuration} \GeneralizedVelocityJacobian\,, && \Complementary{\NormalForce[i]}{\Gap_{i}(\Configuration)} \,.
\end{align}
We denote the active and penetrating contacts at $\Configuration$ as
\begin{align}
	\ContactSet_A (\Configuration) & = \Braces{ i \in \ContactsSpace : \Gap_i(\Configuration ) \leq 0}\,,\label{eq:ActiveContactsDefinition}\\
	\ContactSet_P (\Configuration) &= \Braces{ i \in \ContactsSpace: \Gap_i(\Configuration ) < 0}\,.\label{eq:PenetratingContactsDefinition}
\end{align}
\item \textbf{Maximal dissipation}:
Friction dissipates as much power ($\Jt[i]\Velocity \cdot \FrictionForce[i]$) as possible.
Coulomb friction \citep{Popova2015} with coefficient $\FrictionCoeff[i]$ in particular obeys this property within the admissible set
\begin{equation}
	\Braces{\FrictionForce[i]: \TwoNorm{\FrictionForce[i]} \leq \FrictionCoeff[i]\NormalForce[i]}\,.\label{eq:CoulombAdmissibleSet}
\end{equation}
The corresponding set of generalized forces is the \textit{friction cone}
\begin{align*}
	\FrictionCone(\Configuration) &= \sum_{i \in \ContactSet_A(\Configuration)}\Braces{\J[i](\Configuration)^T\Force[i] : \TwoNorm{\FrictionForce[i]} \leq \FrictionCoeff[i]\NormalForce[i]}\,.
\end{align*}

The maximally-dissipative friction force and associated generalized force $\NetForce[i]$ opposes the sliding direction $\Unit(\Jt[i]\Velocity)$ as much as possible:
\begin{align}
	 \FrictionForce[i] &\in -\FrictionCoeff[i]\NormalForce[i]\Unit(\Jt[i]\Velocity)\,,\label{eq:CoulombForce}\\
	 \NetForce[i](\Configuration, \Velocity, \NormalForce[i]) &= \Parentheses{\Jn[i]^T - \FrictionCoeff[i]\Jt[i]^T\Unit(\Jt[i]\Velocity)}\NormalForce[i]\,.
\end{align}
We note in particular the identity
\begin{equation}
	\FrictionCone \Parentheses{\Configuration} = \sum_{i \in \ContactSet_A(\Configuration)} \NetForce[i](\Configuration, 0, \Real^+)\,.\label{eq:FrictionCone}
\end{equation}
A common variant of this model is the \textit{linearized} Coulomb model, in which the admissible set is replaced with $\Braces{\FrictionForce[i] \in  \FrictionCoeff[i]\NormalForce[i]\Hull\Parentheses{\Braces{\vect d_1, \dots, \vect d_\ConeBases}}}$ for $\ConeBases \in \Natural$ unit-length vectors $\vect D = [\vect d_1, \dots, \vect d_\ConeBases] \in \Real^{2 \times \ConeBases}$, leading to similar definitions of forces and a \textit{linearized friction cone}:
\begin{align}
	\Unit_{\vect D}(\vect r) &= \Hull\Parentheses{\arg\max_{d_i} \vect d_i \cdot \vect r}\,,\\
	 \FrictionForce[i] &\in -\FrictionCoeff[i]\NormalForce[i]\Unit_{\vect D}(\Jt[i]\Velocity)\,,\\
	 \NetForce[{\vect D},i](\Configuration, \Velocity, \NormalForce[i]) &= \Parentheses{\Jn[i]^T - \FrictionCoeff[i]\Jt[i]^T\Unit_{\vect D}(\Jt[i]\Velocity)}\NormalForce[i]\,,\\
	 \LinearFrictionCone \Parentheses{\Configuration} &= \sum_{i \in \ContactSet_A(\Configuration)} \NetForce[\vect D, i](\Configuration, 0, \Real^+)\,.\label{eq:LinearFrictionCone}
\end{align}
The identity $\Unit_{\vect D}(\vect r) \subseteq \Unit_{\vect D}(\ZeroVector) \subseteq \Unit(\ZeroVector)$ leads to
\begin{equation}
\begin{aligned}\textstyle
	\sum_{i \in \ContactSet_A(\Configuration)} \NetForce[\vect D, i](\Configuration, \Velocity, \Real^+) &\subseteq \LinearFrictionCone \Parentheses{\Configuration} &\subseteq \FrictionCone \Parentheses{\Configuration} \label{eq:LinearConeContainment}\,.
\end{aligned}
\end{equation}

\end{itemize}
$\Gap(\Configuration)$ is Lipschitz and continuously differentiable. We also assume that for all active, non-penetraing contacts, there exists a generalized velocity for which the contact is separating:
\begin{assumption}\label{assump:NoDegenerateContacts}
	$\forall i \in \ContactSet$, $\Gap_i(\Configuration) = 0 \implies \Jn[i](\Configuration) \neq \ZeroVector$.
\end{assumption}
$\Jn[i]$ is bounded and continuous by the properties of $\Gap$ and $\GeneralizedVelocityJacobian$, while $\Jt[i]$ has the same properties by assumption.
These properties can be guaranteed, for instance, for piecewise-smooth bodies with bounded curvature.
We note that because $\Gap$ is continuous, $\ContactSet_A (\Configuration)$ and $\ContactsSpace \setminus \ContactSet_P (\Configuration)$ are u.s.c. in $\Configuration$.
From these functions we also define $\ConfigurationSet_A = \Braces{\Configuration: \ContactSet_A (\Configuration) \neq \emptyset}$, the configurations with active contact, and $\ConfigurationSet_P = \Braces{\Configuration: \ContactSet_P (\Configuration) \neq \emptyset}$, the interpenetrating configurations.

We will often see that various theoretical guarantees (seminally including existence of solutions in continuous and discrete time \citep{Stewart2000}) for such systems depend on a \textit{pointedness} assumption on the friction cone $\FrictionCone$:
\begin{assumption}[Pointed Friction Cone]\label{assump:nondegenerate}
	At any configuration $\Configuration$, the friction cone $\FrictionCone(\Configuration)$ is \textit{pointed} in some direction $\vect d(\Configuration)$:
	\begin{equation}
		\forall \vect F \in \FrictionCone(\Configuration),\quad \vect d(\Configuration) \cdot \vect F \geq \TwoNorm{\vect F}\,.
	\end{equation}
	Therefore, there also exists $\Pointedness (\Configuration)$ such that for any $\Force$ with each $\FrictionForce[i]$ in the Coulomb admissible set \eqref{eq:CoulombAdmissibleSet},
	\begin{equation}
		\TwoNorm{\J^T\Force} \geq p(\Configuration)\TwoNorm{\Force}\,.
	\end{equation}
\end{assumption}
Finally, we define the following notation:
\begin{align}
	\Jn &= \begin{bmatrix}
		\Jn[1] \\
		\vdots \\
		\Jn[\Contacts]
	\end{bmatrix}\,, \quad 
	\Jt = \begin{bmatrix}
		\Jt[1] \\
		\vdots \\
		\Jt[\Contacts]
	\end{bmatrix}\,,  \quad 
	\J = \begin{bmatrix}
		\Jn \\
		\Jt
	\end{bmatrix}\,,\label{eq:BlockJacobianDefinition} \\
	\NormalForce &= \begin{bmatrix}
		\NormalForce[1] \\
		\vdots \\
		\NormalForce[\Contacts]
	\end{bmatrix}\,, \quad 
	\FrictionForce = \begin{bmatrix}
		\FrictionForce[1] \\
		\vdots \\
		\FrictionForce[\Contacts]
	\end{bmatrix}\,,  \quad 
	\Force = \begin{bmatrix}
		\NormalForce \\
		\FrictionForce
	\end{bmatrix}\,,\label{eq:BlockForceDefinition}\\
	\ActiveSet (\Configuration) &= \{ \Velocity \in \VelocitySpace : \exists i \in \ContactSet_A(\Configuration),\, \Jn[i] \Velocity < 0\label{eq:ImpactingVelocitiesDefinition}\}\,,\\
	\InactiveSet (\Configuration) &= \{ \Velocity \in \VelocitySpace : \forall i \in \ContactSet_A(\Configuration),\,\Jn[i]\Velocity > 0\}\,.\label{eq:SeparatingVelocitiesDefinition}
\end{align}
$\ActiveSet(\Configuration)$ is the set of \textit{colliding} velocities, for which an active contact is moving towards penetration and must cause an impact.
$\InactiveSet(\Configuration)$ is the set of \textit{separating} velocities, where no impact occurs as all contacting surfaces are moving away from each other.
While $\ActiveSet(\Configuration)$ and $\InactiveSet(\Configuration)$ are disjoint, there may be some velocities in \textit{neither} set; these cases may generate impacts, as in Painlev\'e's Paradox \citep{Stewart2000}, discussed in \Cref{subsubsec:IVPs}.
By Assumption \ref{assump:NoDegenerateContacts}, when $\Configuration$ is non-penetrating, $\InactiveSet (\Configuration) = \Interior \Parentheses{\Complement{\ActiveSet(\Configuration)}}$ and $\ActiveSet (\Configuration) = \Interior \Parentheses{\Complement{\InactiveSet(\Configuration)}}$.
\subsubsection{Instantaneous, Inelastic Impact Laws}\label{subsec:impactbackground}
\eqref{eq:ManipulatorEquations}, \eqref{eq:NormalComplementarity}, and \eqref{eq:CoulombForce} provide only a partial solution to initial value problems (IVPs).
Bodies can {collide} or come into contact with non-zero velocity ($\Gap_i(\Configuration(t)) = 0$ and $\TimeDiff{}{}\Gap_i(\Configuration(t)) < 0$); penetration therefore must be avoided via an \textit{impact} or instantaneous velocity jump from $\Velocity^-$ to $\Velocity^+$ obeying
\begin{equation}\textstyle
	\Mass(\Configuration)(\Velocity^+(t) - \Velocity^-(t)) = \sum_{i \in \ContactSet_A(\Configuration)} \J[i](\Configuration)^T\Impulse[i],\label{eq:ImpulsiveNetwonsSecond}
\end{equation}
arising from instantaneous contact \textit{impulses} $\Impulse[i]$. As $\TimeDiff{\Velocity}$ does not exist, an alternative formulation to ODEs equations in time is required to capture this behavior.

Several models select $\Impulse$ via an impulsive analog to Coulomb's friction law \citep{Anitescu97,Glocker1995,Routh91}, with additional constraints pertaining to the elasticity of the impact.
We focus discussion and our own modeling efforts on \textit{inelastic} collisions, which are well defined in the single-impact case via the constraint $\Jn[i]\Velocity^+ = 0$.
Each discussed model makes its own generalization of this concept to simultaneous impacts, and there is in general no single accepted rule \citep{Stewart2000}.
 we note that many of the models here have extensions to partially- and fully-elastic collisions, with much effort going to preserving energy dissipation in these cases \citep{Stronge90,Mirtich1996,Anitescu97,Liu2008,Liu2008_2,Glocker2012,Glocker2013,Nguyen2018}.

In this paper, we will consider and combine concepts from two families of impact models: algebraic and differential. In this section, we discuss how different methods makes their own nuanced translations of the complementarity and maximal dissipation laws from sustained contact to impacts, resulting in distinct theoretical and computational characteristics.

Algebraic methods calculate $\Impulse[i]$ as the solution to a finite-dimensional system of algebraic equations \citep{Anitescu97,Hurmuzlu1994,Glocker1995,Chatterjee1998TwoInterpretations}, which relate the pre- and post-impact velocities to the impact's underlying impulses.
Such systems of equations can be approximately computed via numerical optimization.

In some of these models, all impacts are resolved simultaneously. For inelastic impacts, \citet{Glocker1995} and \citet{Anitescu97} for instance solve for an impulse $\Impulse$ which both prevents penetration and (approximately) satisfies linearized Coulomb friction at the \textit{post}-impact velocity $\Velocity^+$:
\begin{subequations}
\begin{align}
	\textrm{find} &&\Velocity^+; \Braces{\Impulse[i]:  i \in \ContactSet_A}\,, \\
	\textrm{s.t.}&& \textrm{impulse/impact balance \eqref{eq:ImpulsiveNetwonsSecond} }\,,\\
	&& \Complementary{\NormalImpulse[i]}{\Jn[i]\Velocity^+}\label{eq:velocitybasedcomplementarity}\,,\\
	&& \FrictionImpulse[i] \in -\FrictionCoeff[i]\NormalImpulse[i]\Unit_{\vect D}(\Jt[i]\Velocity^+)\label{eq:PostImpactLinearCoulomb}\,.
\end{align}\label{eq:AnitestcuDiscreteFormulation}
\end{subequations}
A critical feature of the algebraic formulation \eqref{eq:AnitestcuDiscreteFormulation} is the use of linearized Coulomb friction, which allows it to be cast as a solvable, copositive LCP (see \Cref{prop:LCPCopExist}).
We refer the reader to \cite{Stewart1996a} for a full description, but provide a short summary below.
Letting $\FrictionForce[i] = \vect D \FrictionBasisForce[i]$ and $\JD[i] = \vect D^T \Jt[i]$, \eqref{eq:PostImpactLinearCoulomb} can be captured as as the complementarity constraints
\begin{align}
&\Complementary{\FrictionBasisForce[i]}{\JD[i](\Configuration) \Velocity^+ + \OneVector \SlackVelocity[i]}\,,\label{eq:SimulationFrictionBasisComplementarity} \\
&\Complementary{\SlackVelocity[i]}{\FrictionCoeff[i]\NormalForce[i] - \OneVector^T \FrictionBasisForce[i]}
	\,.\label{eq:SimulationSlackVelocityComplementarity}
\end{align}
For convenience, we define the lumped terms
\begin{align}
\FrictionBasisForce &= \begin{bmatrix}
		\FrictionBasisForce[1] \\ \vdots \\ \FrictionBasisForce[\Contacts]
	\end{bmatrix}\,, & \JD &= \begin{bmatrix}
		\JD[1] \\ \vdots \\ \JD[\Contacts]
	\end{bmatrix}\,,\label{eq:ForceDJD} \\
\bar \Force &= \begin{bmatrix}
		\NormalForce \\ \FrictionBasisForce
	\end{bmatrix}\,,\ & \bar \J &= \begin{bmatrix}
		\Jn \\ \JD
	\end{bmatrix}\,.\label{eq:barJbarForce}
\end{align}
This casting of multiple, simultaneous impacts as a single LCP is a significant computational advantage, as only one, solvable numerical program must be instantiated to calculate the post-impact velocity.
Furthermore, it is know that solutions to this LCP always dissipate kinetic energy \citep{Anitescu97}.
However, the constraints embedded in this problem are often violated in real systems with multiple contacts, in particular the so-called velocity-based complementarity \eqref{eq:velocitybasedcomplementarity} formulation of inelasticity \citep{Chatterjee1999}.

An alternative algebraic view of simultaneous impacts that does not require the same velocity-based complementarity constraints is to resolve multi-impact as a sequence of individual impacts, as in \citet{Ivanov1995,Smith2012,Seghete2014}; and many other models.
To summarize this technique:
\begin{enumerate}
	\item Pick a single active contact $i \in \ContactSet_A(\Configuration)$.
	\item Resolve a single impact at $i$ with some impulse $\Impulse[i]$, and increment $\Velocity \gets \Velocity + \Mass^{-1}\J[i]^T\Impulse[i]$.
	\item Terminate and take $\Velocity^+ = \Velocity$ if it is non-colliding ($\Velocity \not \in \ActiveSet(\Configuration)$); otherwise, return to step 1.
\end{enumerate}
Various methods differ in their choice of contact ordering as well as single-impact resolution, resulting in distinctly different final outcomes to the same initial conditions.
Some such methods are only able to guarantee that the process terminates under significant assumptions, e.g. two or fewer contacts \citep{Seghete2014}.
Additionally, such methods by design are unable to directly represent partially-concurrent impacts that occur in real-world systems.
In \Cref{sec:examples}, sequences of single impacts resolved using \eqref{eq:AnitestcuDiscreteFormulation} will serve as a point of comparison for a new collision law that we develop.
As each individual impact dissipates kinetic energy, this sequential application will always predict a post-impact velocity with non-increased energy, provided that the termination condition is reached.

In \Cref{fig:phone} above, we provide a simple example which illustrates both how simultaneous vs. sequential resolution, as well as different sequential orderings, can result in distinct outcomes even for an extremely simple example.
We consider an instance of the classically-studied ``rocking block'' system \citep{Housner1963,Zhang2001,Lygeros2003,Yilmaz2009}.
A slender rectangular block with velocity $\Velocity^+$ is dropped onto flat ground, colliding at two corners $\vect A, \vect B$. It is assumed that the constituent materials generate inelastic impacts (zero coefficient of restitution), such that any concurrent collisions result in non-separating post-impact velocities.
Affixing \eqref{eq:AnitestcuDiscreteFormulation} as the model for impacts and only changing between simultaneous and sequential resolution, we find that 3 different outcomes might be predicted, corresponding to rest or rolling off of either corner. 

As opposed to algebraic models, differential impact models consider continuous evolution of velocity from pre- to post-impact velocity, in which the total derivative of $\Velocity$ satisfies laws of frictional contact in some form.
In the context of rigid contact models, this derivative $\TotalDiff{\Velocity}{s} = \dot\Velocity(s)$ is with respect to a variable of integration $s$ which does not correspond to time, but rather measures the impulse accumulated over an instantaneous collision.
At least in a limited capacity, such methods do directly represent the time-dependence and continual evolution of real-world object velocities during impact, which in \Cref{section:model} will allow us to represent partially-concurrent impacts resolving at arbitrary relative rates.
This fidelity however necessitates computationally expensive simulation of non-smooth or constrained differential equations to resolve impacts, and thus such methods have not been a focus of modern, efficient simulation \citep{Castro2020,Coumans2015}.

We will now describe one of the oldest differential models for a single impact \citep{Routh91}, which we will later extend to the simultaneous impact case.
This method was first presented by Routh in 2 dimensions, and extended to 3 dimensions later by \citet{Keller1986} \citep{Wang1992}.
For a single contact $\ContactSet_A(\Configuration) = \Braces i$, \citet{Routh91} proposed a method which satisfies Coulomb friction differentially.
To summarize this technique,
\begin{enumerate}
    \item Increase the normal impulse $\NormalImpulse[i]$ with slope $\NormalForce[i] = 1$.\label{item:routhnormalstep}
    \item Increment the tangential impulse with slope $\FrictionForce[i]$ satisfying Coulomb friction \eqref{eq:CoulombForce}
    for the mid-impact velocity $\bar{\Velocity}=\Velocity + \Mass^{-1}\J[i]^T\Impulse[i]$.\label{item:routhfrictionstep}
    \item Stop at the inelastic condition $\Jn[i]\bar{\Velocity}=0$; set $\Velocity^+ = \bar \Velocity$.\label{item:routhtermination}
\end{enumerate}
As observed in \cite{Posa16}, this process is equivalent to the DI
\begin{equation}
\TotalDiff{\Velocity}{s}\in \Mass(\Configuration)^{-1}\NetForce[i](\Configuration,\Velocity(s), 1)\,.\label{eq:routhsingle}
\end{equation}
Note that for a frictionless contact ($\FrictionCoeff = 0$), this simplifies to $
	\Mass\dot \Velocity = \Jn[i]^T$.
\begin{figure}
\centering
\includegraphics[width=0.6\hsize]{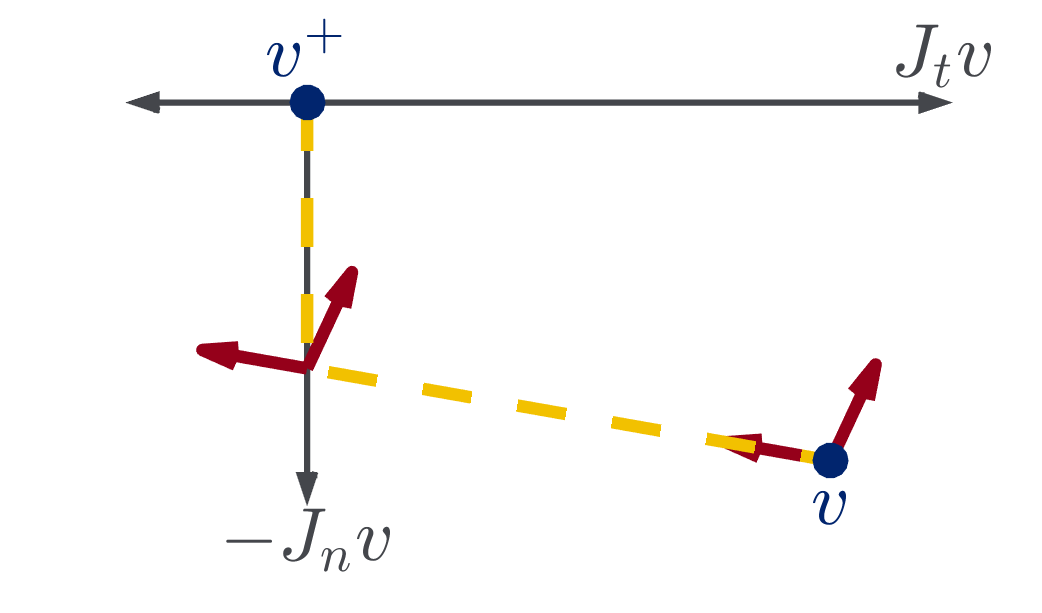}
\caption{Velocity through an impact resolved by Routh's method (adapted from \citet{Posa16}).
The extreme rays of the friction cone are shown as solid red arrows.
The contact begins in a sliding regime.
When $\Velocity$, shown in the yellow dotted line, intersects $\Jt \Velocity = \ZeroVector$, the contact transitions to sticking and the impact terminates when $\Jn \Velocity=\ZeroVector$.}
\label{fig:routh}
\vspace{-6mm}
\end{figure}
A diagram depicting the resolution of a planar impact with this method is shown in Figure~\ref{fig:routh}.
Solutions may transition from sliding to sticking, and the direction of slip may even reverse.
While the path is piecewise linear in the planar case, this is not true in three dimensions \citep{Keller1986,Wang1992}.
We additionally note that while \eqref{eq:routhsingle} predicts ``forces'' even when $\Velocity$ is separating ($\Jn[i]\Velocity > 0$), Routh's method is by definition only used on velocity trajectories starting with $\Jn[i]\Velocity \leq 0$ until the first moment that that $\Jn[i]\Velocity = 0$, and thus inelasticity is preserved.

Implicit in Routh's method is an assumption that the terminal condition in step \ref{item:routhtermination}) will eventually be reached; if it is possible to get ``stuck'' with $\Jn[i]\Velocity < 0$ forever, then Routh's method would be ill-defined and not predict a post impact state.
This does not happen in the frictionless case, as $\Jn[i]\Velocity$ has constant positive derivative $\Jn[i]\dot \Velocity = \Norm{\Jn[i]}_{\Mass^{-1}}^2$.
With more careful treatment capturing kinetic energy dissipation, a similar result can be shown for the frictional case:
\begin{lemma}[Single Impact Termination (Appendix \ref{adx:singlestopproof})]\label{lem:singlefrictional}
Let $\Configuration \not \in \ConfigurationSet_P$ be a non-penetrating configuration, and $i \in \ContactSet_A(\Configuration)$ be an active contact. Then there exists $\kappa(\Configuration)>0$ such that for any solution $\Velocity(s): \Brackets{0,\TwoNorm{\Velocity(0)}\kappa(\Configuration)} \to \VelocitySpace $ of the single frictional contact system \eqref{eq:routhsingle}, $\Velocity(s)$ exits the impact at some $s^* \leq \TwoNorm{\Velocity(0)}\kappa(\Configuration)$; i.e., $\Jn[i]\Velocity(s^*) \geq 0$.
\end{lemma}

The implication of \Cref{lem:singlefrictional} is that \emph{a priori}, one can determine an $s > 0$ proportional to the pre-impact speed $\TwoNorm{\Velocity}$ (with constant of proportionality $\kappa$) such that any solution to the DI \eqref{eq:routhsingle} on $\Brackets{0,s}$ can be used to construct the post-impact velocity $\Velocity^+$. We will see, however, that the extension of this methodology to multiple concurrent impacts is non-trivial, and that physical systems associated with these models often exhibit non-uniqueness.

We note that Routh's method has previously be extended to the multiple impact case by \citet*{Liu2008,Liu2008_2}, often called the LZB model \citep{Nguyen2018}.
In this framework, relative rates of impulse accrual are set via an energy-based framework, which takes as parameterization the stiffnesses of each contact involved.
These models have the capability to capture Coulumb friction as well as partially-elastic collisions via a bi-stiffness modeling approach.
As we instead develop a model which allows for simultaneous, inelastic impacts to resolve at arbitrary relative rates when stiffnesses are unknown, the special case of perfectly-inelastic LZB impacts with any material stiffnesses will be exactly captured by our model.

\subsubsection{Initial value problems through impact}\label{subsubsec:IVPs}
Any complete solution to continuous-time IVP's for rigid bodies undergoing impacts must somehow combine the sustained-contact and instantaneous impact models described above.
Several formalisms have been developed to this end. Hybrid systems modeling combines ODE's with discrete jumps which are triggered when the continuous-time state reaches certain algebraic conditions; in the context of rigid-body models, such events represent instantaneous impacts \citep{Brogliato2002,Ames2006,Johnson2016a,Burden2016}.
Such methods are commonly simulated in an event-driven scheme, in which ODE numerical integration is interrupted when impact conditions are met, and instantaneous impulses are resolved \citep{Ames2006,Johnson2016a}.
Building on the early ideas of \citet{Lecornu1905}, \citet{Moreau1977} instead developed an alternative measure differential inclusion (MDI) formalism which permits non-zero impulses in $\FrictionCone(\Configuration)$ to occur over an infinitesimal time period $\Differential t$.
Similar to differential inclusions, MDI's are rigorously defined in the language of Lebesgue calculus and measure theory.
These models are often simulated with a \textit{time-stepping} scheme \citep{Stewart1996a}, in which net impulses combining continuous forces and and impacts over a non-zero time period $\Delta t$ are determined.

Much theoretical work has been concerned with the \textit{consistency} of such models \citep{Stewart1998,Stewart2000,Brogliato2002,Ames2006,Marques2013} or the existence of solutions to IVP's for every valid initial condition.
Two types of pathological scenarios to this end have received much attention: \citet{Painleve1895} and Zeno \citep{Ames2006} behaviors. The model which we develop is capable of producing solutions through each of these scenarios; we accordingly now describe these behaviors and discuss related results in other modeling frameworks.

Early hybrid-system formulations trigger impact events if and only if a collision occurs \citep{Brogliato2002,Ames2006}.
However, since at least \citet{Jellet1872} and later detailed by \citet{Painleve1895}, this rule lead to non-existence of solutions for sustained contact when the continuous-time manipulator equations \eqref{eq:ManipulatorEquations} are combined with Coulomb friction\endnote{It is important to note that Painlev\'e also considered non-uniqueness to be a pathology of rigid-body assumptions and Coulomb friction; more discussion of this topic is covered at length in \citet{Stewart2000}.
As the subject of this paper concerns deliberate non-uniqueness, we forgo detailed discussion of this perspective in this work.}.
Although controversial, the prevailing treatment of these scenarios is to allow for impacts without collisions (IWC's, also called tangential collisions) when non-existence is encountered \citep{Genot1999,Stewart2000,Brogliato2002,Zhao2007}.
These behaviors are characterized by an instantaneous impact of the form \eqref{eq:ImpulsiveNetwonsSecond} despite the fact that no bodies are colliding (i.e. $\Jn[i]\Velocity^- = 0$ rather than $\Jn[i]\Velocity^- < 0$).
This can be modeled in hybrid systems for instance by adding additional events to trigger IWC's \citep{Genot1999,Brogliato2002}.
\citet{Stewart1998} seminally proved and demonstrated on a classic 2D rod example that Moreau's MDI naturally generates IWC behaviors, and accordingly IVP's can be solved with this model.
The associated proof of existence, derived by constructing a solution as the limit of discrete time-stepping simulations as the time-step duration $\Delta t \to 0$, is a preeminent consistency proof for MDI's and applies broadly to single-contact systems.
It is not known if such a method works completely for multiple contacts, in particular if such limits correctly comply with inelasticity constraints and Coulomb friction; a partial characterization of such limits is available assuming that the friction cone is pointed \citep{Stewart1998}.
\citet{Zhao2007} demonstrated that Routh's method can be used to resolve a 3-D analogue of this rod example, with an IWC that results in sticking contact.
Our model, also derived from Routh's model and equivalent to it in the one-contact case, accordingly produces solutions to such scenarios with IWC's.

Another pathology of particular interest for hybrid systems, \textit{Zeno behavior} \citep{Ames2006}, occurs when models lead to an infinite sequence of impact events within a finite duration of time (i.e. impact $i$ happens at $t_i$ with $\lim_{i\to\infty} t_i < \infty$).
Such behavior presents both a practical simulation challenge as well as a theoretical challenge, as numerical solvers would have to compute solutions to infinite impact resolutions to simulate a finite time duration.
A familiar example of Zeno behavior is a ball bouncing on flat ground with partially elastic collisions \citep{Acary08}.
Such phenomena can occur even with completely inelastic impacts, such as with a rocking block which wobbles from corner to corner, losing a fraction of momentum each time in a similar fashion to the bouncing ball; a detailed analysis is available in \citet{Lygeros2003}.
\citet{Johnson2016a} model this example by introducing a ``pseudo-impulse'' behavior that precludes Zeno phenomena, which modifies the wobbling behavior to predict sticking after finitely-many events.
\citet{Ames2006} instead proposes a ``completed'' hybrid system which extends solutions past the Zeno point by maintaining sticking contact at each contact involved in the Zeno phenomenon.
Neither method captures a broad array of frictional behaviors, with the former capturing only sticking friction on massless limbs, and the latter entirely frictionless.
In \Cref{subsec:zenoexample}, we reproduce a version of this example to illustrate our model's predictions in the presence of Zeno behavior.

In Section \ref{section:continuous_time_model}, we derive a differential inclusion model  (\Cref{eq:ContinuousTimeModel}) which generally applies to multi-body, multiple contact systems; specifies impacts to be inelastic; and guarantees existence of solutions (\Cref{thm:ContinuousTimeSolutions,thm:ContinuousTimeMinimumAdvancement}) under similar assumptions as \citet{Stewart1998} (see Assumption \ref{assump:nondegenerate}).
The theoretical guarantees for our model are more general than those for the MDI presented in \citet{Stewart1998}, in that Coulomb friction and inelasticity are well-characterized even in the multiple contacts case.
While DI's have long been used in rigid-body dynamics \citep{Leine08}, this paper and concurrent work \citep{Nurkanovic2020,Nurkanovic2021} are the first to solve IVPs through impacts via adding time as a state.
This work is the first DI to capture both inelasticity and friction in impact.
We additionally combine these ideas with the LCP-based structure of time-stepping simulation \cite{Stewart1996a} to develop our own discrete impact integrator in Section \ref{section:simulation}.

\section{Simultaneous Impact Model}
\label{section:model}
\begin{figure*}[h]
    \center        
    \begin{subfigure}[b]{.24\linewidth}
        \includegraphics[width=.96\hsize,trim={15mm 24mm 10mm 0mm},clip]{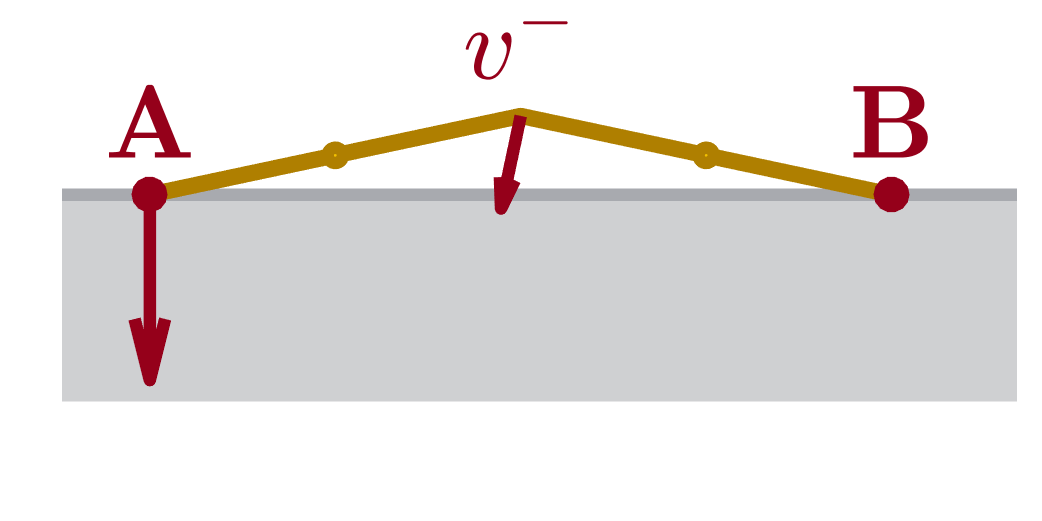}
        \centering
        \caption{\label{fig:compass_cartoon} Initial condition \\ (pre-impact)}
    \end{subfigure}
    \begin{subfigure}[b]{.72\linewidth}
    	\centering
        \includegraphics[width=.32\hsize,trim={15mm 24mm 10mm 0mm},clip]{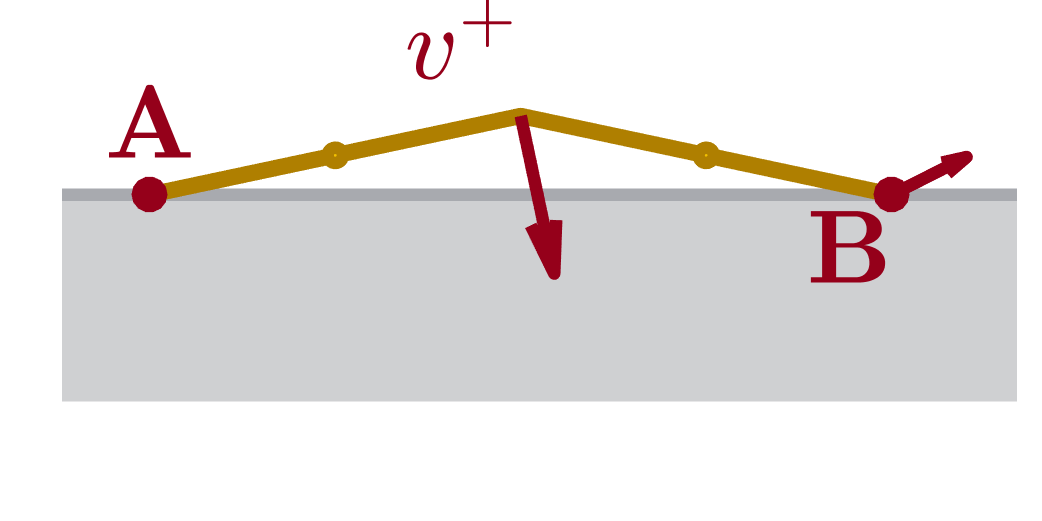}
        \includegraphics[width=.32\hsize,trim={15mm 24mm 10mm 0mm},clip]{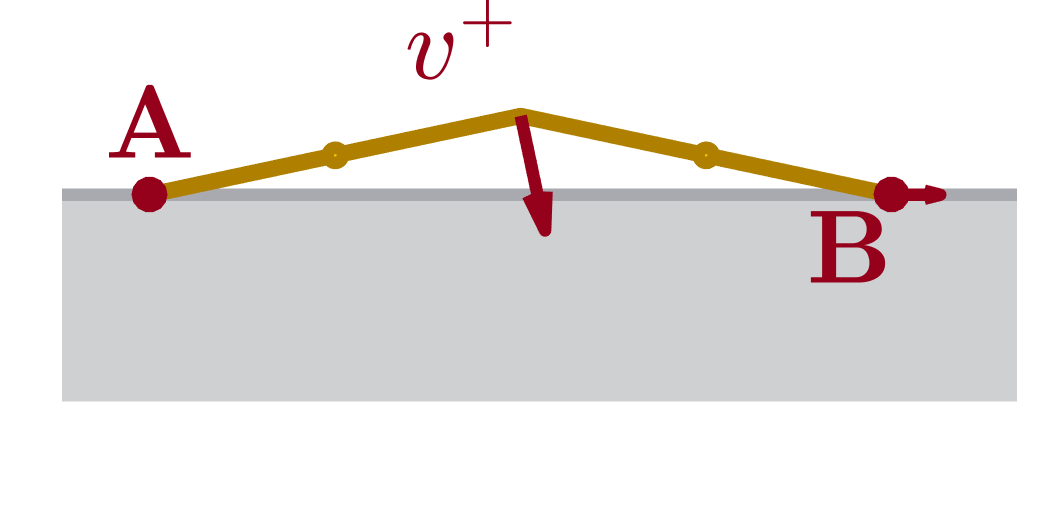}
        \includegraphics[width=.32\hsize,trim={15mm 24mm 10mm 0mm},clip]{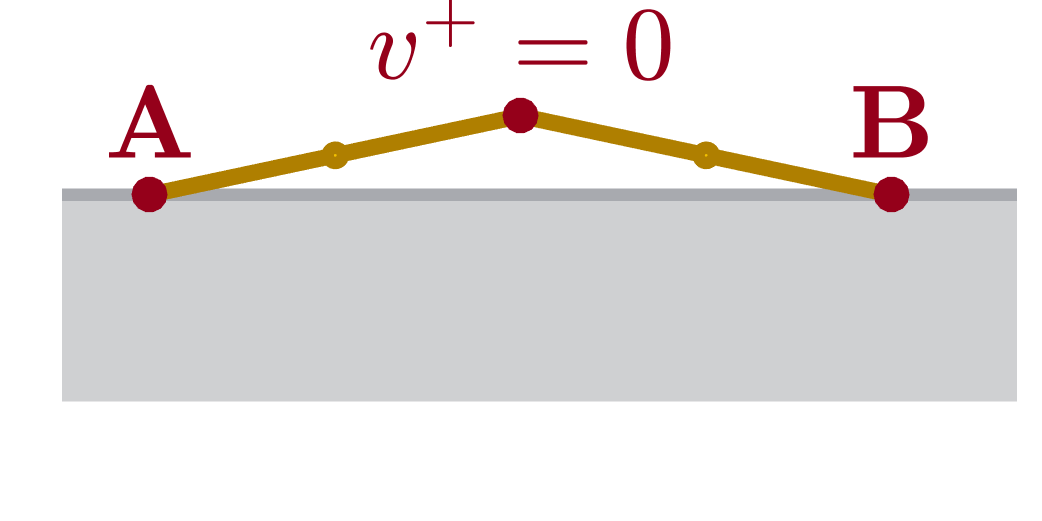}
        \caption{\label{fig:compass_examples} Simultaneous algebraic impacts \\ (post-impact)}
    \end{subfigure}
    \caption{(\subref*{fig:pushing_cartoon}) A compass gait walker, consisting of two legs attached with a hinge joint at the hip, takes a step with hip velocity $v$ and excites non-uniqueness in the model of \citet{Anitescu97}. 
        (\subref*{fig:compass_examples}) A single impact at that the leading foot (point A) can cause the trailing foot (point B) to lift off the ground. Alternatively, impacts at both feet can cause the trailing foot to slide or come to rest.\label{fig:compass_setup}}
\end{figure*}
\begin{figure*}[ht]
    \center        
    \begin{subfigure}[b]{.24\textwidth}
    \centering
        \includegraphics[width=.90\hsize,trim={10mm 14mm 10mm 5mm},clip]{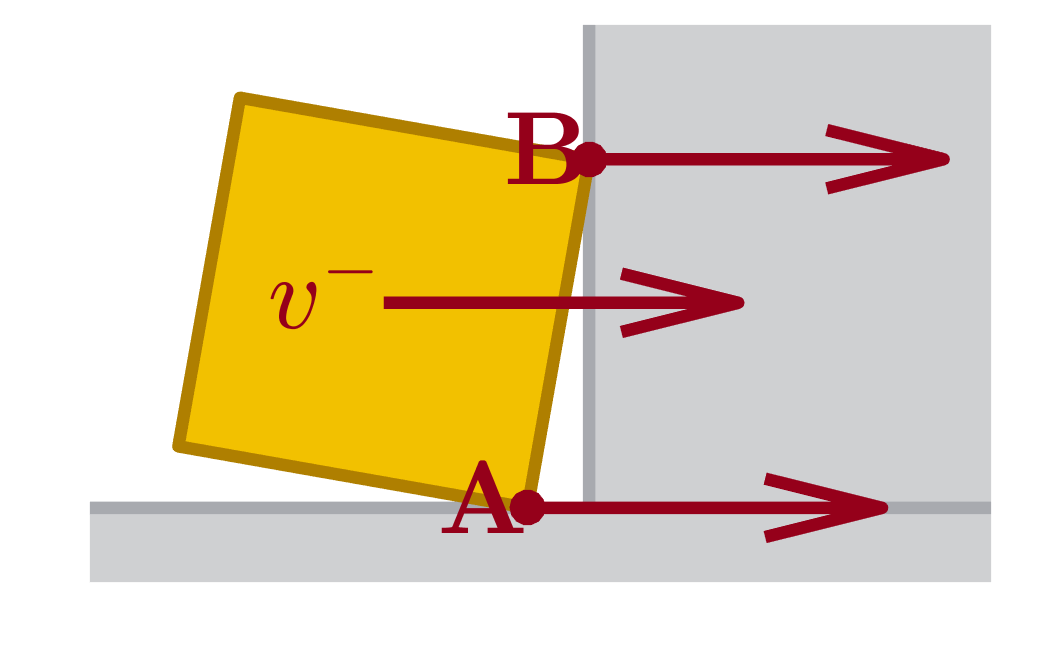}
        \caption{\label{fig:pushing_cartoon} Initial condition \\ (pre-impact)}
    \end{subfigure}
    \begin{subfigure}[b]{.24\textwidth}
    	\centering
        \includegraphics[width=.90\hsize,trim={10mm 14mm 10mm 5mm},clip]{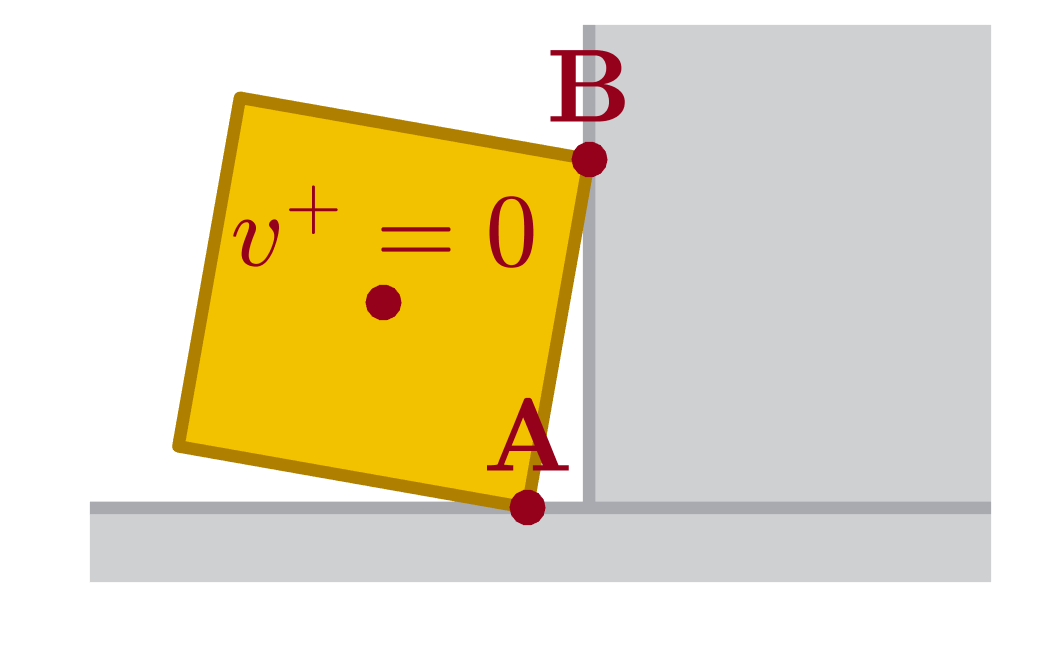}
        \caption{\label{fig:pushing_example_1} Simultaneous algebraic impact \\ (post-impact)}
    \end{subfigure}
    \begin{subfigure}[b]{.48\textwidth}
    	\centering
        \includegraphics[width=.45\hsize,trim={10mm 14mm 10mm 5mm},clip]{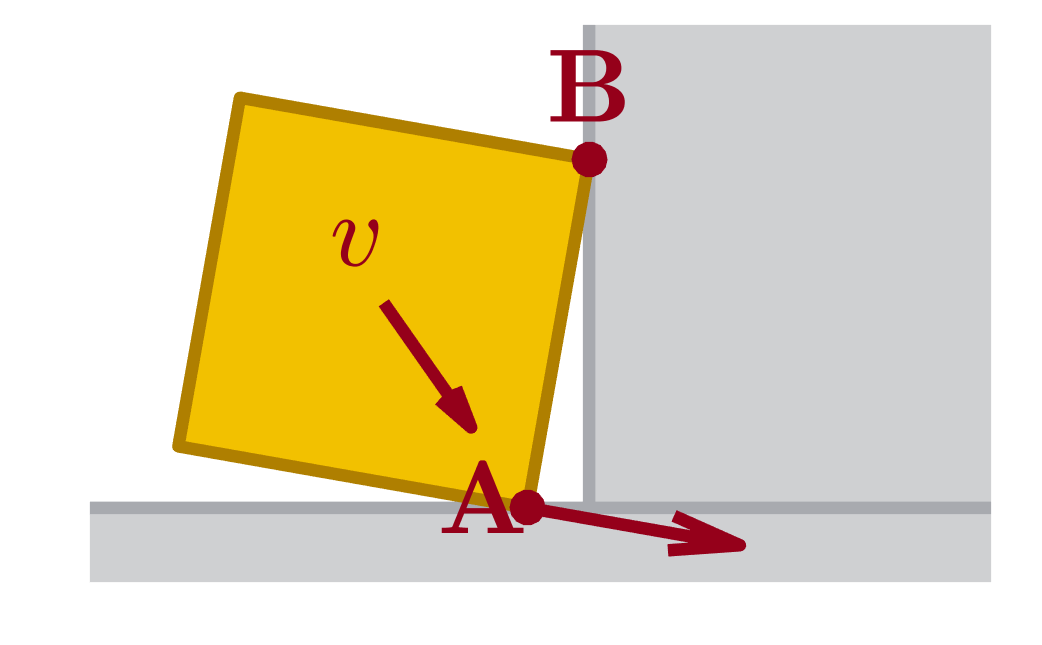}
        \includegraphics[width=.45\hsize,trim={10mm 14mm 10mm 5mm},clip]{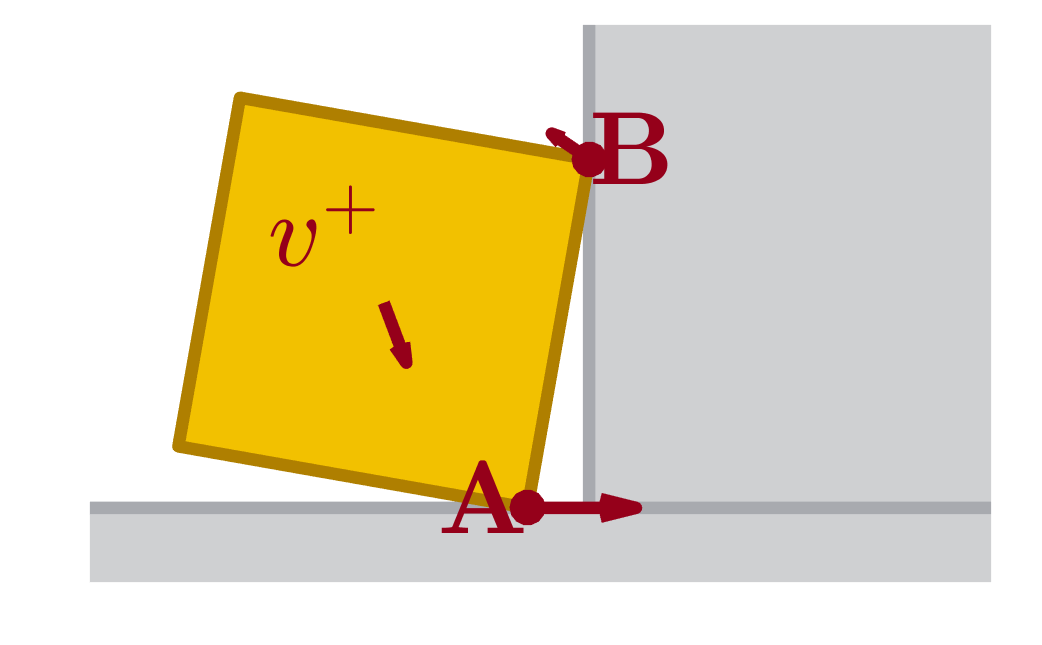}
        \caption{\label{fig:pushing_example_2} B-then-A sequential algebraic impact \\ (mid- then post-impact)}
    \end{subfigure}
    \caption{Subtly different solutions for a box sliding into a wall with velocity $v$ (\subref*{fig:pushing_cartoon}) are shown. 
        (\subref*{fig:pushing_example_1}) When a simultaneous impact is generated via \citet{Anitescu97}, the box comes to rest.
        (\subref*{fig:pushing_example_2}) When point B has an impact before point A, point A instead continues sliding, while point B lifts off the wall.\label{fig:box}}
\end{figure*}
Figure \ref{fig:phone} demonstrates that simultaneous collisions can excite quantitative and qualitative disagreement between common impact models' predictions, with even the post-impact contact mode differing.
This discrepancy occurs even when the same physical parameters such as mass and coefficients of friction and restitution are provided to these models.
However, making two points collide at \textit{exactly} the same time is unlikely in real life.
Nonetheless, as shown on a real-world system by \citet{Chatterjee1999}, even a single collision can result in multiple outcomes depending on the ordering of impulse accumulation between contacts.
In this section, we first offer two additional examples of this type---one related to legged locomotion and the other to manipulation; further details on the models and experiments can be found in Appendix \ref{adx:exampledetails} and \Cref{sec:examples}.
We then describe and characterize model that captures the non-uniqueness due to impulse ordering by extending Routh's method to multiple contacts with arbitrary relative rates.

\subsection{Motivating Examples}\label{subsection:motivatingexamples}

A ubiquitous model of bipedal walking is the compass gait walker, which consists of two rods (legs) connected with a revolute joint at the hip.
Bipedal walking involves stepping with a leading foot while a trailing foot rests on the ground, as shown in Figure \ref{fig:compass_setup}.
As observed by \citet{Remy2017}, if a wide step ($156^{\circ}$ between the legs) is taken by the model, then the simultaneous method of \citet{Anitescu97} results in three categorically different solutions.
In one case, there is only an impact at the leading foot, and the trailing foot lifts off the ground.
In two others, impacts at both feet can result in the trailing foot sliding or coming to rest.

In the second example, motivated by non-prehensile pushing of an object, we consider a box which slides on one corner on a floor before impacting a wall (Figure \ref{fig:box}).
If a single impact occurs between the box and the wall, it will trigger a second impact against the floor.
Due to the position of the center of mass of the box, both impacts add counter-clockwise rotational momentum to the box, causing the contact with the wall to lift off.
Alternatively, if both of these impacts are resolved simultaneously, the box comes to rest under sufficient friction.

\subsection{Simultaneous Impact Model Construction}\label{subsec:ImpactModelConstruction}
We have previously demonstrated that some simultaneous impact models are sensitive to impulse ordering.
As predicting this ordering demands precise knowledge of initial conditions and material properties beyond the fidelity of robotic sensors and simplified rigid-body models, we instead seek to predict the set of outcomes that result from arbitrary impulse orders.

The foundational concept of this model is that while Routh's method models impacts as instantaneous \citep{Routh91}, the variable of integration $s$ provides a natural way to specify the relative rates of impulse accural between concurrent impacts.
A similar model, without theoretical results or a detailed understanding, was proposed by \citet{Posa16} where it proved useful for stability analysis of robots undergoing simultaneous impact.
We consider the following extension to Routh's method which at any given instant during the resolution process, the impacts are allowed to concurrently resolve at \textit{any} relative rate:
\begin{enumerate}
	\item Increase $\NormalImpulse[i]$ on each non-separating $(\Jn[i]\Velocity \leq 0)$ active contact $i \in \ContactSet_A(\Configuration)$ at rate $\NormalForce[i] \geq 0 $ such that
		\begin{equation}
			\sum_{i} \NormalForce[i] = \Norm{\NormalForce}_1 = 1\,.\label{eq:convexforcecombination}
		\end{equation}\label{step:MultiContactRouthSelection}
	\item Increment each tangential impulse with slope $\FrictionForce[i]$ satisfying Coulomb friction \eqref{eq:CoulombForce}
    at $\bar{\Velocity}=\Velocity + \Mass^{-1}\J^T\Impulse$.
	\item Terminate when all $\Jn[i]\bar{\Velocity}\geq 0$, i.e. $\bar \Velocity \not \in \ActiveSet(\Configuration)$. $\Velocity^+ = \bar \Velocity$.
\end{enumerate}
We can understand the constraint (\ref{eq:convexforcecombination}) on $\Force$ as choosing a net force that comes from a convex combination of the forces that Routh's method might select for any of the individual contacts $i \in \ContactSet_A(\Configuration)$.
In particular, we note that step \ref{step:MultiContactRouthSelection} restricts the normal forces to be dissipative, i.e.
\begin{equation}
	\NormalForce[i] \cdot \Jn[i]\Velocity \leq 0\label{eq:NormalDissipationGuaranteed}\,.
\end{equation}
As before, we can capture this behavior as a DI:
\begin{align}
	\dot \Velocity &\in \DerivativeMap[\Configuration](\Velocity) = \Mass^{-1} \Hull \Parentheses{\bigcup_{i \in \ContactSet_{\Configuration}(\Velocity)}\NetForce[i](\Configuration, \Velocity, 1)}\,,\label{eq:multicontact}\\
	\ContactSet_{\Configuration}(\Velocity) 	&= \begin{cases}
	\Braces{i \in \ContactSet_A(\Configuration) : \Jn[i] \Velocity \leq 0} & \Velocity \in \Closure \ActiveSet(\Configuration)\,,\\
		\arg\min_{i \in \ContactSet_A(\Configuration)}\Jn[i] \Velocity & \Otherwise \,.\label{eq:permissiblecontacts}
	\end{cases}
\end{align}
While non-physical, the behavior outside of $\Closure \ActiveSet(\Configuration)$ has been chosen to preserve upper semi-continuity, and is not encountered when resolving impacts due to the termination condition $\Velocity \not \in \ActiveSet(\Configuration)$.
The construction of \eqref{eq:multicontact} is similar to that of the single contact system \eqref{eq:routhsingle}; it is furthermore equivalent to \eqref{eq:routhsingle} and therefore Routh's method when only one contact is active. 
\subsection{Properties}
We now detail properties of our simultaneous impact system that are useful for analyzing its solution set. 
\subsubsection{Existence and Closure}

For any configuration $\Configuration \in \ConfigurationSet_A$, $\DerivativeMap[\Configuration](\Velocity)$ is non-empty, closed, uniformly bounded, and convex. Therefore by \Cref{prop:closure}, we obtain the following: 
\begin{theorem}[Existence of Solutions (Appendix \ref{adx:NonEmptyClosureProof})]\label{lem:NonEmptyClosure}
	For all configurations $\Configuration \in \ConfigurationSet_A$, velocities $\Velocity_0$, and compact intervals $[a,b]$, $\SolutionSet{\DerivativeMap[\Configuration]}[[a,b]]$ and $\IVP{\DerivativeMap[\Configuration]}{\Velocity_0}{[a,b]}$ are non-empty and closed under uniform convergence.
\end{theorem}

\subsubsection{Energy Dissipation}\label{subsubsec:Dissipation}
An essential behavior of inelastic impacts reflected in our model is that they dissipate kinetic energy.
By construction of \eqref{eq:multicontact}, the kinetic energy $\KineticEnergy(\Configuration,\Velocity(s))$ is continually non-increasing during impact (i.e. when $\Velocity(s) \in \Closure \ActiveSet(\Configuration)$) as normal forces are constrained to be dissipative \eqref{eq:NormalDissipationGuaranteed} and frictional forces are naturally, maximally dissipative:
\begin{theorem}[Dissipation (Appendix \ref{adx:dissipateproof})]\label{lem:dissipate}
	Let $\Configuration \in \ConfigurationSet_A$, and let $[a,b]$ be a compact interval. If $\Velocity(s) \in \SolutionSet{\DerivativeMap[\Configuration]}[[a,b]]$ and $\Velocity([a,b]) \subseteq \Closure \ActiveSet(\Configuration)$, then $\Norm{\Velocity(s)}_{\Mass}$ is non-increasing.
\end{theorem}
\noindent The proof of this Theorem involves the calculation of the total derivative of $\KineticEnergy$ as
\begin{equation}
	\dot \KineticEnergy = \Velocity^T \J^T \Force\,.
\end{equation}
One might also wonder if $\KineticEnergy$ strictly decreases during impact; certainly, this would not be the case if $\Velocity(s)$ could stay constant.
Therefore, solutions to the differential inclusion must not be permitted to select $\dot \Velocity = \ZeroVector$, i.e., $\ZeroVector \not\in \DerivativeMap[\Configuration](\Velocity^*)$ for every $\Velocity^* \in \Closure\ActiveSet(\Configuration)$.
As $\DerivativeMap[\Configuration](\Velocity) \subseteq \Mass^{-1}\FrictionCone (\Configuration)$, this property is guaranteed by the pointed friction cone assumption (Assumption \ref{assump:nondegenerate}).
Assumption \ref{assump:nondegenerate} covers most situations in robotics---including grasping and locomotion---with the notable exception being jamming between immovable surfaces.
We note that this assumption does not preclude Painlev\'e-type scenarios necessitating impacts without collision \citep{Stewart1998}. 
Furthermore, it guarantees strict dissipation during the entirety of the impact process:
\begin{corollary}[Strict Dissipation (\Cref{adx:strictdissipationproof})]\label{thm:strictdissipation}
	Let $\Configuration \in \ConfigurationSet_A \setminus \ConfigurationSet_P$ and $[a,b]$ be a compact interval. If $\Velocity(s) \in \SolutionSet{\DerivativeMap[\Configuration]}[[a,b]]$ and $\Velocity([a,b]) \subseteq \Closure\ActiveSet(\Configuration)$, $\Norm{\Velocity(s)}_{\Mass}$ is strictly decreasing.
\end{corollary}
\subsubsection{Linear Impact Termination}
\label{section:termination}
While solutions to the underlying DI are guaranteed to exist in the simultaneous impact model, we have yet to prove that they terminate the impact process, as in Routh's single-contact method.
We now discuss a similar linear-duration condition:
\begin{proposition}[Finite Termination]\label{prop:nondegeneratedissipation}
	For any configuration $\Configuration \in \ConfigurationSet_A \setminus \ConfigurationSet_P$ and pre-impact velocity $\Velocity(0)$, the DI (\ref{eq:multicontact}) resolves the impact within a duration proportional to $\Norm{\Velocity(0)}_{\Mass}$.
\end{proposition}

We will prove this claim as a consequence of kinetic energy decreasing fast enough to force termination---a significant expansion of \Cref{thm:strictdissipation}. Even though $\KineticEnergy$ always decreases, \Cref{thm:strictdissipation} does not forbid $\dot\KineticEnergy$ from getting arbitrarily close to zero. For example, consider a 2 DoF system with 2 frictionless, axis-aligned contacts ($\Mass = \Jn = \Identity_2$). For any $\epsilon > 0$, we can pick a velocity and impulse increment which satisfy $\dot K > -\epsilon$: 
	\begin{align}
		\Velocity_\epsilon &= \frac{-1-\epsilon}{2}\begin{bmatrix}
			1 \\
			\epsilon
		\end{bmatrix} \in \ActiveSet(\Configuration)\,, &
		\Jn^T\begin{bmatrix}
			\epsilon \\ 1
		\end{bmatrix}\frac{1}{1+\epsilon} \in \DerivativeMap[\Configuration](\Velocity_\epsilon)\,. 
	\end{align}
	However as we take $\epsilon \StrongConvergence 0$, $\Velocity_\epsilon$ converges to a non-impacting velocity; thus, $\dot K$ only remains small for a short duration before impact termination. It remains possible that the aggregate dissipation over an interval of nonzero length can be bounded away from zero. We define this quality as $\DissipationRate(s)$-dissipativity:

\begin{definition}[$\DissipationRate(s)$-dissipativity]
For a positive definite function $\DissipationRate(s):\Closure \Real^+ \rightarrow [0,1]$, the system $\dot\Velocity \in \DerivativeMap[\Configuration](\Velocity)$ is said to be $\boldmath{\DissipationRate}(s)$\textbf{-dissipative} if for all $s > 0$, for all $\Velocity \in \SolutionSet{\DerivativeMap[\Configuration]}[[0,s]]$ s.t. $\Velocity\Parentheses{[0,s]} \subseteq \Closure \ActiveSet(\Configuration)$, if $\Norm{\Velocity(0)}_{\Mass} = 1$, $\Norm{v(s)}_{\Mass} \leq 1 - \DissipationRate(s)$.
\end{definition}
$\DissipationRate(s)$-dissipativity is a sufficient condition for linear-duration impact termination (\Cref{prop:nondegeneratedissipation}) from any initial velocity, and the particular form of $\DissipationRate(s)$ can be used to bound the linear rate:
\begin{lemma}[Termination via Aggregate Dissipation (Appendix \ref{adx:exitproof})]\label{lem:exit}
		Let $\Configuration \in \ConfigurationSet_A$ and let $\dot\Velocity \in \DerivativeMap[\Configuration](\Velocity)$ be $\DissipationRate[\Configuration](s)$-dissipative. Then if $\Velocity(s) \in \SolutionSet{\DerivativeMap[\Configuration]}[\Brackets{0, s^*}]$ and $\Velocity([0, s^*]) \subseteq \Closure\ActiveSet(\Configuration)$,
		$$s^* \leq \Parentheses{ \inf_{s>0} \frac{s}{\DissipationRate[\Configuration](s)}} \Norm{\Velocity(0)}_{\Mass}.$$
\end{lemma}

Under Assumption \ref{assump:nondegenerate}, $\dot\Velocity \in \DerivativeMap[\Configuration](\Velocity)$ exhibits $\DissipationRate(s)$-dissipativity for \textit{every} $\Configuration \in \ConfigurationSet_A \setminus \ConfigurationSet_P$, a direct proof of \Cref{prop:nondegeneratedissipation}: 
\begin{theorem}[Aggregate Dissipation (Appendix \ref{adx:nondegeneratedissipationproof})]\label{thm:nondegeneratedissipation}
	For every configuration $\Configuration \in \ConfigurationSet_A \setminus \ConfigurationSet_P$ there exists an $\DissipationRate[\Configuration](s)$ such that $\dot\Velocity \in \DerivativeMap[\Configuration](\Velocity)$ is $\DissipationRate[\Configuration](s)$-dissipative.
\end{theorem} 

The u.s.c. structure of $\DerivativeMap[\Configuration]$ has the additional implication that nearby configurations obey a uniform dissipation rate:

\begin{corollary}[Uniform Aggregate Dissipation (Appendix \ref{adx:uniformproof})]\label{coro:uniform}
	For compact $\ConfigurationSet \subseteq \ConfigurationSet_A \setminus \ConfigurationSet_P $, there exists a single $\DissipationRate[\ConfigurationSet](s)$ such that $\dot\Velocity \in \DerivativeMap[\Configuration](\Velocity)$ is $\DissipationRate[\ConfigurationSet](s)$-dissipative for all $\Configuration \in \ConfigurationSet$.
\end{corollary}
\section{Continuous-Time Dynamics Model}
\label{section:continuous_time_model}
We now describe how the simultaneous impact DI can be embedded into a full, continuous-time dynamics model.
As the impact model integrates over a variable other than time, rather than switching between integration spaces, we define time advancement $t$ as a variable in an augmented state $\bar \State(s)$:
\begin{equation}
	\bar \State(s) = \begin{bmatrix}
		\State(s) \\ t(s) \end{bmatrix} = \begin{bmatrix}
		\Configuration(s) \\ \Velocity(s) \\ t(s) \end{bmatrix} \in \Real^{\Configurations+\Velocities+1}\,.\label{eq:AugmentedStateDefinition}
\end{equation}
For any state $\bar \State(s)$ we can extract the relevant configuration, velocity, and time as by selecting the appropriate indices, e.g. as $\Configuration(\bar \State(s))$.
For notational compactness, whenever clear, we will write this construction in the shortened form $\Configuration(s)$. We will also frequently make use of the sets
\begin{align}
	\bar \StateSet_A = \Braces{\bar \State : \Configuration(\bar \State) \in \ConfigurationSet_A},\quad\label{eq:ActiveStateSetDefintion} &
	\bar \StateSet_P = \Braces{\bar \State : \Configuration(\bar \State) \in \ConfigurationSet_P}.
\end{align}
\subsection{Model Construction}
We now construct the dynamics model as a differential inclusion $\TotalDiff{}{s} \bar \State(s) \in \DerivativeMap(\bar \State(s))$.
Under this formulation, the velocity $\Velocity(s)$ is continuous with respect to $s$, but can be discontinuous with respect to time $t(s)$ in the sense that $\Velocity$ can evolve while $t$ is held constant.
To make the system autonomous, we represent the external forces $\Input$ as set-valued, time-varying full-state feedback $\InputSet(\bar \State)$.
In order for the system to be well-behaved, we assume that the convex-compact u.s.c. properties exploited in the impact dynamics carry over into the continuous time case:
\begin{assumption}\label{assump:ConvexCompactUSCInputs}
	$\NetForce[s](\State,\InputSet(\bar\State))$ is convex-compact u.s.c. in $\bar \State$.
\end{assumption}
\noindent We identify three behaviors that $\dot {\bar \State} \in \DerivativeMap(\bar \State)$ should obey:
\subsubsection{No Contact Forces}\label{subsubsec:NoContact} Whenever all active contacts have separating velocities (and when no contacts are active), i.e.
\begin{equation}
	\bar \State(s) \in \SeparatingStateSet = \Braces{\bar \State: \Velocity(\bar \State) \in \InactiveSet (\Configuration(\bar \State)) }\label{eq:SeparatingStateSetDefintion}  \,,
\end{equation}
$\bar \State(s)$ should evolve according to \eqref{eq:ManipulatorEquations} with no contact forces ($\Force = \ZeroVector$), in the sense that
\begin{subequations}
	\begin{align}
	\Mass(\Configuration)\Differential \Velocity &\in \NetForce[s]\Parentheses{\State, \InputSet(\bar \State)}\Differential s \,, \\ \Differential \Configuration &= \GeneralizedVelocityJacobian (\Configuration) \Velocity\Differential s\,, \\
		\Differential s &= \Differential t \,.
\end{align}
\end{subequations}
These equations can be packaged into DI form as
\begin{equation}
	\dot{\bar \State} \in \SeparatingStateDerivative(\bar \State) = {\begin{bmatrix}
		\GeneralizedVelocityJacobian \Velocity \\
		\Mass^{-1} \NetForce[s](\State,\InputSet(\bar\State)) \\
		1
	\end{bmatrix}}\,.
\end{equation}
\subsubsection{Collision}\label{subsubsec:Impact} Whenever $\Velocity(s)$ is colliding over $[a,b]$, i.e.
\begin{equation}
	\bar \State([a,b]) \subseteq \ImpactingStateSet = \Braces{\bar \State: \Velocity(\bar \State) \in \ActiveSet (\Configuration(\bar \State)) }\,,\label{eq:ImpactingStateSetDefintion}
\end{equation}
$t$ and $\Configuration$ should be constant, and $\Velocity$ should obey our simultaneous impact model:  
\begin{equation}
	\dot{\bar \State} \in \ImpactingStateDerivative (\bar \State) = \begin{bmatrix}
		\ZeroVector \\ 
		\DerivativeMap[\Configuration] (\Velocity) \\
		0
	\end{bmatrix}\,.
\end{equation}

\subsubsection{Sustained Contact}\label{subsubsec:SustainedContact} The model must capture continuous state evolution with respect to time under sustained contact, as in \eqref{eq:ManipulatorEquations}.
Additionally, proving that our model is well-behaved requires that $\DerivativeMap(\bar \State(s))$ be convex.
Conveniently, sustained contact can be represented as a convex combination of contactless and collision dynamics:
\begin{equation}
	\dot{\bar \State}(s) \in \Hull\Parentheses{ \SeparatingStateDerivative(\bar \State) \cup \ImpactingStateDerivative(\bar \State)}\,.\label{eq:SustainedContactInclusion}
\end{equation}
To demonstrate this property, we consider that \eqref{eq:ManipulatorEquations} dictates that the state $\Configuration$, $\Velocity$ under sustained contact obeys
\begin{align}
	\Differential \Configuration &= \GeneralizedVelocityJacobian \Velocity \Differential t\,, &
	\Mass \Differential\Velocity &\in \Parentheses{\J^T\Force + \NetForce[s]}\Differential t \,,\label{eq:SustainedEvolution}
\end{align}
for finite, non-zero contact forces $\Force = [\NormalForce;\;\FrictionForce]$.
Letting $\tilde \Force = \frac{\Force}{\Norm{\NormalForce}_1}$,
our impact model would allow $\Mass\Differential\Velocity \in \J^T\tilde\Force\Differential s$ at $\Configuration, \Velocity$.
Thus selecting $\dot t = \frac{1}{1 + \Norm{\NormalForce}_1} \in (0,1)$, we rewrite \eqref{eq:SustainedEvolution} as
\begin{subequations}
\begin{align}
	\Differential \Configuration &= ((1-\dot t)\ZeroVector + \dot t\GeneralizedVelocityJacobian \Velocity )\Differential s\,,\label{eq:SustainedConfigurationRearrangement} \\
	\Mass\Differential\Velocity &\in \Parentheses{(1-\dot t)\J^T\tilde \Force + \dot t \NetForce[s](\State,\InputSet(\bar\State))}\Differential s \,,\\
	\Differential t & = ((1-\dot t)0 + \dot t1)\Differential s\,.\label{eq:SustainedTimeRearrangement}
\end{align}
\label{eq:SustanedEvolutionRearrangement}
\end{subequations}
The convex combination DI \eqref{eq:SustainedContactInclusion} can then generate sustained contact with this choice of $\dot t$.
As a result, $t(s)$ neither evolves directly with $s$ nor remains constant; effectively, solutions of \eqref{eq:SustainedContactInclusion} slow down time by a factor of $(1 + \Norm{\NormalForce}_1)$. We will show that this factor is bounded on average under mild assumptions.

We now combine these three modes into a single differential inclusion.
While we might easily choose the contactless mode when $\bar \State \in \SeparatingStateSet$, switching between impact and sustained contact when the velocity is non-separating is less obvious, particularly as Painlev\'e's Paradox (see \citet{Stewart2000} for details) might require impact dynamics even without a collision (IWC's).
Furthermore, almost all selections of $\dot{\bar \State}$ from $\Hull\Parentheses{ \SeparatingStateDerivative(\bar \State) \cup \ImpactingStateDerivative(\bar \State)}$ will correspond to non-physical behavior; a particular $\dot{\bar \State}$ must be chosen to \textit{maintain} contact by exactly counteracting forces such that inter-body distance is \textit{identically} zero during contact.
In the subsequent section, we will prove that each of these behaviors correctly emerges in the following full DI model:
\begin{equation}\label{eq:ContinuousTimeModel}
	\dot{\bar \State} \in \DerivativeMap(\bar \State) = \begin{cases}
		\SeparatingStateDerivative(\bar \State) & \bar \State \in \SeparatingStateSet\,, \\
		\ImpactingStateDerivative(\bar \State) & \bar \State \in \Interior (\ImpactingStateSet)\,, \\
		\Hull\Parentheses{ \SeparatingStateDerivative(\bar \State) \cup \ImpactingStateDerivative(\bar \State)}& \textrm{otherwise}\,.\\
	\end{cases}
\end{equation}
By including $\ImpactingStateDerivative(\bar \State)$ in the right hand side whenever $\bar \State$ is not separating, \eqref{eq:ContinuousTimeModel} by construction allows IWC's to occur.
We will show that in this model, $\Gap(\Configuration) = \ZeroVector$ is effectively a \textit{barrier}: solutions beginning at a non-penetrating configuration are forced to \textit{never} penetrate.
Thus, under proven existence of solutions, the model will switch between sustained contact and impacts (possibly without collision) as necessary.
\subsection{Properties}\label{subsec:ContinuousTimeProperties}
\subsubsection{Existence and Closure}
As we previously reviewed, existence guarantees for continuous-time evolution through impact have thus far been severely limited.
We now show that our philosophy of including a wide set of behaviors leads to existence of solutions via \Cref{prop:closure}, and the only additional assumptions required are that energy and inputs are bounded (Assumptions \ref{assump:EnergyIncreasesSlowly} and \ref{assump:DerivativeMapCompactImages}).
The continuous-time DI \eqref{eq:ContinuousTimeModel} directly exhibits many of the properties required for \Cref{prop:closure}.
B its construction, at any $\bar \State$, $\DerivativeMap(\bar \State)$ is non-empty, compact, and convex.
We will additionally see that it is u.s.c. in our proof of \Cref{thm:ContinuousTimeSolutions}.
However, as Coriolis components of $\NetForce[s]$ can grow quadratically, $\DerivativeMap$ is often not uniformly bounded; thus \Cref{prop:closure} cannot be directly used to prove existence of solutions.
However, nearly identical properties of IVP's can still be established in the following manner.
Suppose first that smooth forces can only input power at a bounded rate: 
\begin{assumption}\label{assump:EnergyIncreasesSlowly}
	$\exists c > 0$, $\Velocity \cdot \NetForce[s](\State,\InputSet(\bar\State)) \leq c \Norm{\Velocity}_{\Mass}$.
\end{assumption}
This condition is widely satisfied by many robotic systems, including those with globally bounded controllers and potential gradients (such as gravity).
Assumption \ref{assump:EnergyIncreasesSlowly} implies that $\bar \State$ cannot diverge to infinity over a finite horizon.
Furthermore, we will assume that if $\bar \State$ is bounded, $\dot{\bar \State}$ is bounded as well:
\begin{assumption}\label{assump:DerivativeMapCompactImages}
	Over any compact set $\bar \StateSet$, $\NetForce[s](\State,\InputSet(\bar\State))$ is bounded, and therefore $\DerivativeMap(\bar \StateSet)$ is compact.
\end{assumption}
 Assumptions \ref{assump:EnergyIncreasesSlowly} and \ref{assump:DerivativeMapCompactImages} imply that over a finite interval, the solutions $\bar \State(s)$ beginning from a compact set $\bar \StateSet$ have bounded derivative and therefore inherit the key existence, closure, and u.s.c. structure of globally bounded DI's:
\begin{theorem}[Existence of Solutions (Appendix \ref{adx:ContinuousTimeSolutionsProof})]\label{thm:ContinuousTimeSolutions}
	Let $\bar \StateSet$ be a compact set and $[a,b]$ be a compact interval. Then $\IVP{\DerivativeMap}{\bar \StateSet}{[a,b]}$ is compact and $\IVP{\DerivativeMap}{\bar \State}{[a,b]}$ is non-empty, closed, convex, and u.s.c. in $\bar\State$ over $\bar \StateSet$.
\end{theorem}

\begin{figure}[h]
    \center        
    \begin{subfigure}[b]{.38\linewidth}
    	\raisebox{5mm}{
        \includegraphics[width=\hsize,trim={0mm 24mm 0mm 3mm},clip]{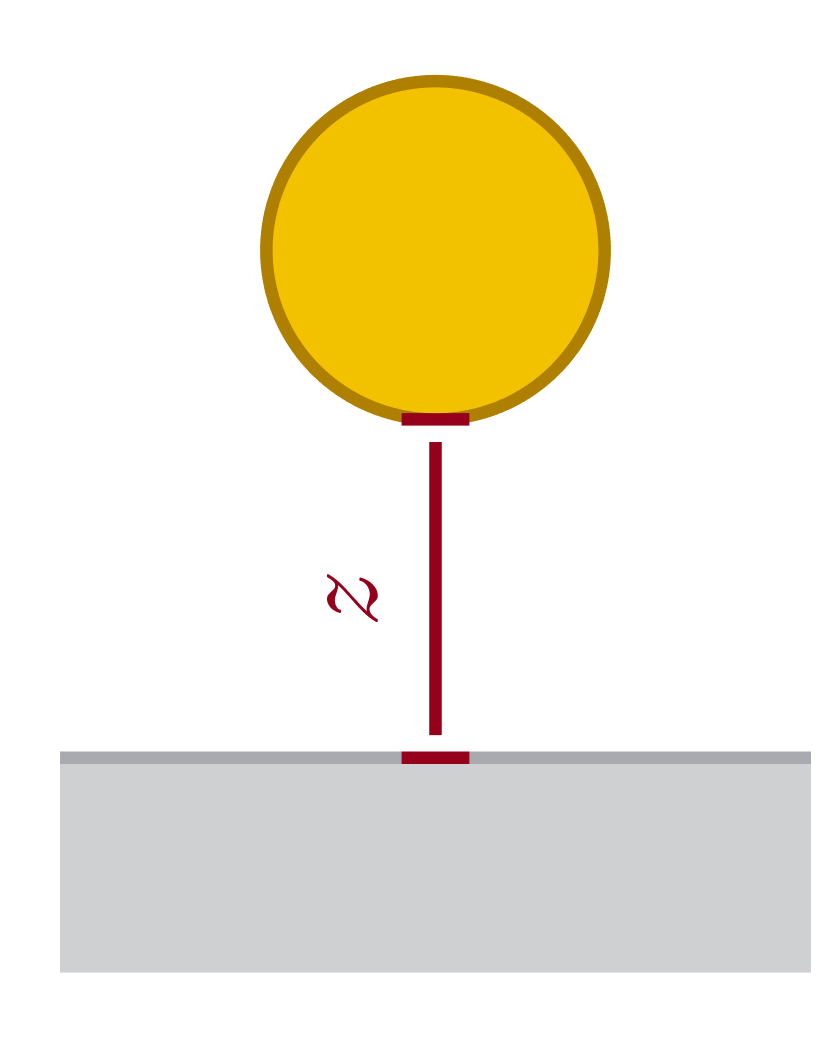}
        \centering
        }
        \caption{\label{fig:ballfall} 1D ball-ground system}
    \end{subfigure}
    \begin{subfigure}[b]{.60\linewidth}
    	\centering
        \includegraphics[width=\hsize,,trim={1mm 0mm 14mm 0mm},clip]{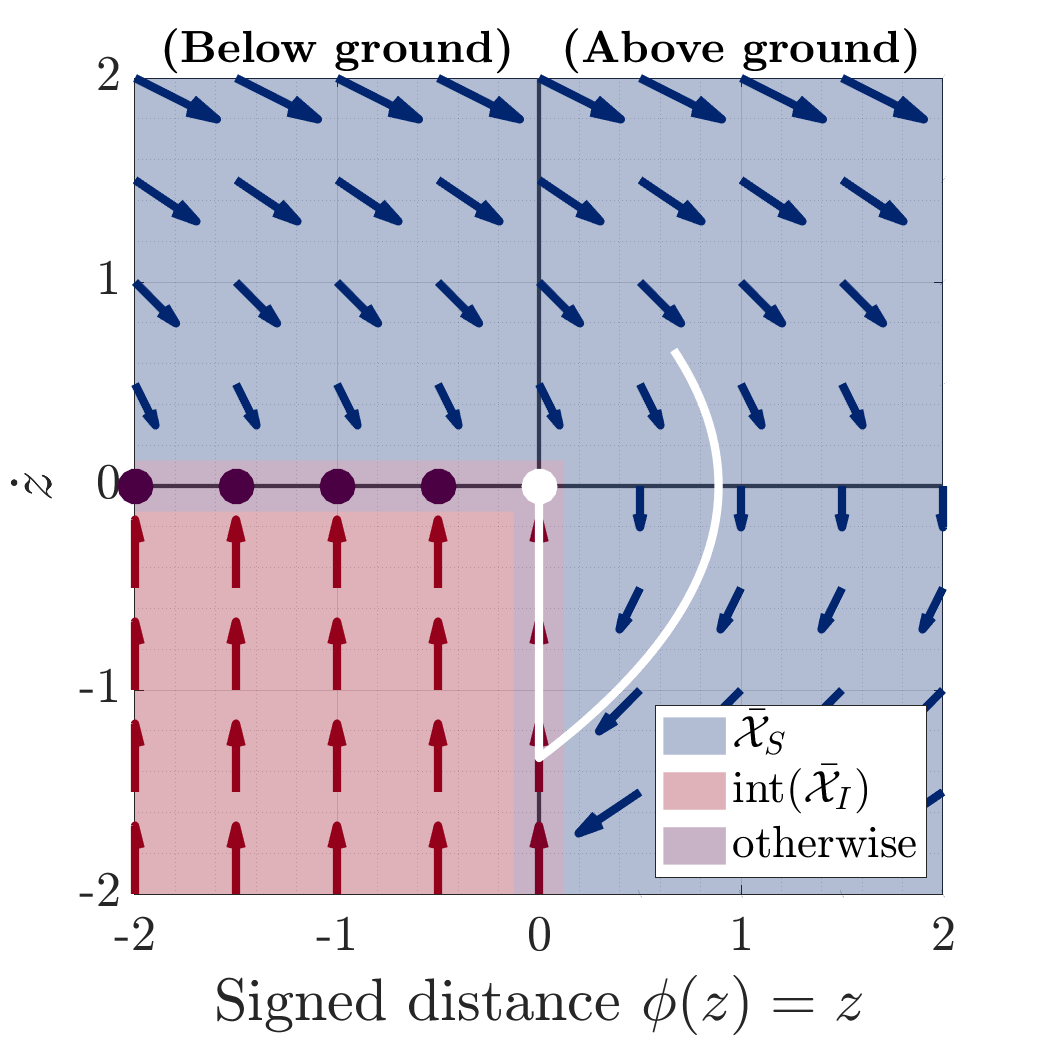}\\
        \caption{\label{fig:ball_phase} 1D system phase portrait}
    \end{subfigure}
    \caption{(\subref*{fig:pushing_cartoon}) A simple, 1D system of a non-rotating ball falling under gravity with configuration $\Configuration = z = \Gap(z)$ is shown. 
        (\subref*{fig:ball_phase}) A phase plot demonstrates why our DI prevents penetration; an example trajectory is shown in white. Penetration corresponds to crossing from the right-half- to the left-half-plane. This cannot happen on the top half of the vertical axis ($\Gap = 0$,  $\dot z \geq 0$), as the flow by definition points right. The cross also cannot happen on the bottom half of the axis, as the quadrant III has purely-vertical flow ($\Differential \Gap = \dot z\Differential t = 0$).}
    \label{fig:penetrationproof}
\end{figure}

\subsubsection{Non-Penetration} While there is no structure in $\DerivativeMap(\bar \State)$ that explicitly prevents penetration, $\Gap(\Configuration) \geq \ZeroVector$ is naturally, implicitly preserved, as the DI requires $\Gap_i(\Configuration(s))$ to be constant under penetration.
A graphical argument is given in Figure \ref{fig:penetrationproof}.
\begin{lemma}[Non-Penetration (Appendix \ref{adx:ContinuousTimeDoesNotPenetrateProof})]\label{thm:ContinuousTimeDoesNotPenetrate}
	Let $\bar \State_0 \not \in \bar \StateSet_P$ be non-penetrating, let $[a,b]$ be a compact interval, and let $\bar \State(s) \in \IVP{\DerivativeMap}{\bar \State_0}{[a,b]}$. Then $\bar \State(s) \not \in \bar\StateSet_P$ for all $s \in [a,b]$.
\end{lemma}
\subsubsection{Correct Mode Selection}
Our requirements dictate that solutions $\bar \State(s)$ containing only separating velocities ($\bar \State \in \SeparatingStateSet$) should comply with contactless dynamics, and likewise with impact dynamics when $\bar \State(s)$ contains only colliding velocities and non-penetrating configurations ($\bar \State \in \ImpactingStateSet \setminus \bar \StateSet_P$).
The former is a trivial result of the construction of $\DerivativeMap$, but the latter is only similarly trivial when $\bar \State \in  \Interior (\ImpactingStateSet) \setminus \bar \StateSet_P$. However, all states $\bar \State \subseteq \ImpactingStateSet$ have penetrating velocity, and thus any contactless dynamics component in $\dot{\bar \State}$ would by definition cause $\bar \State$ to penetrate (i.e. enter $\bar \StateSet_P$), allowing a proof by contradiction:
\begin{lemma}[Impact Dynamics (Appendix \ref{adx:ContinuousTimeImpactSelectionProof})]\label{thm:ContinuousTimeImpactSelection}
	Let $[a,b]$ be a compact interval and $\bar \State(s) \in \SolutionSet{ \DerivativeMap}[[a,b]]$ with $\bar \State([a,b]) \subseteq \ImpactingStateSet \setminus \bar \StateSet_P $. Then $\bar \State(s) \in \SolutionSet{ \ImpactingStateDerivative}[[a,b]]$.
\end{lemma}

\subsubsection{Linear Time Advancement}
While \Cref{thm:ContinuousTimeSolutions} guarantees existence of solutions over any interval of $s$, practical application often requires reasoning about solution sets over intervals in time (over $t(s)$).
To do so, solutions of the model must significantly advance time---i.e. for any time duration $t_f$, all solutions of the model have $t(s_f) - t(0) > t_f$ for large enough $s_f$.
For small enough $t_f$, this property only requires the solution to exit the impact dynamics regime, which by \Cref{thm:nondegeneratedissipation} is guaranteed to occur:
\begin{theorem}[Time Advancement (Appendix \ref{adx:ContinuousTimeMinimumAdvancementProof})]\label{thm:ContinuousTimeMinimumAdvancement}
	Let $\bar \StateSet \subseteq \Complement{\bar \StateSet_P}$ be a compact set with no penetrating configurations. Then there exists $s^*(\bar \StateSet), t^*(\bar \StateSet) > 0$, such that for all $s_f > s^*(\bar \StateSet)$, if $\bar \State(s) \in \IVP{\DerivativeMap}{\bar \StateSet}{[0,s_f]}$, then $t(s_f) - t(0) > t^*(\bar \StateSet)$.
\end{theorem}
If $t(s_f) - t(0) > t^*$ is guaranteed over a set $\bar \StateSet$, then $t(s)$ must at least advance at rate $\frac{t^*}{s^*}$ over arbitrarily long horizons:
\begin{corollary}[Amortized Advancement (Appendix \ref{adx:AggregateAdvancementProof})]\label{coro:AggregateAdvancement}
	Let $\bar \StateSet \subseteq \Complement{\bar \StateSet_P}$ be a compact set with no penetrating configurations, such that
	\begin{multline}
	\bar \StateSet (s_f) = \\ \Braces{\bar \State(\cdot) \in  \IVP{\DerivativeMap}{\bar \StateSet}{[0,s_f]} : \bar \State([0,s_f]) \subseteq \bar \StateSet}\,,
		\end{multline}
	is non-empty for all $s_f > 0$.
	Define $s^*(\bar \StateSet), t^*(\bar \StateSet) > 0$ as in \Cref{thm:ContinuousTimeMinimumAdvancement}, and let 
	\begin{equation}
		t_f(s_f) = \min_{\bar \State(\cdot) \in \bar \StateSet(s_f)} t(\bar \State(s_f)) - t(\bar \State(0))\,.
	\end{equation}
	Then $\lim \inf_{s_f \to \infty} \frac{t_f(s_f)}{s_f} \geq \frac{t^*(\bar \StateSet)}{s^*(\bar \StateSet)}$.
\end{corollary}

The results in this section guarantee that solutions to our model \eqref{eq:ContinuousTimeModel} are well-behaved and exist over arbitrary time horizons.
These results came with a number of structural assumptions on the involved terms in the manipulator equations, but ultimately provide a state-of-the-art result on consistency with relatively few assumptions.
We have argued that Assumptions \ref{assump:NoDegenerateContacts} and \ref{assump:ConvexCompactUSCInputs}--\ref{assump:DerivativeMapCompactImages} are satisfied by the large majority of robotic systems, while Assumption \ref{assump:nondegenerate} is also made in the preeminent solution existence results for MDI's \citep{Stewart1998}.
However, unlike the limited analysis of the multiple contacts case for MDI's, each solution of our DI is by definition compliant with inelastic and Coulomb friction constraints.

Under these assumptions, our model is able to make predictions in pathological scenarios, including Painlev\'e and Zeno behaviors.
We have seen how our model complies with Coulomb friction in the sustained contact case, and thus can capture Painlev\'e's ubiquitous example of a rod sliding on a flat surface with high friction \citep{Zhao2007,Stewart1998}.
As our model is guaranteed to have a solution over some time horizon for this system (\Cref{thm:ContinuousTimeSolutions}), the only possibility for this scenario is that our model generates an impact without collision.
As the resulting behavior is equivalent to Routh's (and therefore Darboux-Keller's) model for one contact, we refer the reader to \citep{Zhao2007} to learn more about the prediction of such models in this example.

\subsection{Zeno behavior example}\label{subsec:zenoexample}
\begin{figure*}[h]
    \center
     \begin{subfigure}[b]{.24\textwidth}
        \includegraphics[width=.8\hsize]{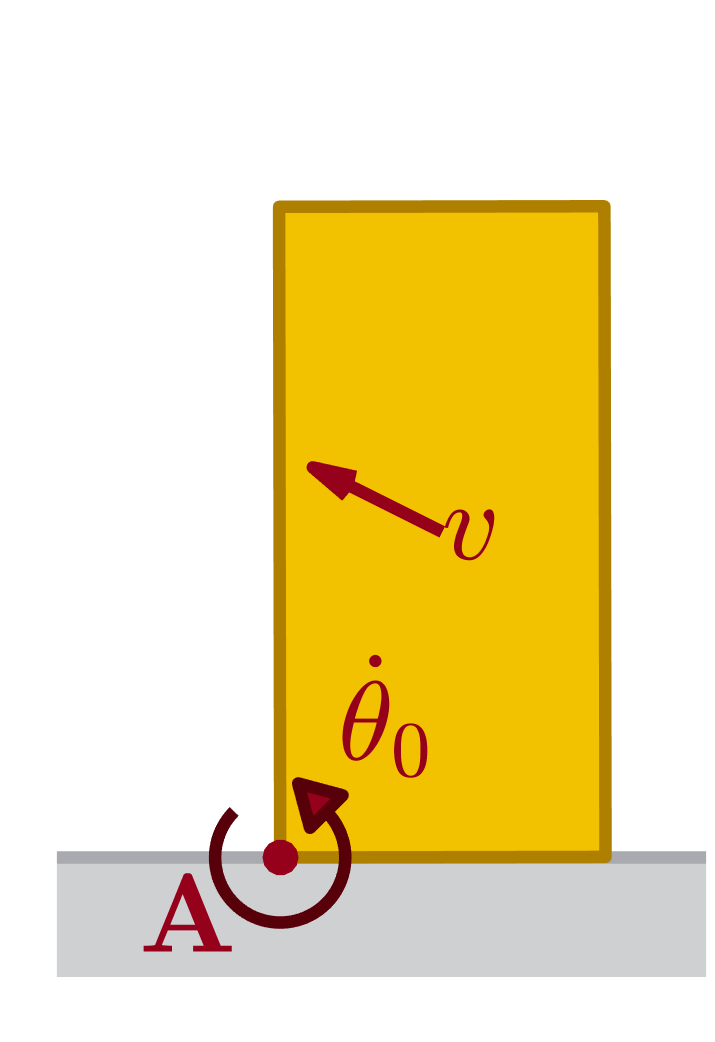}
        \centering
        \caption{\label{fig:zeno_0}Initial condition}
    \end{subfigure}
    \begin{subfigure}[b]{.24\textwidth}
        \includegraphics[width=.8\hsize]{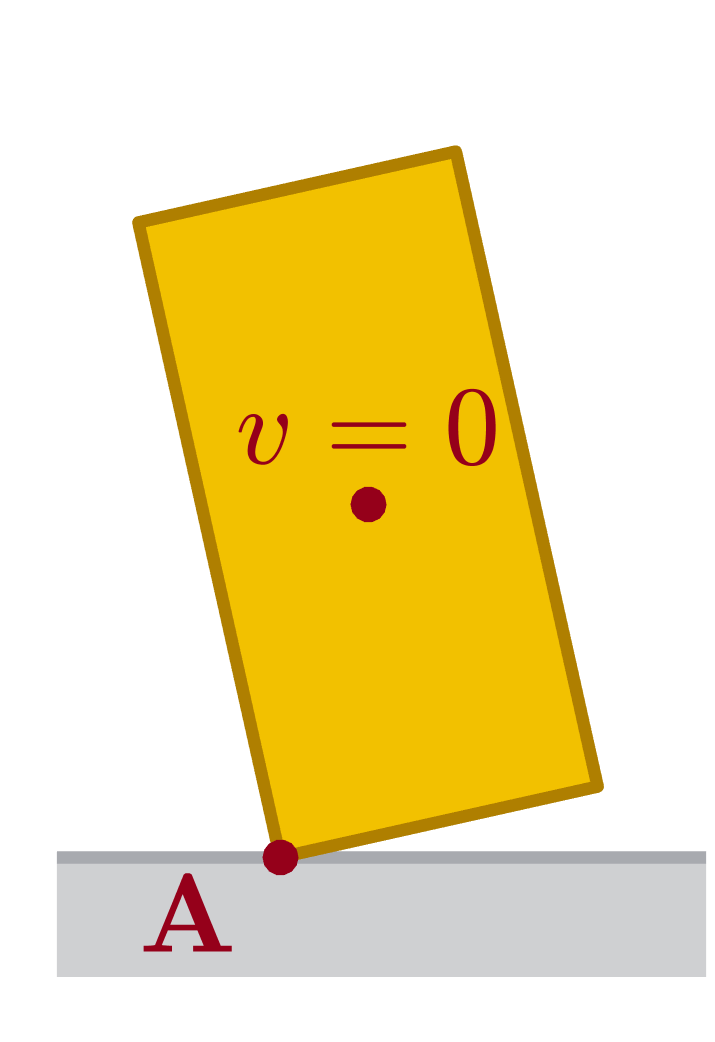}
        \centering
        \caption{\label{fig:zeno_apex}Wobble apex}
    \end{subfigure}
    \begin{subfigure}[b]{.24\textwidth}
        \includegraphics[width=.8\hsize]{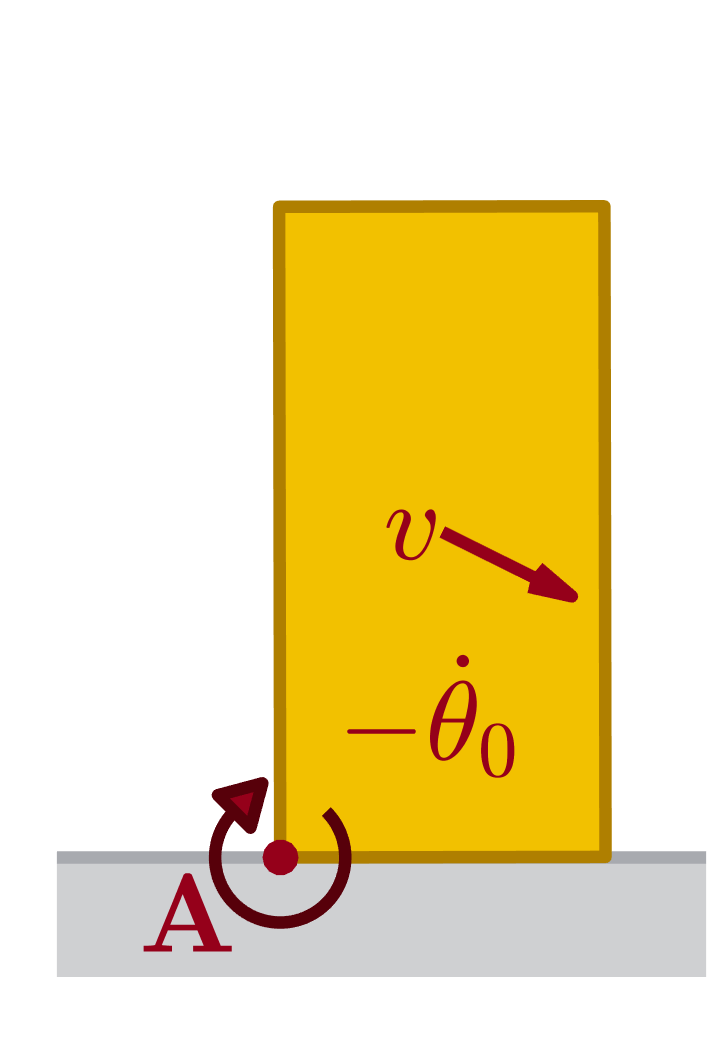}
        \centering
        \caption{\label{fig:zeno_pre}Before first impact}
    \end{subfigure}
    \begin{subfigure}[b]{.24\textwidth}
        \includegraphics[width=.8\hsize]{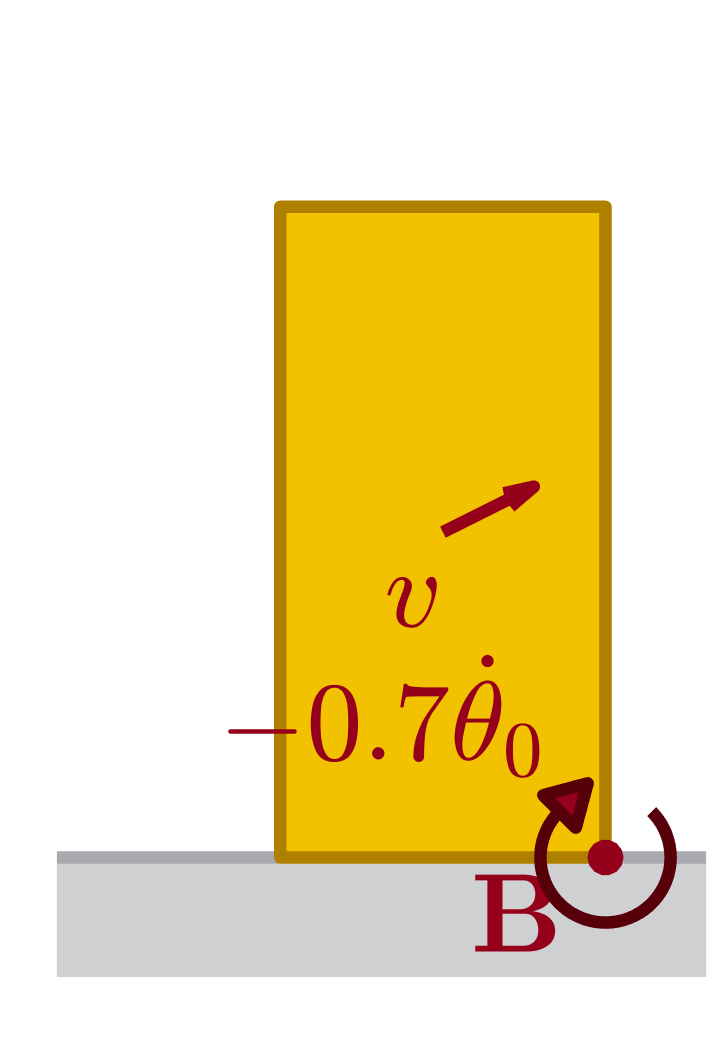}
        \centering
        \caption{\label{fig:zeno_post}After first impact}
    \end{subfigure}
    \caption{Half-cycle of rocking block Zeno trajectory, modified from \cite{Lygeros2003}. (\subref{fig:zeno_0}) At the beginning of the half-cycle, the block is pivoting about the bottom left corner $(\vect A)$, with angular velocity $\dot \theta_0 = 1$. (\subref{fig:zeno_apex}) Coulumb friction maintains sticking contact at $\vect A$, such that the block follows a pendulum trajectory, reaching an apex at $\theta^* \approx .22 \si{\radian}$. (\subref{fig:zeno_pre}) As the center of mass remained to the right of the pivot point $\vect A$, the block falls back down and by conservation of energy impacts the ground with angular velocity $\dot \theta = -\dot \theta_0$. (\subref{fig:zeno_post}) By conservation of angular momentum, the impact reduces the angular velocity by a factor of $0.7$.
    This end state is a mirror image of the initial condition, except with the angular velocity reduced by a factor of $0.7$.
    Therefore, another half cycle will return to the initial condition, except with a modified angular velocity $\dot \theta = (0.7)^2\dot \theta_0$.\label{fig:zeno}}
    \vspace{2mm}
    \begin{subfigure}[b]{.33\textwidth}
        \includegraphics[width=\hsize]{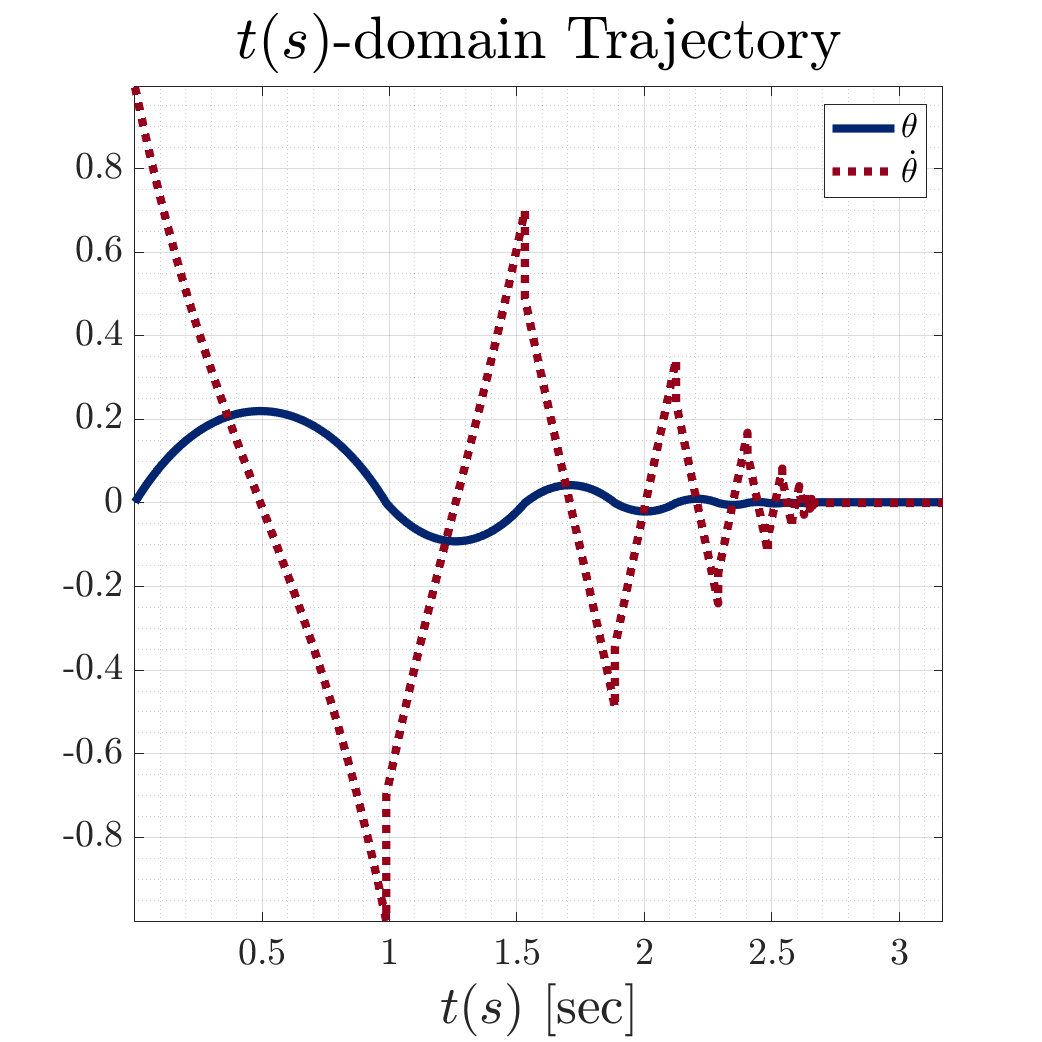}
        \centering
    \end{subfigure}
    \begin{subfigure}[b]{.33\textwidth}
        \includegraphics[width=\hsize]{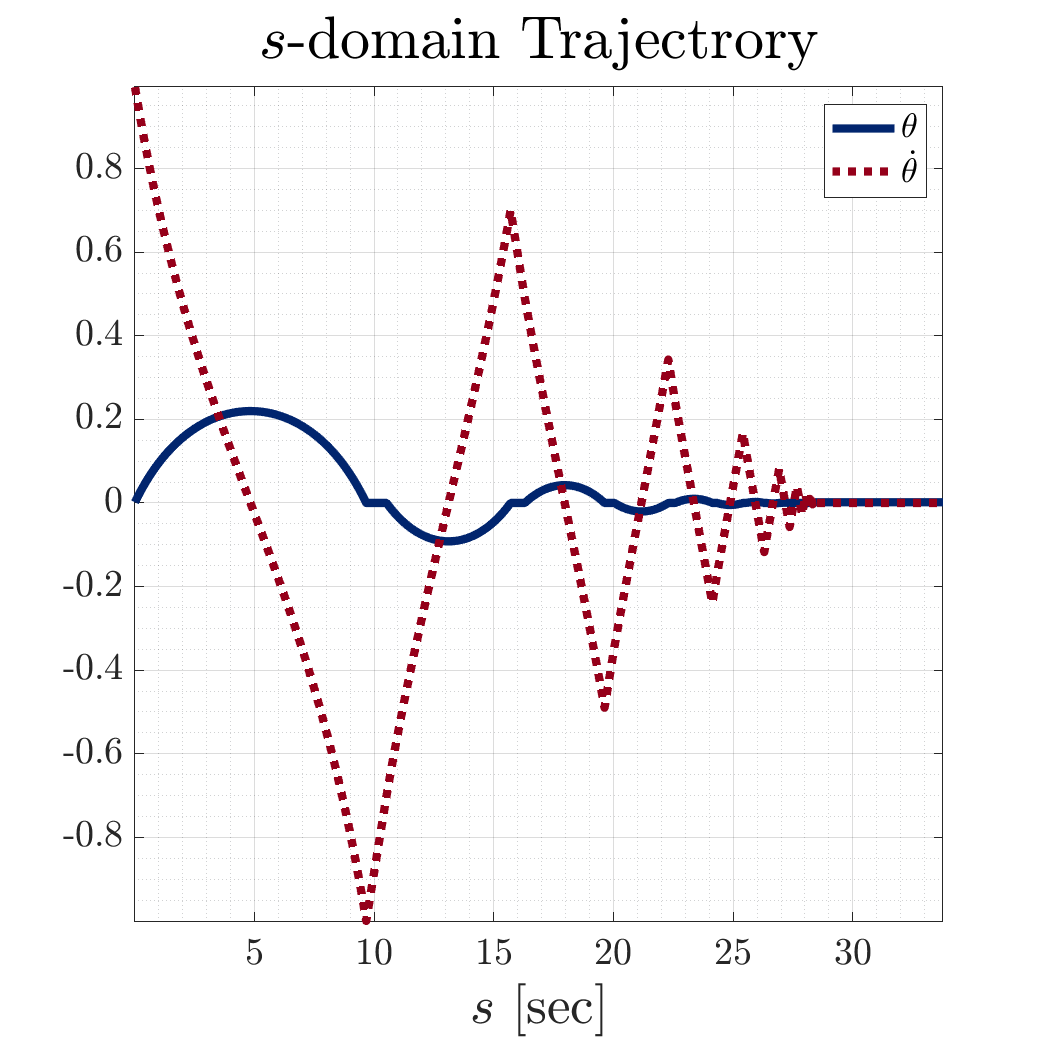}
        \centering
    \end{subfigure}
    \begin{subfigure}[b]{.33\textwidth}
        \includegraphics[width=\hsize]{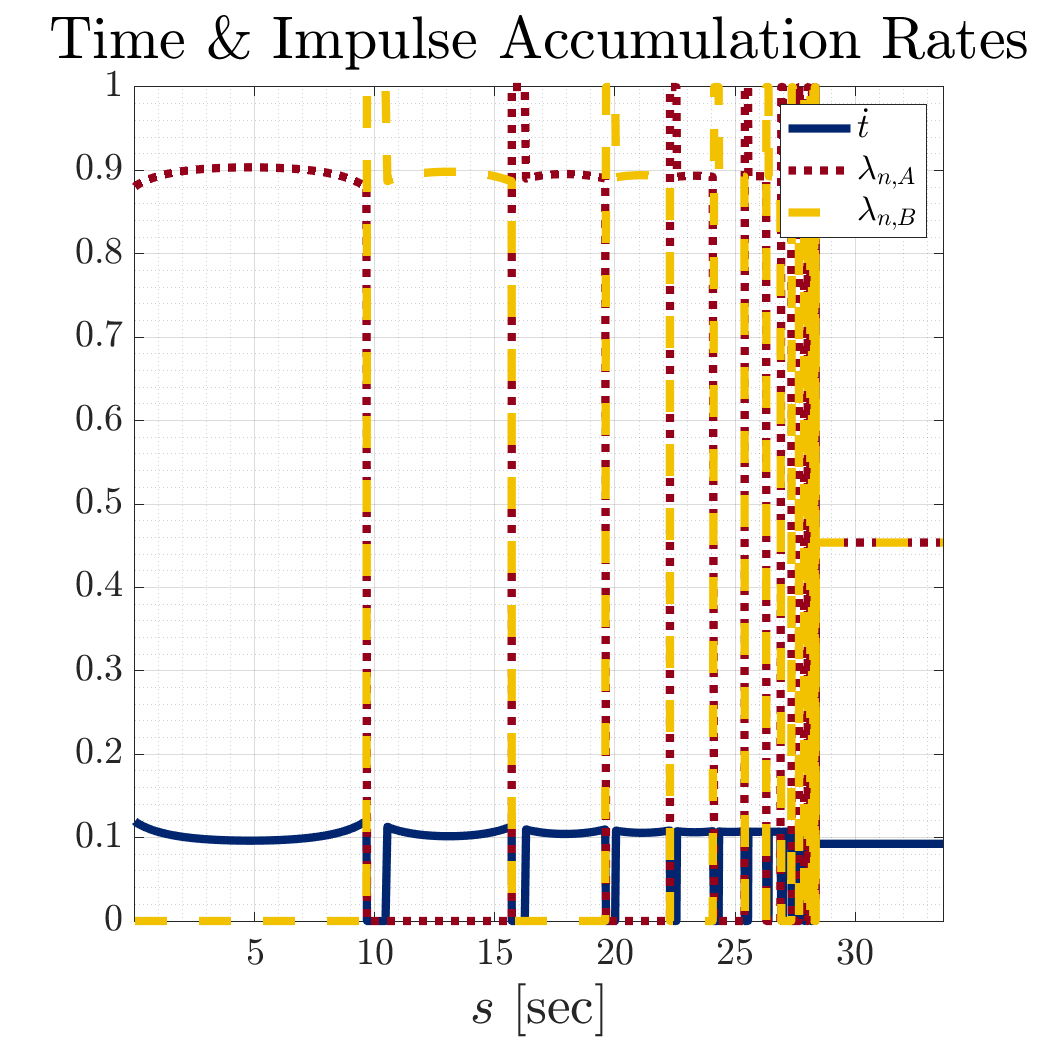}
        \centering
    \end{subfigure}
	\caption{\label{fig:zeno_trajectory}Trajectory of the wobbling trajectory of the rocking block system. Left: The evolution of $\theta$ and $\dot \theta$ are shown in the $t(s)$-domain. Beginning with $t(s) \approx 1$, there are a series of vertical jumps in the plot of $\dot \theta$, corresponding to impacts; the trajectory comes to rest after the Zeno accumulation point at $t(s^*) \approx 2.67$. Center: in the $s$ domain, the jumps in $\dot \theta$ are replaced with time-frozen impact, where $\theta$ is held constant. Due to the time slowing factor described by \Cref{eq:SustainedConfigurationRearrangement}, the $s$ domain is significantly longer, with the accumulation point occuring at $s^* \approx 28.4$. Right: the rates of accumulation for time ($\dot t$) and normal impulse ($\Force[A], \Force[B]$) are displayed.
	While the trajectory is absolutely continuous in $s$, these derivatives are only defined almost everywhere, with discontinuous switches during the transitions between impact and sustained contact.
	These switches become closer and closer together as $s$ reaches $s^*$.}
\end{figure*}
As discussed in \Cref{subsubsec:IVPs}, even inelastic contact can exhibit Zeno behavior.
In the remainder of this section, both to verify that our model \eqref{eq:ContinuousTimeModel}, captures Zeno behavior and to illustrate its solutions $\bar\State(s)$, we will examine an instance of the rocking block example of \citet{Lygeros2003}, where an alternate ``wobbling'' trajectory (\Cref{fig:zeno}) of the system described in \Cref{fig:phone} and \Cref{subsec:rockingblockexample} is considered.
The simplified setting will allow us to explicitly construct a solution $\bar\State(s)$ which exhibits Zeno behavior as well as transitions between sustained contact and impact modes.
The code used to generate the figures associated with the example is available online\endnote{\label{endnote:code}The codebase for this paper is available at \url{https://github.com/mshalm/routh-multi-impact}.
Examples related to \Cref{sec:examples} can be run by calling \texttt{Results()}; more details are available in \Cref{adx:exampledetails}.
Figures relating to the Zeno example are computed numerically by calling \texttt{ZenoBlock()}.}.

We consider a block of width $w = 1 \si{\meter}$; height $h = 2 \si{\meter}$; coefficient of friction $\FrictionCoeff = 1$ with the ground; and uniformly-distributed mass $m = 1 \si{\kilo\gram}$, with moment of inertia $I = \frac{1}{12}(w^2 + h^2)$.
The block has configuration $\Configuration = [x;y;\theta]$ composed of its center of mass position and angle with the horizontal.
The block begins by rotating about the bottom left corner on the ground (\Cref{fig:zeno_0}), with initial angular velocity $\dot \theta_0 = 1$ at time $t = 0$.

As $\mu > \frac{w}{h}$, Coulomb friction can maintain stiction during the rotating motion \citep{Zhang2001}.
The motion of the block is thus fully determined by its orientation $\theta$ which follows pendulum dynamics
\begin{equation}
	(I + m\TwoNorm{\vect r}^2) \frac{\Differential^2 \theta}{\Differential t^2} = -mg\TwoNorm{\vect r}\cos(\theta + \alpha)\,,\label{eq:pendulumODE}
\end{equation}
where $g=9.81$ is the gravitational acceleration; $\alpha = \arctan\Parentheses{\frac{h}{w}}$; and $\vect r = [r_x(\theta); r_y(\theta)]$ is the (world-frame coordinates) vector from the corner to the center of mass.
By conservation of energy, this motion will reach an apex at
\begin{equation}
	\theta^* = \arcsin\Parentheses{\frac{(I + m\TwoNorm{\vect r}^2)\dot\theta_0^2}{2mg\TwoNorm{\vect r}} - \sin(\alpha)} - \alpha \approx .22\,[\si{\radian}]\,,\label{eq:zenoapex}
\end{equation}
at which the center of mass remains to the right of the contact point.
Thus, again by energy conservation the block pivots back down to its initial position with angular velocity $-\dot\theta_0$ at some time $t_1$,
which by \eqref{eq:pendulumODE} and \eqref{eq:zenoapex} is no more than $L\dot\theta_0$, with
\begin{align}
	L &= \frac{2\dot\theta_0}{\ddot \theta_{min}}\,,\\
	\ddot \theta_{min} &= \frac{mg\TwoNorm{\vect r}\cos(\theta^* + \alpha)}{(I + m\TwoNorm{\vect r}^2)}\,.
\end{align}
The velocity of the center of mass during this period is
\begin{equation}
	\TimeDiff{}\begin{bmatrix}
		x \\ y
	\end{bmatrix} = \TimeDiff{\theta}\begin{bmatrix}
		-r_y \\ r_x
	\end{bmatrix}\,.
\end{equation}

The forces which maintain stiction can be determined by applying the manipulator equations \eqref{eq:ManipulatorEquations} along with the constraint that the acceleration of the pivot point $A$ is zero:
\begin{align}
	\TimeDiff{} \J[A] (\Configuration)\Velocity &= \J[A](\Configuration)\TimeDiff{\Velocity} + \TimeDiff{\J[A]}\Velocity = 0\,,\\
	\TimeDiff{\Velocity}  &= \Mass^{-1}\Parentheses{\begin{bmatrix}
		0 \\ -g \\ 0
	\end{bmatrix} + \J[A]^T\Force[A]}\,.
\end{align}
As $\J[A]\Mass^{-1}\J[A]^T \succ 0$ is bounded away from $0$ and the constituent functions are smooth in time, $\Force[A]$ is a smooth function of time during this period.
By examining \eqref{eq:SustanedEvolutionRearrangement}, we form a solution of the DI by reparameterization of time into $s$ with the relation $\Differential s = \dot t + \NormalForce[A]$, and thus
\begin{equation}
	s(t) = \int_0^t (1 + \NormalForce[A](\tau))\Differential\tau \,.
\end{equation}
$s(t)$ is smooth and strictly monotonically increasing with slope $\geq 1$ and thus has differentiable, continuous inverse $t(s)$ with slope $\dot t \in [0,1]$. Define
\begin{align}
	\Velocity(s) &= \TimeDiff{\theta}\begin{bmatrix}
		-r_y(\theta(t(s))) \\ r_x(\theta(t(s))) \\ 1
	\end{bmatrix}\,,\\
	\Configuration(s) &= \Configuration(0) + \int_0^{s} \dot t(s) \Velocity(\bar s)\Differential\bar s\,.
\end{align}
$t(s)$ can thus be composed with $\Configuration(s), \Velocity(s)$ to form a a solution to the DI \eqref{eq:ContinuousTimeModel}.
Once the bottom right corner (point $\vect B$) comes back down to the ground, the DI can admit a sticking impact at $\vect B$ starting integration value $s_1 = s(t_1)$ (\Cref{fig:zeno_pre,fig:zeno_post}).
As $\Jt[B]\Velocity(s_1) = 0$, an impact which sticks the entire time (and thus any impulse accumulation rate $\NormalForce[B] \geq \FrictionCoeff |\FrictionForce[B]|$) is allowed by the DI.
We can determine the total impulse $\Impulse[B]$ which brings point $\vect B$ to rest by solving the impact equation \eqref{eq:ImpulsiveNetwonsSecond}
\begin{equation}
0 = \J[B]\Parentheses{\Velocity(s_1) + \Mass^{-1}\J[B]^T\Impulse[B]}\,.
\end{equation}
We note that $\NormalImpulse[B]$ is homogeneous in $\Velocity(s_1)$.
We can thus extend the DI solution from the pendulum phase through the impact by freezing time and configuration and setting
\begin{equation}
	\Velocity(s) = \Velocity(s_1) + \frac{s-s_1}{\NormalImpulse[B]}\Mass^{-1}\J[B]^T\Impulse[B]\,,
\end{equation}
on $s \in [s_1, s_1 + \NormalImpulse[B]]$.
By conservation of angular momentum, the post-impact velocity is pivoting about point $\vect B$, in a mirror image to the initial condition (\Cref{fig:zeno_post}), with angular velocity
\begin{equation}
	-\dot \theta_0\frac{I + m\TwoNorm{\vect r}^2(-\cos(2\alpha))}{I + m\TwoNorm{\vect r}^2} = -0.7\dot \theta_0\,.
\end{equation}

By symmetry, the block with then have a mirror-image pendulum motion of time/integration duration no more than $0.7t_1$ and $0.7s_1$, causing another impact to again pivot about point $A$ with angular velocity $(0.7)^2\theta_0$.
Thus, we can infinitely repeat this cycle to construct a solution to the DI with the total time/integration duration reaching finite accumulation points
\begin{align}
	s^* &\leq \sum_{i=0}^\infty 0.7^i(s_1 + \NormalImpulse[B])\,,\\
	t(s^*) &\leq \sum_{i=0}^\infty 0.7^i t(s_1)\,.
\end{align}
As $\Velocity$ and $\Configuration$ are continuous on $[0,s^*)$, we can take the limit as $s \to s^*$ to extend to the accumulation point as $\Velocity(s^*) = \ZeroVector$ and $\Configuration(s^*)$ is at rest with both points $\vect A$ and $\vect B$ on the ground.
This is clearly an equilibrium of the DI model, where normal forces of $0.5mg$ at each contact point keep the block at rest.
Thus $\Velocity(s) = \ZeroVector$, $\Configuration(s) = \Configuration(s^*)$, $t(s) = t(s^*) + \frac{s - s^*}{1 + mg}$ is a valid continuation of the solution past $s^*$.
A visualization of this trajectory is given in both $t(s)$ and $s$ domains in \Cref{fig:zeno_trajectory}, with a supplemental figure displaying the relative accumulations rates of $t$, $\Impulse[A]$, and $\Impulse[B]$ with respect to $s$.


\section{Discrete Impact Integration}
\label{section:simulation}
Section \ref{section:model} provides a rigorous theoretical framework guaranteeing existence of solutions to our impact model.
In this section, we now develop a computational framework that allows this model to be applied to two key settings.

In the first, we consider embedding this model into an ``event-based'' discrete-time simulation environment, such as the one developed in \citet{Anitescu97}, where collisions are resolved instantaneously via an update function
\begin{equation}
	\Velocity^+ \gets \mathrm{ImpactLaw}(\Configuration,\Velocity)\,.
\end{equation}
Faithful capture of the non-uniqueness in our model means that $\mathrm{ImpactLaw}(\Configuration,\Velocity)$ should be capable of returning a set of different values for $\Velocity^+$; simulation will be considered as sampling stochastically from this set.
Our second application, reachability analysis, is to approximate the entire set of possible outcomes for a given initial condition.
In pursuit of both applications, we develop an LCP-based integration scheme for our impact DI \eqref{eq:multicontact}.
We will bound the number of LCP solves required for each of these applications.
We conclude the section with several numerical examples of post-impact set approximation, and provide comparisons to other complaint and rigid impact resolution methods.

We note that this section is purely focused on resolving an impact event with $\mathrm{ImpactLaw}(\Configuration,\Velocity)$, rather than the integration of this subroutine into e.g. a particular event-based simulator.
Extensive analysis on when the impact update should be triggered and whether every post-impact velocity is suitable for every event-based simulation scheme is therefore excluded, but we offer some brief discussion here.
For instance, most event-based simulators are vulnerable to Zeno behaviors, as infinite impacts would require infinite runs of the $\mathrm{ImpactLaw}$ algorithm.
Secondly, the post-impact termination used herein will simply be that the velocity is non-colliding.
While this is not sufficient e.g. to avoid immediately triggering another impact in Painlev\'e-type scenarios, some simulators such as \citet{Anitescu97} will successfully step forward in time if this condition is met.
Finally, many simulators such as \citet{Anitescu97} only trigger $\mathrm{ImpactLaw}$ at a collision, and will thus never trigger a collision under grazing.
We assume that the logic for handling such events is appropriately handled outside of the impact resolution scheme.

\subsection{Model Construction}
Just as forward Euler integration can cause penetration in continuous-time simulators \citep{Halm2018}, it can also break the inelastic condition if applied to our impact model \eqref{eq:multicontact}.
To rectify this issue, we develop an approximate, implicit, and discrete integration scheme.
Our method takes a simulation step by finding a small contact impulse increment $\Force$:
\begin{subequations}
\begin{align}
	\textrm{find} &&\Velocity'; \Braces{\Force[i]:  i \in \ContactSet_A}\,, \\
	\textrm{s.t.}&& \Mass(\Velocity'-\Velocity) = \J^T\Force\,,\\
	&& \NormalForce[i] \geq \ZeroVector \textrm{ and } \NormalForce[i]\Jn[i]\Velocity' \leq 0 \label{eq:ImplicitInelasticConstraint}\,,\\
	&& \FrictionForce[i] \in -\FrictionCoeff[i]\NormalForce[i]\Unit_{\vect D}(\Jt[i]\Velocity')\label{eq:ImplicitCoulombConstraint}\,, 
\end{align}\label{eq:SimulationUpdate}
\end{subequations}
where the dependence of $\J, \Mass$ on $\Configuration = \Configuration_0$ is suppressed for notational compactness.

Conceptually, our simulation scheme can be understood as differential, as it closely mirrors our impact DI \eqref{eq:multicontact} which selects $\dot \Velocity \in \Mass^{-1}\FrictionCone (\Configuration)$.
The primary changes are that \eqref{eq:SimulationUpdate} approximates the derivative with a finite difference $(\Velocity' - \Velocity)$; enforces Coulomb friction \eqref{eq:ImplicitCoulombConstraint} and inelasticity \eqref{eq:ImplicitInelasticConstraint} at the incremented velocity $\Velocity'$; and replaces Coulomb friction with its linear approximation.
However, computationally, our method seems most similar to the LCP-based algebraic method \eqref{eq:AnitestcuDiscreteFormulation}, and we will show the each increment \eqref{eq:SimulationUpdate} can also be solved as an LCP.
Despite this similarity, the are significant philosophical differences between \eqref{eq:AnitestcuDiscreteFormulation} and \eqref{eq:SimulationUpdate} that lead to qualitatively different predictions.
As opposed to the unrealistic velocity-based complementarity constraint \eqref{eq:velocitybasedcomplementarity}, the termination condition $\Jn[i]\Velocity' \geq 0$ is removed from \eqref{eq:ImplicitInelasticConstraint}, and thus it may take many increments of our model to reach post-impact velocity.
Furthermore, the removal of this constraint makes \eqref{eq:SimulationUpdate} underconstrained, and thus allows significant freedom for selection of the normal impulse increments.

As our model intentionally captures a set of realistic outcomes, we frame resolving an impact as sampling from that set.
We parameterize the sampling process with a normal impulse distribution with finite-valued probability density $p(\NormalForce)$ over the unit box; step size $h>0$; and (possibly infinite) max iteration count $N$.
We compute samples from our discrete approximation of Routh's method (Algorithm \ref{alg:ImpactSimulation}) as follows:
\begin{enumerate}
	\item Generate a non-zero, maximum normal impulse increment $\NormalForce[max] \sim h\cdot p(\NormalForce)$.
	\item Find a set of forces $\Force \neq \ZeroVector$ with normal component $\NormalForce \leq \NormalForce[max]$ that solves \eqref{eq:SimulationUpdate}; Increment $\Velocity \gets \Velocity + \Mass^{-1}\J^T\Force$.\label{step:DiscreteIncrementSelection}
	\item Terminate and take $\Velocity^+ = \Velocity$ if it is non-colliding ($\Velocity \not \in \ActiveSet(\Configuration)$) or the iteration limit is reached; else, return to 1.
\end{enumerate}
While our theoretical results extend to any $p(\NormalForce)$, we assume in this section that $p(\NormalForce)$ is uniform density for simplicity.
For notational compactness, we assume that all contacts are active and non-penetrating ($\Gap(\Configuration) = 0$).

A difficulty in step \ref{step:DiscreteIncrementSelection}) above is that $\Force = \ZeroVector$ solves \eqref{eq:SimulationUpdate}, and makes no progress towards impact termination.
Additionally, it is possibile that no solution to step \ref{step:DiscreteIncrementSelection}) allows $\NormalForce = \NormalForce[max]$.
We therefore add constraints that encourage $\NormalForce$ to be large:
\begin{align}
	&\Complementary{\SlackForce}{\NormalForce[max] - \NormalForce}\,,\label{eq:ImpactSimulationSlackForceComplementairty} \\
	&\Complementary{\NormalForce}{\Jn \Velocity' + \SlackForce}\,. \label{eq:ImpactSimulationNormalForceComplementairty}
\end{align}
Together, \eqref{eq:ImpactSimulationSlackForceComplementairty}--\eqref{eq:ImpactSimulationNormalForceComplementairty} enforce \eqref{eq:ImplicitInelasticConstraint}; $\NormalForce \leq \NormalForce[max_i]$; and either $\NormalForce[i] = \NormalForce[max_i]$ or contact $i$ has terminated.

Similar to the methods of \citet{Glocker1995} and \citet{Anitescu97} described in \ref{subsec:impactbackground} (see also \eqref{eq:SimulationFrictionBasisComplementarity}--\eqref{eq:SimulationSlackVelocityComplementarity}), we transcribe our model as $\LCP{\LCPMatrix_{\Configuration_0}}{\LCPVector_{\Configuration_0}(\Velocity,\NormalForce[max])}$:
\begin{subequations}
 \begin{align}
	\LCPVariables &= \begin{bmatrix}
		\SlackForce \\ \bar\Force \\ \SlackVelocity
	\end{bmatrix}\,, \quad \LCPVector_{\Configuration}(\Velocity,\NormalForce[max])= \begin{bmatrix}
		\NormalForce[max] \\
		\bar\J \Velocity \\
		\ZeroVector \\
	\end{bmatrix}\,,\\
	\LCPMatrix_{\Configuration} & = \begin{bmatrix}
		\ZeroVector & -\Identity & \ZeroVector & \ZeroVector \\
		\Identity & \Jn \Mass^{-1} \Jn & \Jn \Mass^{-1} \JD & \ZeroVector \\
		\ZeroVector & \JD \Mass^{-1} \Jn &  \JD \Mass^{-1} \JD & \OneVectorMatrix \\
		\ZeroVector & \FrictionCoeff & -\OneVectorMatrix^T & \ZeroVector
	\end{bmatrix}\,,\\
	\Velocity'(\bar \Force) &= \Velocity + \Mass^{-1}\bar\J^T\bar\Force\,,
\end{align}\label{eq:SimulationLCP}
\end{subequations}
where $\Configuration$ is the configuration of the impacting state; $\FrictionCoeff = \mathrm{diag}(\FrictionCoeff[1],\dots,\FrictionCoeff[\Contacts])$; and $\OneVectorMatrix = \mathrm{blkdiag}(\OneVector,\dots,\OneVector)$.
We note in particular that the columns and rows of $\LCPMatrix_{\Configuration}$ and $\LCPVector_{\Configuration}$ associated with $[\bar \Force; \SlackVelocity]$ above are identical to the impact LCP of \citet{Anitescu97}.
\begin{algorithm}
\small
\SetAlgoLined
\KwIn{step $h$, initial state $\State_0 = [\Configuration_0;\Velocity_0]$, max iterations $N$}
\KwOut{final velocity $\Velocity$}
 $(\Velocity, i) \gets (\Velocity_0, 0)$\;
 \While{$\Velocity \in \ActiveSet(\Configuration_0)$ and $i \leq N$}{
  $\NormalForce[max] \sim h\cdot p(\NormalForce)$\label{line:LambdaMaxSelection}\;
  Select $\LCPVariables = [\SlackForce; \bar{\Force}; \SlackVelocity] \in \LCP{\LCPMatrix_{\Configuration_0}}{\LCPVector_{\Configuration_0}(\Velocity,\NormalForce[max])}$\label{line:ZSelection}\;
  $(\Velocity, i) \gets (\Velocity + \Mass(\Configuration_0)^{-1}\bar\J(\Configuration_0)^T\bar{\Force}, i + 1)$\label{line:VelocityUpdate}\;
 }
 \caption{$\mathrm{Sim}(h,\State_0,N)$\label{alg:ImpactSimulation}}
\end{algorithm} 
\subsection{Properties}
\subsubsection{Existence} The most essential property of our integration step is that, because $\LCPMatrix_{\Configuration}$ is copositive, we can leverage \Cref{prop:LCPCopExist} to show that the constituent LCP has a solution:
\begin{theorem}[Single-Step Existence (Appendix \ref{adx:TimesteppingSolutionExistenceProof})]\label{thm:TimesteppingSolutionExistence}
	$\LCP{\LCPMatrix_{\Configuration}}{\LCPVector_{\Configuration}(\Velocity,\NormalForce[max])}$ is non-empty for all states $[\Configuration;\; \Velocity]$, and normal impulse $\NormalForce[max] \geq \ZeroVector$. A solution can be found with Lemke's Algorithm in finite time.
\end{theorem}
\subsubsection{Dissipation} As discussed in Section \ref{subsubsec:Dissipation}, an essential property of inelastic impacts is energy dissipation; because solutions to our model approximate the DI, each integration step cannot increase kinetic energy:
\begin{theorem}[Dissipation (Appendix \ref{adx:TimesteppingDissipationProof})]\label{thm:TimesteppingDissipation}
	Let $[\Configuration;\; \Velocity]$ be any state with active contact, and let $\NormalForce[max] \geq \ZeroVector$ be a normal impulse. Then all impulses $\bar \Force$ generated by the impact constraints ($\LCP{\LCPMatrix_{\Configuration}}{\LCPVector_{\Configuration}(\Velocity,\NormalForce[max])}$) dissipate kinetic energy:
	\begin{equation}
		\KineticEnergy(\Configuration,\Velocity'(\bar \Force)) \leq \KineticEnergy(\Configuration,\Velocity)\,.
	\end{equation}
\end{theorem}
\subsubsection{Linear Impact Termination}
We now show that Algorithm \ref{alg:ImpactSimulation} likely terminates in a small number of steps, allowing it to be used to implement $\mathrm{ImpactLaw}(\Configuration,\Velocity)$.

To understand the rate at which this termination happens, we consider that a pointed friction cone (\Cref{assump:nondegenerate}) guarantees that the magnitude of the change in velocity $\Mass^{-1}\bar \J^T \bar\Force$ for a single step not only grows linearly in $\Norm{\NormalForce}_1$, but also moves $\Velocity$ in some (non-unit) direction $\vect r(\Configuration)$:
\begin{lemma}[Net Force Bound (Appendix \ref{adx:NetForceBoundedByIndividualForcesProof})]
	Consider a configuration $\Configuration \in \ConfigurationSet_A \setminus \ConfigurationSet_P$.
	There exists a nonzero $\vect r(\Configuration) \in \Real^{\Velocities}$, such that for any $\bar \Force = [\NormalForce;\; \FrictionBasisForce]$ obeying \eqref{eq:SimulationSlackVelocityComplementarity},
	\begin{equation}
		\vect r(\Configuration) \cdot \Mass^{-1}\bar\J(\Configuration)^T\bar\Force \geq \Norm{\NormalForce}_1\label{eq:repsiloninequality} \,.
	\end{equation}
		\label{lem:NetForceBoundedByIndividualForces}
\end{lemma}
\noindent $\vect r(\Configuration)$ is computable as a linear program, as it arises from minimization over a polygonal set, $\LinearFrictionCone$.

Let the random variable $Z(h,\Configuration_0,\Velocity_0)$ be the number of LCP solves required for $\mathrm{Sim}(h,[\Configuration_0;\;\Velocity_0],\infty)
$ to terminate. Given that multiple impacts might occur in a single time-step, it is crucial that $Z(h,\Configuration_0,\Velocity_0)$ be as small as possible. Consider that \Cref{lem:NetForceBoundedByIndividualForces} implies that the velocity takes large steps in the $\vect r$ direction with high probability, yet total movement in any direction is bounded by $2\Norm{\Velocity_0}_{\Mass}$ as kinetic energy is non-increasing (\Cref{lem:NonzeroLambdamaxNonzeroLambda}). We can therefore show that with high probability, $Z$ grows linearly with $\Norm{\Velocity_0}_{\Mass}$:
\begin{theorem}[Discrete Termination (Appendix \ref{adx:ExponentialSimulationDecayProof})]
	Let $\Configuration_0 \in \ConfigurationSet_A \setminus \ConfigurationSet_P$ have $\Contacts$ active contacts. Pick $\sigma$ such that $\Mass(\Configuration_0) \succeq \sigma \Identity$; pick $\vect r(\Configuration_0)$ as in \Cref{lem:NetForceBoundedByIndividualForces}; let $h>0$ be a step-size. Let
	\begin{equation}\textstyle
		c = 4 \left\lceil \frac{(\Contacts +1)\TwoNorm{\vect r(\Configuration_0)}}{h \sqrt{\sigma}} \right\rceil\,.
	\end{equation}
	Then for all $k \in \Integer^+$, $\Velocity_0 \in \ActiveSet(\Configuration_0)$,
	\begin{equation}
		P\Parentheses{Z(h,\Configuration_0,\Velocity_0) > c \left\lceil \Norm{\Velocity_0}_{\Mass} \right\rceil + k} \leq e^{-\frac{k}{(m+1)^2}}\,.
	\end{equation}\label{thm:ExponentialSimulationDecay}
\end{theorem}
As the probability density of $Z$ exponentially decays, it has finite moments (including mean and variance).

We conclude by noting that \Cref{lem:NetForceBoundedByIndividualForces}, and thus the pointed friction cone assumption (\Cref{assump:nondegenerate}), is an essential component of our theoretical analysis for impact termination.
Without this assumption there is no \textit{guarantee} that impact simulations will terminate, but there is no inherent reason that simulations that \textit{happen} to terminate are any less reasonable.
\subsection{Post-Impact Set Approximation}\label{subsec:SetApproximation}
We now describe a method to approximate the set of outcomes to simultaneous impacts as modeled in our DI \eqref{eq:multicontact}, which culminates in probabilistic guarantees on densely sampling this set via \Cref{prop:epsilonnet}.

In order for computation of the set of possible outcomes of Algorithm \ref{alg:ImpactSimulation} to be well-posed, we must consider a key practical ramification of the LCP solve on line \ref{line:ZSelection}: numerical LCP solvers typically only find a \textit{single} solution, and may be systematically biased in their selection among multiple solutions.
For all claims in this section, we therefore make the additional assumption that this selection process does not affect the outcome of an individual integration step:
\begin{assumption}
Consider a configuration $\Configuration \in \ConfigurationSet_A \setminus \ConfigurationSet_P$.
For each velocity $\Velocity$ and normal impulse increment $\NormalForce[max] \geq \ZeroVector$, every $\bar \Force$ generated by \eqref{eq:SimulationLCP} results in the same incremented velocity $\Velocity'$.
Equivalently, there exists a function $\vect f_{\Configuration}: \VelocitySpace \times \Closure(\Real^{\Contacts+}) \to \VelocitySpace$, such that 
\begin{equation}
\Velocity'(\bar\Force) = 
	\Velocity + \Mass^{-1}\bar\J^T\bar\Force = \vect f_{\Configuration}(\Velocity,\NormalForce[max])\,.
	\end{equation}\label{assump:uniqueLCPoutcome}
\end{assumption}
This assumption can be verified via Semidefinite Programming \citep{Aydinoglu2020}.
We note that $\Velocity'(\bar\Force)$ under Assumption \ref{assump:uniqueLCPoutcome} is only unique \textit{given} $\NormalForce[max]$; different velocity increments can be 
While Assumption \ref{assump:uniqueLCPoutcome} is violated for at least some systems (e.g. for compass gait and RAMone in Section \ref{sec:examples}), it implies useful properties including Lipschitz continuity:
\begin{lemma}
	For each configuration $\Configuration \in \ConfigurationSet_A \setminus \ConfigurationSet_P$, $\vect f_{\Configuration}(\Velocity,\NormalForce[max])$ is Lipschitz continuous.\label{lem:SimulationLCPLipschitz}
\end{lemma}
\begin{proof}
	Because $\Velocity'(\bar \Force)$ is unique, we must have that
\begin{equation}
	\bar\J^T\bar\Force = \begin{bmatrix}
		\ZeroVector & \bar\J^T & \ZeroVector
	\end{bmatrix}\LCP{\LCPMatrix_{\Configuration}}{\LCPVector_{\Configuration}(\Velocity,\NormalForce[max])}\,,
\end{equation}
is a singleton over the convex domain $\LCPVector_{\Configuration}(\VelocitySpace,\Closure(\Real^{\Contacts+}))$. Therefore by \Cref{prop:AffineLCPLipschitz}, $\vect f_{\Configuration}$ is Lipschitz continuous.
\end{proof}
We will also make use of two scenarios where the integration step LCP is guaranteed select zero impulse:
\begin{lemma}
	Consider a configuration $\Configuration \in \ConfigurationSet_A \setminus \ConfigurationSet_P$ and $\NormalForce[max] \geq \ZeroVector$. Then if either $\Jn\Velocity \geq \ZeroVector$ or $\NormalForce[max] = \ZeroVector$,
	\begin{equation}
		\Velocity = \vect f_{\Configuration}(\Velocity,\NormalForce[max])\,.
	\end{equation}
	\label{lem:VelocityStationaryConditions}
\end{lemma}
\begin{proof}
	Observe that if either $\NormalForce[max] = \ZeroVector$ or if $\Velocity$ is not impacting ($\Jn\Velocity \geq \ZeroVector$), we can select zero normal impulse ($\NormalForce = \ZeroVector$, thus $\Velocity' = \Velocity$) and satisfy the normal complementary equations \eqref{eq:ImpactSimulationSlackForceComplementairty}--\eqref{eq:ImpactSimulationNormalForceComplementairty}. Setting $\FrictionBasisForce = \ZeroVector$; $\SlackVelocity = \ZeroVector$; and $\SlackForce$ as the negative part of $\Jn\Velocity$ constitutes a full solution to the LCP.
\end{proof}
The continuity of $\vect f_{\Configuration}$ allows for expansion of the $\Jn\Velocity \geq \ZeroVector$ case; if $\Velocity$ is \textit{almost} terminated, then only a single simulation step with a small $\NormalForce[max]$ is required to end the impact:
\begin{lemma}[Single-Step Termination (Appendix \ref{adx:ApproximateTerminationProof})]
	For all configurations $\Configuration \in \ConfigurationSet_A \setminus \ConfigurationSet_P$, velocities $\Velocity$, and $\varepsilon > 0 $, there exists $\delta(\varepsilon,\Velocity)$, such that for any almost-separating velocity $\bar \Velocity$ ($\Jn\bar \Velocity \geq -\delta(\varepsilon,\Velocity)$) that is sufficiently small ($\Norm{\bar \Velocity}_{\Mass} \leq \Norm{\Velocity}_{\Mass}$), a small impulse can terminate the impact: $\vect f_{\Configuration}(\bar \Velocity,\varepsilon\OneVector ) \not \in \ActiveSet (\Configuration)$.
	\label{lem:ApproximateTermination}
\end{lemma}
\begin{figure}
	\centering
    \includegraphics[width=0.7\linewidth,trim={0mm 0mm 15mm 12mm},clip]{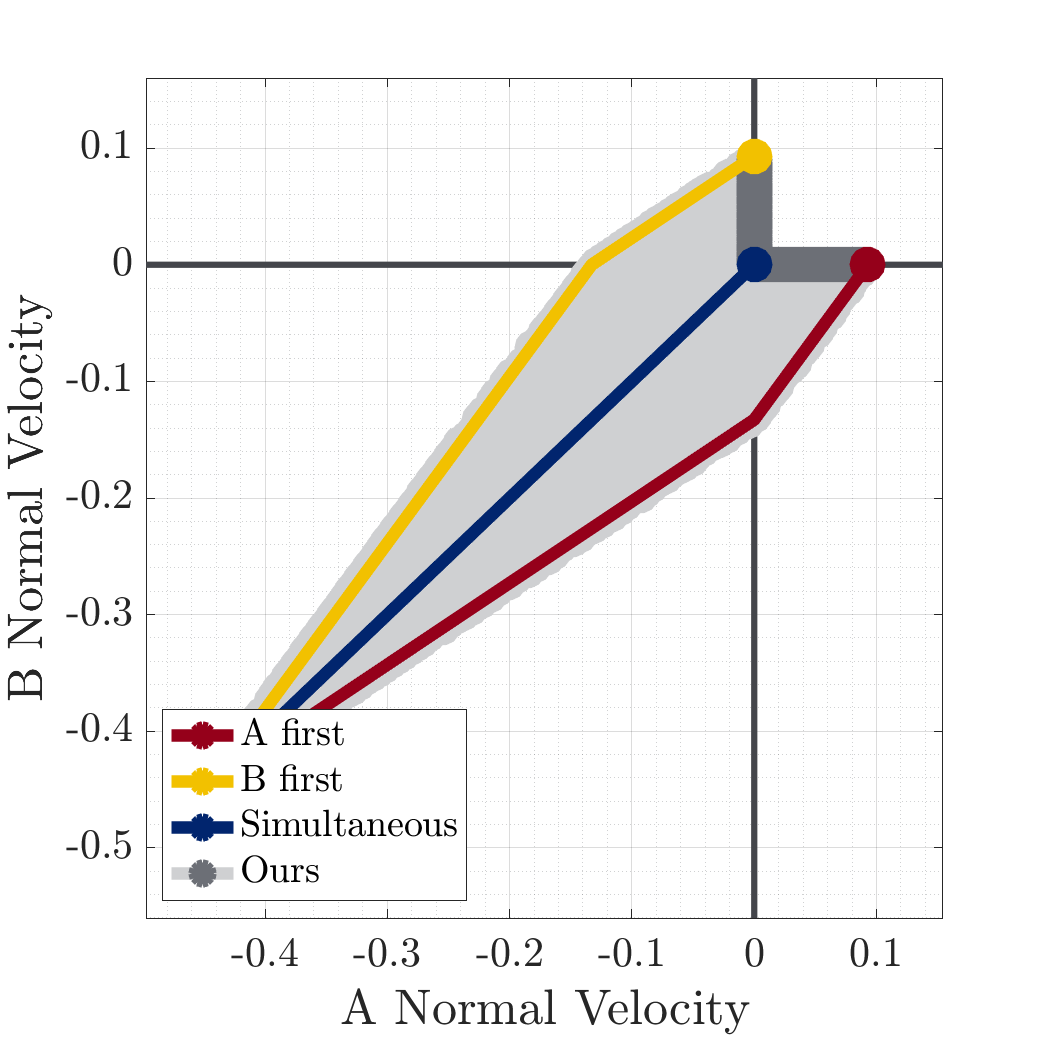}
    \includegraphics[width=0.7\linewidth,trim={0mm 0mm 15mm 12mm},clip]{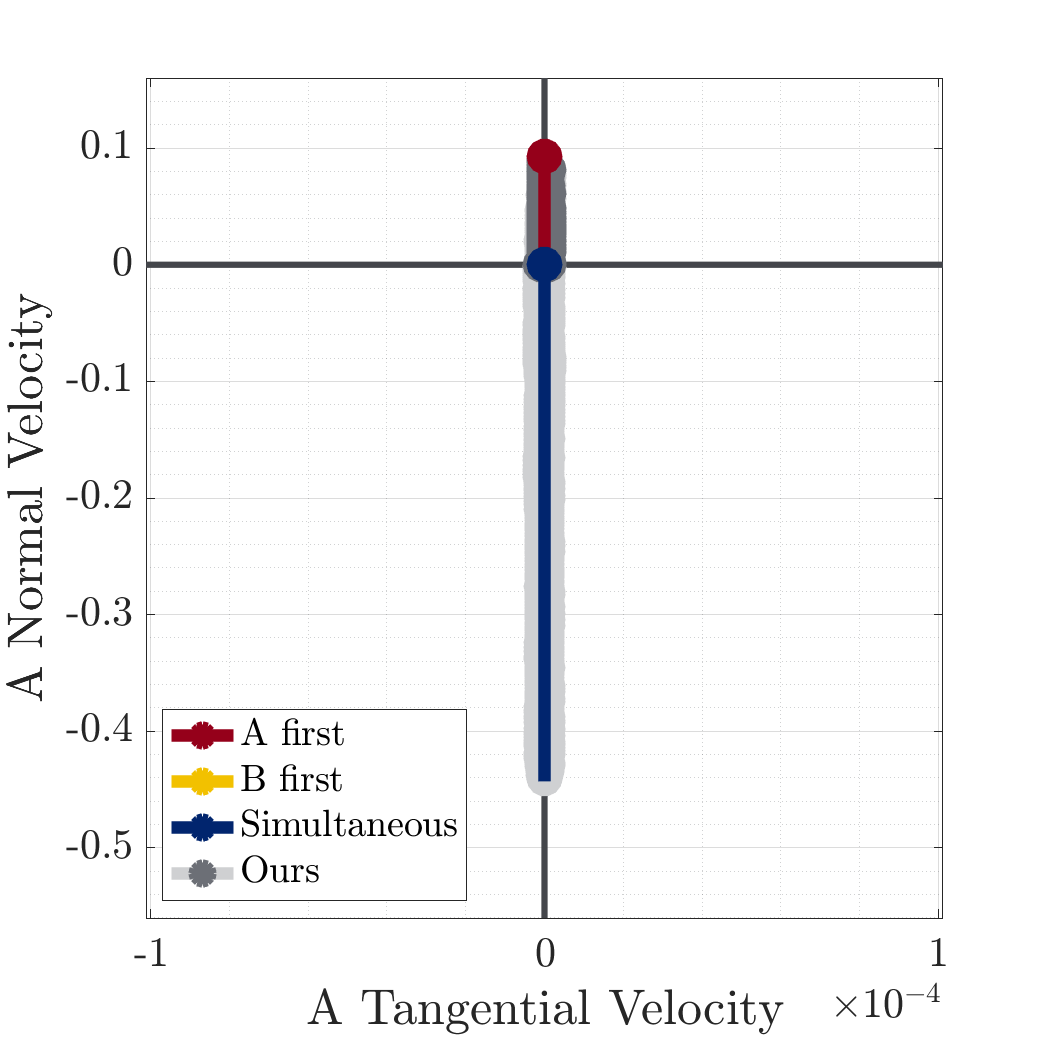}
    \caption{Evolution of the rocking block impact, displayed as the normal velocities of the two points (top), and the normal and tangential velocity of point A (bottom). Our method generates all three outcomes from Figure \ref{fig:phone}, as well as intermediate results between the symmetric and sequential impacts.}
    \label{fig:phone_evolution}
\end{figure}
\begin{figure*}[ht]

    \centering
    \includegraphics[width=0.33\hsize,trim={0mm 0mm 15mm 12mm},clip]{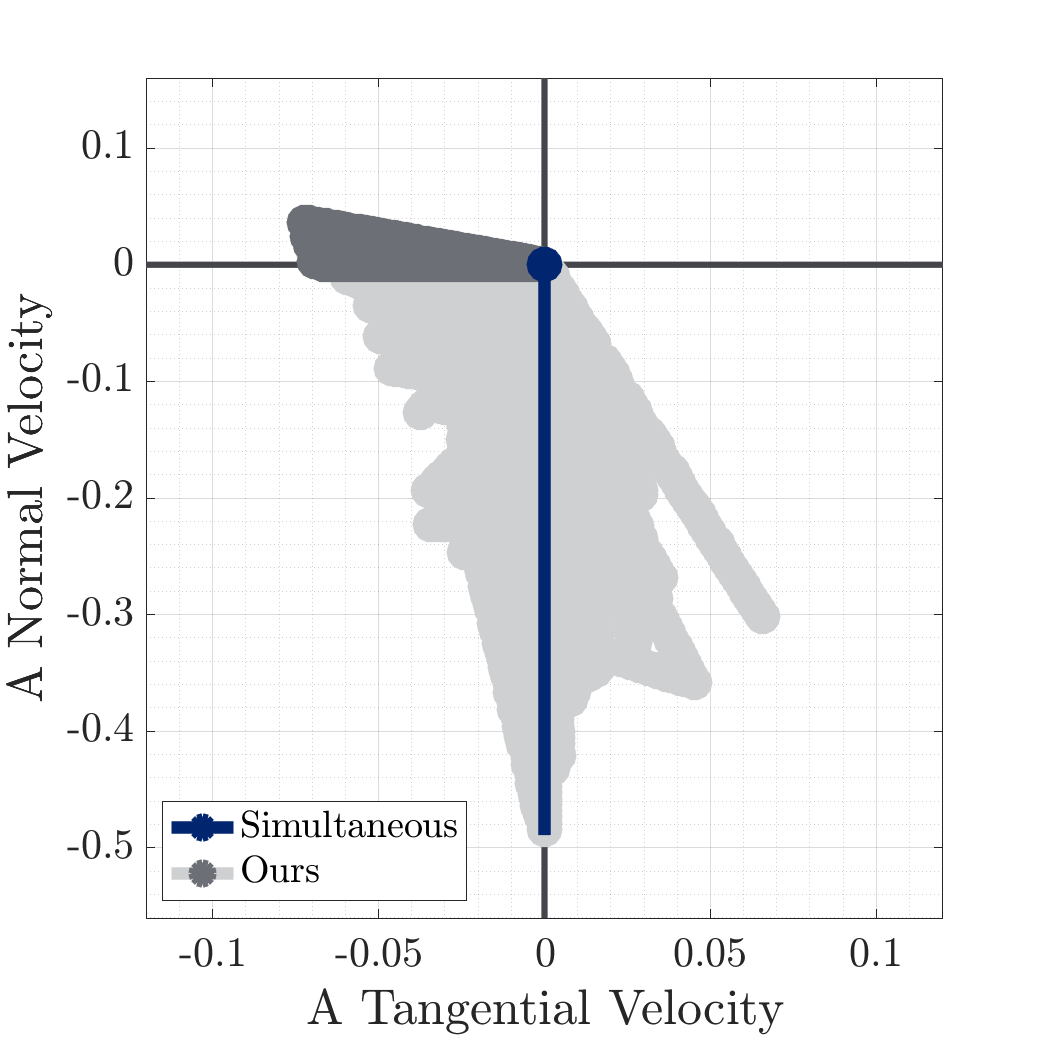}
    \includegraphics[width=0.33\hsize,trim={0mm 0mm 15mm 12mm},clip]{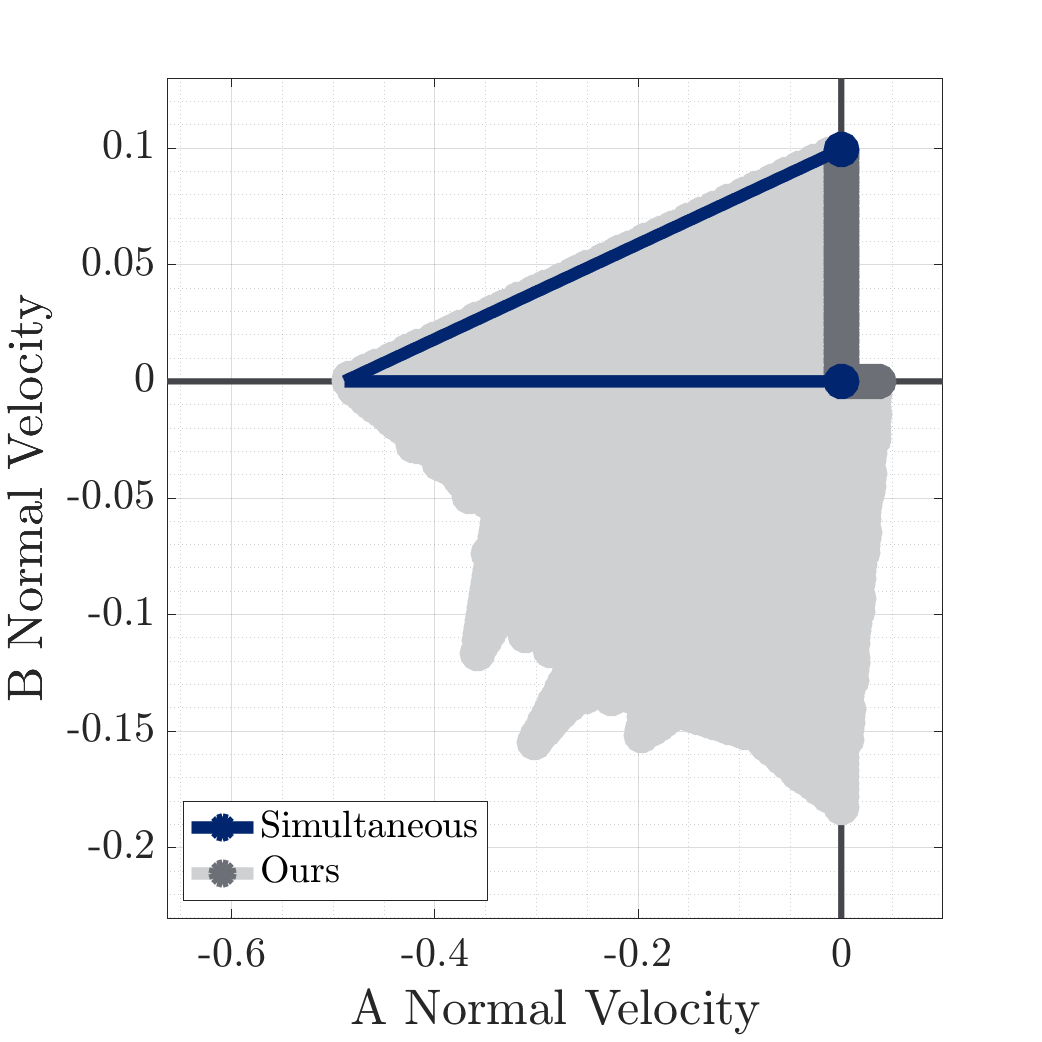}
    \includegraphics[width=0.33\hsize,trim={0mm 0mm 15mm 12mm},clip]{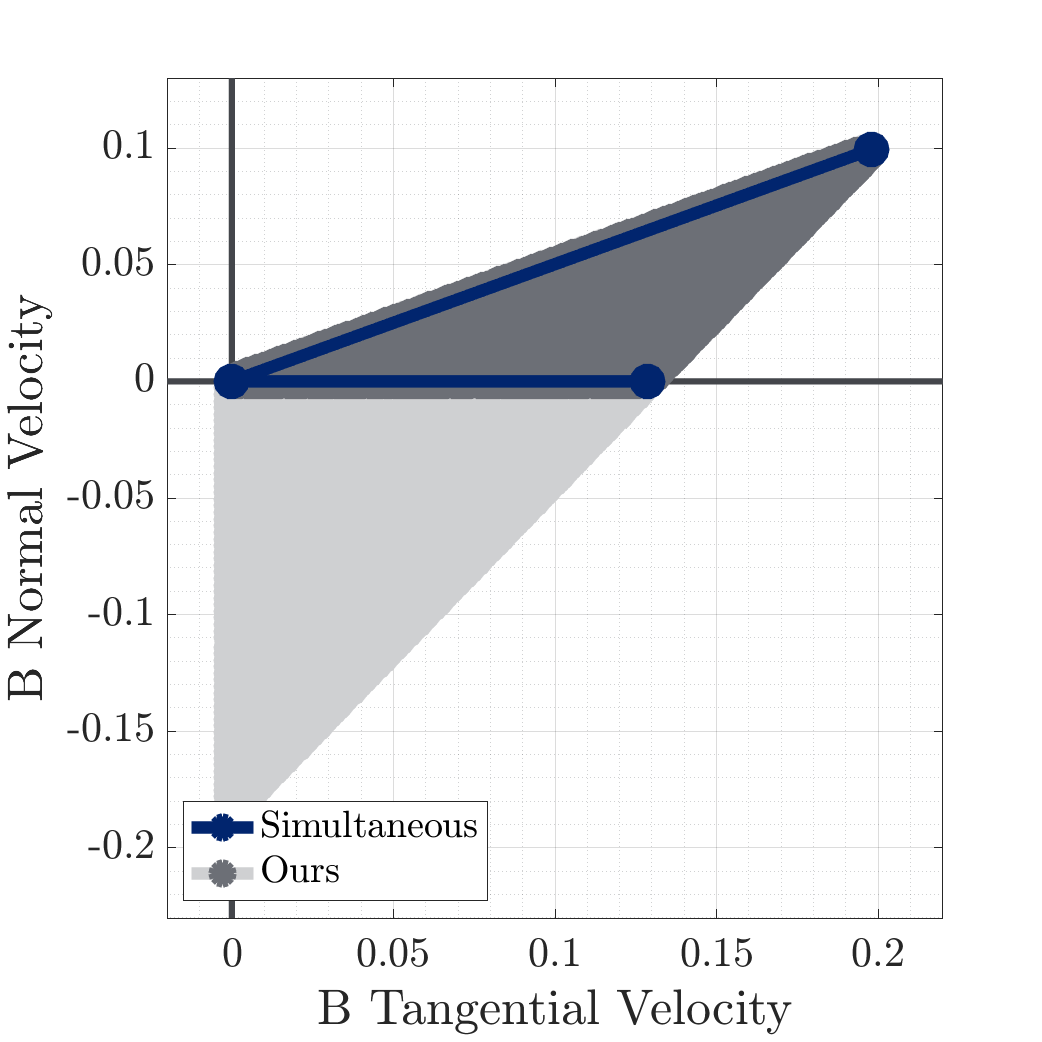}
    \caption{Evolution of the compass gait step. The center plot compares the normal velocities of the two contacts, while the left and right show velocities of points A and B, respectively. Our model produces the three outcomes in Figure \ref{fig:compass_examples}, as well as all reasonable intermediate velocities of point B. Furthermore, oscillation of impact between the feet allows point A to slide or lift off, while point B maintains contact.}
    \label{fig:compass_evolution}\vspace{2mm}
        \centering
    \includegraphics[width=0.33\linewidth,trim={2mm 2mm 15mm 12mm},clip]{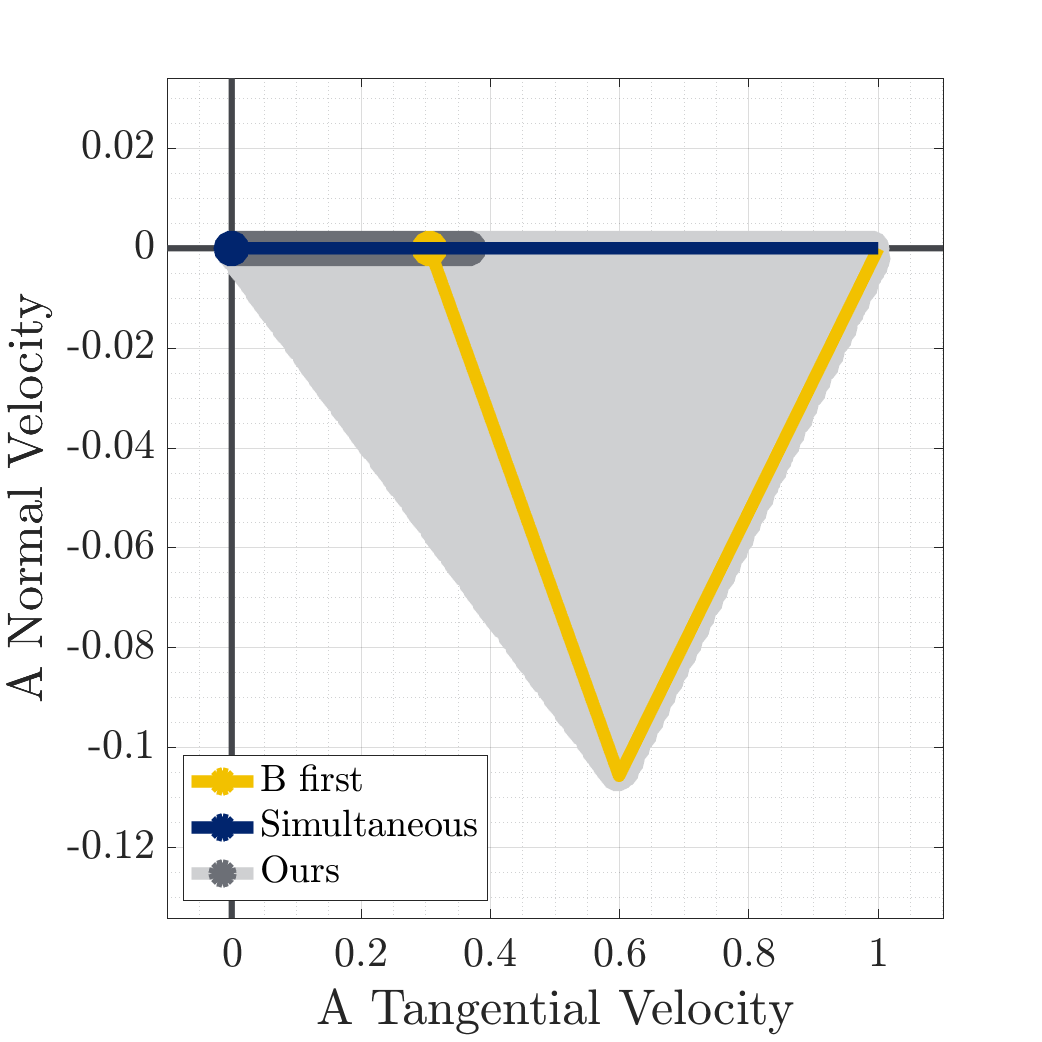}
    \includegraphics[width=0.33\linewidth,trim={2mm 2mm 15mm 12mm},clip]{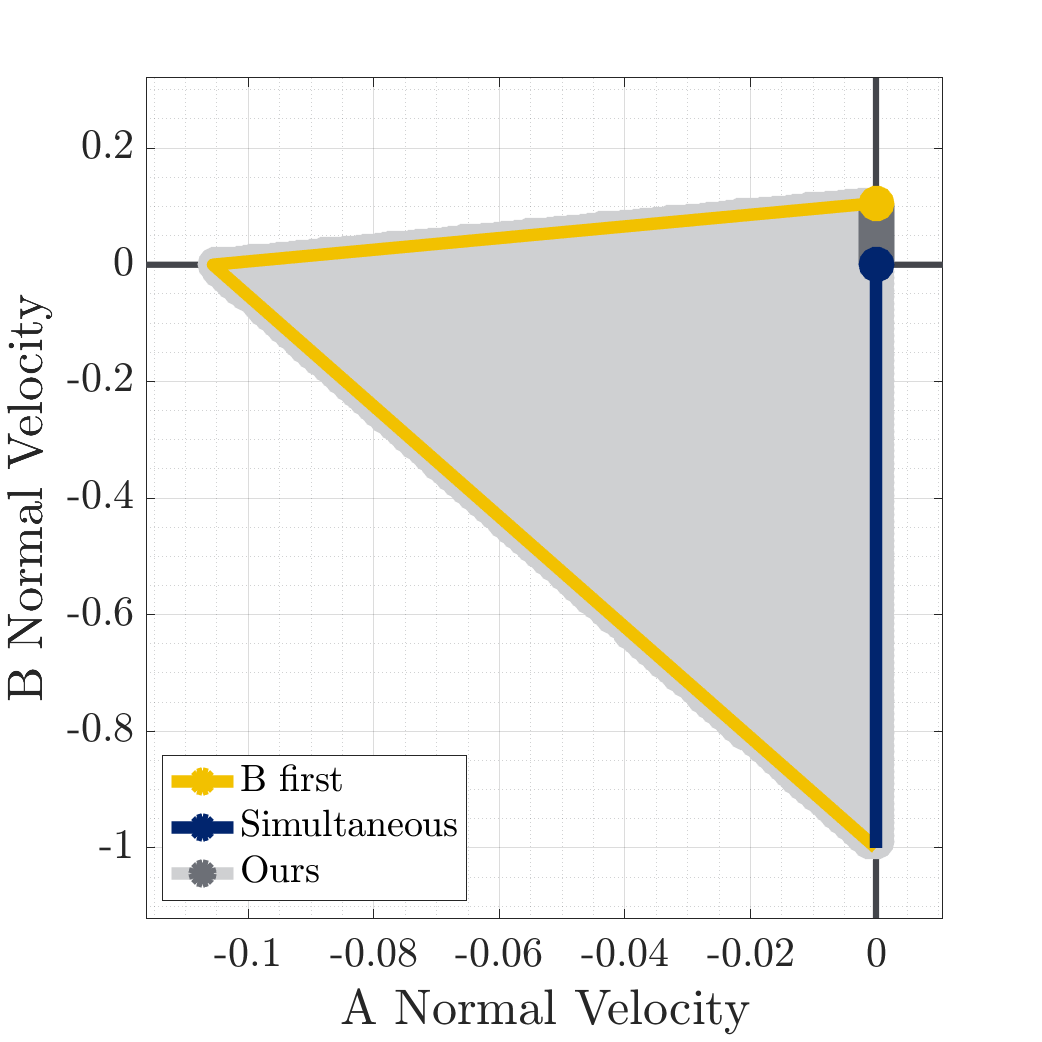}
    \includegraphics[width=0.33\linewidth,trim={2mm 2mm 15mm 12mm},clip]{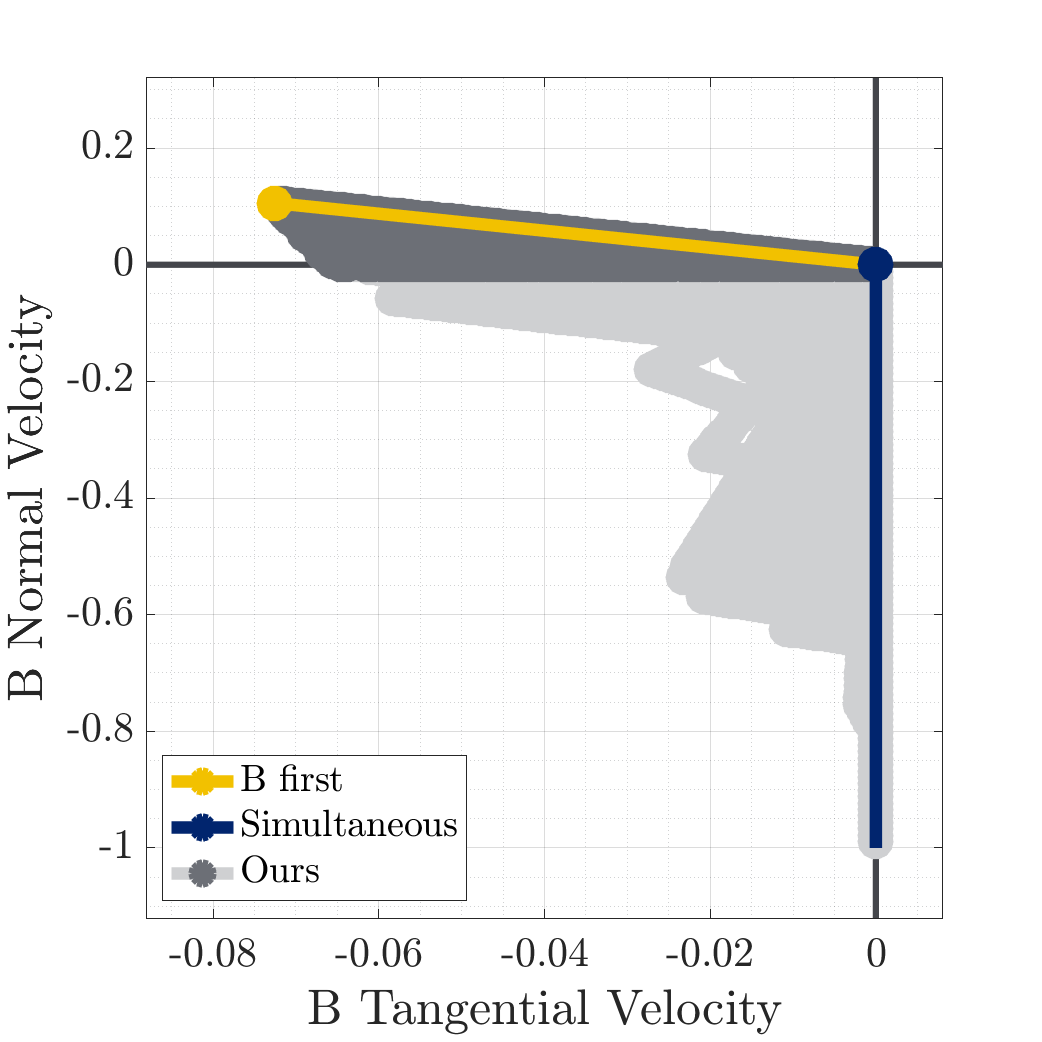}
    \caption{Evolution of the box-wall impact. The center plot compares the normal velocities of the two contacts, while the left and right show velocities of points A and B, respectively. Our model produces both the simultaneous and sequential outcome in Figure \ref{fig:box}, as well as all reasonable intermediate velocities where A still slides and B lifts off. Furthermore, results are also generated where B slides instead of sticking or lifting off.}
    \label{fig:box_evolution}

\end{figure*}

We now iteratively define the reachable set of Alg. \ref{alg:ImpactSimulation}. Let $\VelocitySet_N(\State_0,h)$ be the set of possible outputs of $\mathrm{Sim}(h,\State_0,N)$. Then we have that
\begin{align}
	\VelocitySet_0(\State_0,h) &= \Braces{\Velocity_0}\,,\\
	\VelocitySet_i(\State_0,h) &= \vect f_{\Configuration}(\VelocitySet_{i-1}(\State_0,h),[0,h]^\Contacts)\,,\label{eq:ReachabilityStep} \\
	\VelocitySet_i(\State_0,h) &\supseteq \VelocitySet_{i-1}(\State_0,h)\,.\label{eq:VPlusMonotone}
\end{align}
Here, we used
\Cref{lem:VelocityStationaryConditions} to ignore early termination (i.e. $\Jn\Velocity \geq \ZeroVector$ before $N$ loop iterations) in \eqref{eq:ReachabilityStep}, and
to establish the monotonic growth in \eqref{eq:VPlusMonotone}.
We construct the entire set of reachable velocities as
\begin{equation}
	\VelocitySet_\infty(\State_0,h) = \cup_{i\in\Natural} \VelocitySet_i(\State_0,h)\,.
\end{equation}
$\VelocitySet_i(\State_0,h)$ can approximate $\VelocitySet_\infty(\State_0,h)$ with arbitrarily well:
\begin{lemma}
	Consider a configuration $\Configuration_0 \in \ConfigurationSet_A \setminus \ConfigurationSet_P$; velocity $\Velocity_0$, and step-size $h \geq 0$. Then for each $\varepsilon > 0$, there exists an $i$, such that $\VelocitySet_i([\Configuration_0;\;\Velocity_0],h)$ is an $\varepsilon$-net of $\VelocitySet_\infty([\Configuration_0;\;\Velocity_0],h)$.\label{lem:VPlusNConvergence}
\end{lemma}
\begin{proof}
	 $\VelocitySet_i([\Configuration_0;\;\Velocity_0],h)$ is a monotonic \eqref{eq:VPlusMonotone} and uniformly bounded (via \Cref{thm:TimesteppingDissipation}) sequence of sets. It is then convergent in the $\varepsilon$-net sense to $\cup_i \VelocitySet_i =\VelocitySet_\infty([\Configuration_0;\;\Velocity_0],h)$.
\end{proof}
Similarly, the post-impact reachable set is simply the reachable velocities which are non-penetrating:
\begin{equation}
	\mathrm{Sim}(h,\State_0,\infty) \in \VelocitySet_\infty(\State_0,h) \setminus \ActiveSet(\Configuration_0)\,.
\end{equation}
\begin{algorithm}[h]
\small
\SetAlgoLined
\KwIn{step size $h$, initial state $\State_0 = [\Configuration_0;\;\Velocity_0]$, approximation $\varepsilon \in (0,h)$, trajectory length $N$, trajectory count $M$}
\KwOut{post-impact set approximation $\tilde\VelocitySet^+$}
  $\tilde \VelocitySet^+ \gets \Braces{}$\;
  $\psi \gets \sigma_{max}\Parentheses{\Mass^{-1}\bar\J^T}\Contacts(1 + \max_i\FrictionCoeff[i]) + 1$\label{line:psi} \;
 \For{$i=1$ \KwTo $M$}{
  $\Velocity \gets \mathrm{Sim}(h,\State_0,N)$\;\label{line:ApproximationSimulation}
  $\tilde \VelocitySet^+ \gets \tilde\VelocitySet^+ \cup \Braces{\vect f_{\Configuration_0}(\Velocity,\frac{\varepsilon}{3\psi}\OneVector_\Contacts)}$\label{line:FinalStep}\;
 }
 $\tilde \VelocitySet^+ \gets \tilde \VelocitySet^+ \setminus \ActiveSet(\Configuration_0)$\;
 \caption{$\mathrm{Approximate}(h,\State_0,\varepsilon,N,M)$\label{alg:Approximate}}
\end{algorithm}

We can finally use the above derived properties to construct a method, Algorithm \ref{alg:Approximate}, for approximating the post-impact set. \Cref{lem:VPlusNConvergence} and \Cref{prop:epsilonnet} together show that $M$ samples from $\mathrm{Sim}(h,\State_0,N)$ well-approximate $\VelocitySet_\infty$, and can be forced to terminate with only a small additional step (\Cref{lem:ApproximateTermination}). Therefore, Algorithm \ref{alg:Approximate} is approximately complete:
\begin{theorem}
	Consider an initial configuration $\Configuration_0 \in \ConfigurationSet_A \setminus \ConfigurationSet_P$, initial velocity $\Velocity_0 \in \VelocitySpace$, and step-size $h > 0$.
	For all $\varepsilon,\delta > 0$, there exists $N,M > 0$, such that $\mathrm{Approximate}(h,\State_0,\varepsilon,N,M)$ returns an $\varepsilon$-net of $\VelocitySet_\infty(\State_0,h) \setminus \ActiveSet(\Configuration_0)$ with probability at least $1-\delta$.
	\label{thm:VelocitySetApproximation}
\end{theorem}
\begin{proof}
	See Appendix \ref{adx:VelocitySetApproximationProof}
\end{proof}

\subsection{Numerical Examples}
\label{sec:examples}

We now show several examples of the post-impact velocity sets generated by our model.
The MATLAB code is available online\endnotemark[2].
We analyze three examples shown thus far (Figs. \ref{fig:phone}, \ref{fig:compass_setup}, \ref{fig:box}), along with two more complex systems.

For each system, we plot the evolution of the velocity through the impact process with lines, projected onto the contact frames; these plots compare the impact process in our method to both sequential and simultaneous resolution via \Cref{eq:AnitestcuDiscreteFormulation} \citep{Anitescu97}, as described in \Cref{subsec:impactbackground}.
Our method is shown in gray and simultaneous resolution via \citet{Anitescu97} is shown in blue.
For two-contact impacts, we show the two sequential resolutions $(A,B,A,\dots)$ and $(B,A,B,\dots)$ in red and yellow.
We show samples of the post-impact velocity sets generated via Algorithm \ref{alg:Approximate}, as a dark gray region.
The light gray region by contrast traces the intermediate velocities encountered along the impact-resolving trajectories of our model.
For some examples, axes of symmetry were used to duplicate samples.

All examples were computed on a MacBook Pro with 2.4 GHz Quad-Core Intel Core i5.
In \Cref{table:CompStats}, we report mean runtime of our algorithm for each of these examples in terms of the number of LCP steps to resolve each impact; wall-clock time of each impact sample; and wall-clock computation time for each LCP solve.
In general, we find that the step sizes implemented in our examples are capable of terminating all impacts within a handful of steps.
From a simulation perspective, generating a single sample would therefore only be a few times slower than e.g. the LCP method of \citet{Anitescu97}, with solve times on the order of \SI{2}{\milli\second}.
However, global set approximation takes between $2^{10}$ and $2^{20}$ samples depending on the example (see \Cref{adx:exampledetails}), and thus fast set approximation remains an open challenge.

\begin{table}
\centering
	\caption{Computational Performance for Post-impact Set Sampling}
	\label{table:CompStats}
	\small
	\centering
	\setlength\tabcolsep{1.2mm}
	\begin{tabular}{c c c c}
		\toprule
		Example & LCP's / Sample & Time / Sample  & Time / LCP \\
		\midrule
		Rocking Block & 2.67 & \SI{2.1e-3}{\second} & \SI{8.0e-4}{\second} \\
		Compass Gait & 2.27 & \SI{1.9e-3}{\second}  & \SI{8.5e-4}{\second} \\
		Box and Wall & 1.97 & \SI{1.6e-3}{\second}  & \SI{8.1e-4}{\second} \\
		RAMone & 3.20 & \SI{2.2e-3}{\second}  & \SI{7.0e-4}{\second} \\
		Disk Stacking & 9.04 & \SI{1.1e-2}{\second}  & \SI{1.3e-3}{\second} \\
		\bottomrule
	\end{tabular}
\end{table}

Additionally, for the Rocking Block example, we analyze whether it is valid to interpret the set of predictions of our model as the results of highly-sensitive outcomes of impacts between highly-stiff, deformable bodies.
\subsubsection{Rocking Block}\label{subsec:rockingblockexample}
We revisit dropping a narrow, rectangular object onto flat ground (Fig. \ref{fig:phone}), which may either result in the object coming to rest or pivoting on a corner.
As shown in Figure \ref{fig:phone_evolution}, our method produces each of these symmetric and sequential outcomes.
The real-world analogues of these three outcomes are that the short but non-zero duration impacts either happen at the exact same time and rate or sequentially with no overlap.
Our model also produces analogues to where there is some partial overlap in these durations, for which scaled-down versions of the purely-sequential outcomes (i.e., rolling on one foot with a smaller angular velocity) is the final result.

To examine whether or not these additional results can be physically interpreted as originating from sensitivity to impact ordering, we employ a compliant contact simulation scheme as a point of comparison.
Under these dynamics, the block evolves through time according to the manipulator equations \eqref{eq:ManipulatorEquations}, with contact forces determined by a Kelvin-Voigt linear elasticity model.
The normal forces are applied at each corner $ i \in \Braces{\vect A, \vect B}$ (see \Cref{fig:phone}), computed as
\begin{equation}
	\NormalForce[i] = \begin{cases}
		\max\Parentheses{0,-k_i\Gap_i - 2\zeta \sqrt{\frac{k_i}{m}} \TimeDiff{\Gap_i}} & \Gap_i \leq 0\,,\\
		0 & \Gap_i > 0\,.
	\end{cases}
\end{equation}
$k_i$ is the contact stiffness at point $i$.
$\zeta$ is the damping ratio, which can make impacts inelastic if set high enough; for all experiments, we use $\zeta = 5$ and stiffness at least \SI{1e6}{\newton\per\meter} to approximate instantaneous, inelastic impacts.
As the resulting impacts are inelastic, we can consider a continuous-time collision as being terminated when the penetrating velocity has nearly vanished i.e. $\Jn[i]\Velocity \geq -\delta_i$ for all active contacts $\Gap_i(\Configuration) \leq 0$.
For all experiments, we use $\delta_i = \frac{m}{\min_i k_i}$ \SI{1e-3}{\meter\per\second}.
The friction forces follow a typical continuous approximation of Coulomb friction \eqref{eq:CoulombForce} \citep{Nurkanovic2021,Castro2020}, with
\begin{align}
	\FrictionForce[i] &= -\FrictionCoeff[i]\NormalForce[i]\Unit_\varepsilon(\Jt[i]\Velocity)\,,\\
	\Unit_\varepsilon(\vect r) &= \frac{\vect r}{\max\Parentheses{\TwoNorm{\vect r}, \varepsilon}}\,.
\end{align}
For all experiments, we set $\varepsilon = \SI{1e-10}{\meter\per\second}$.
We simulate trajectories of this stiff system with MATLAB's stiff ODE solver \texttt{ode15s}.

\begin{figure}[t]
	\centering
    \includegraphics[width=0.8\linewidth]{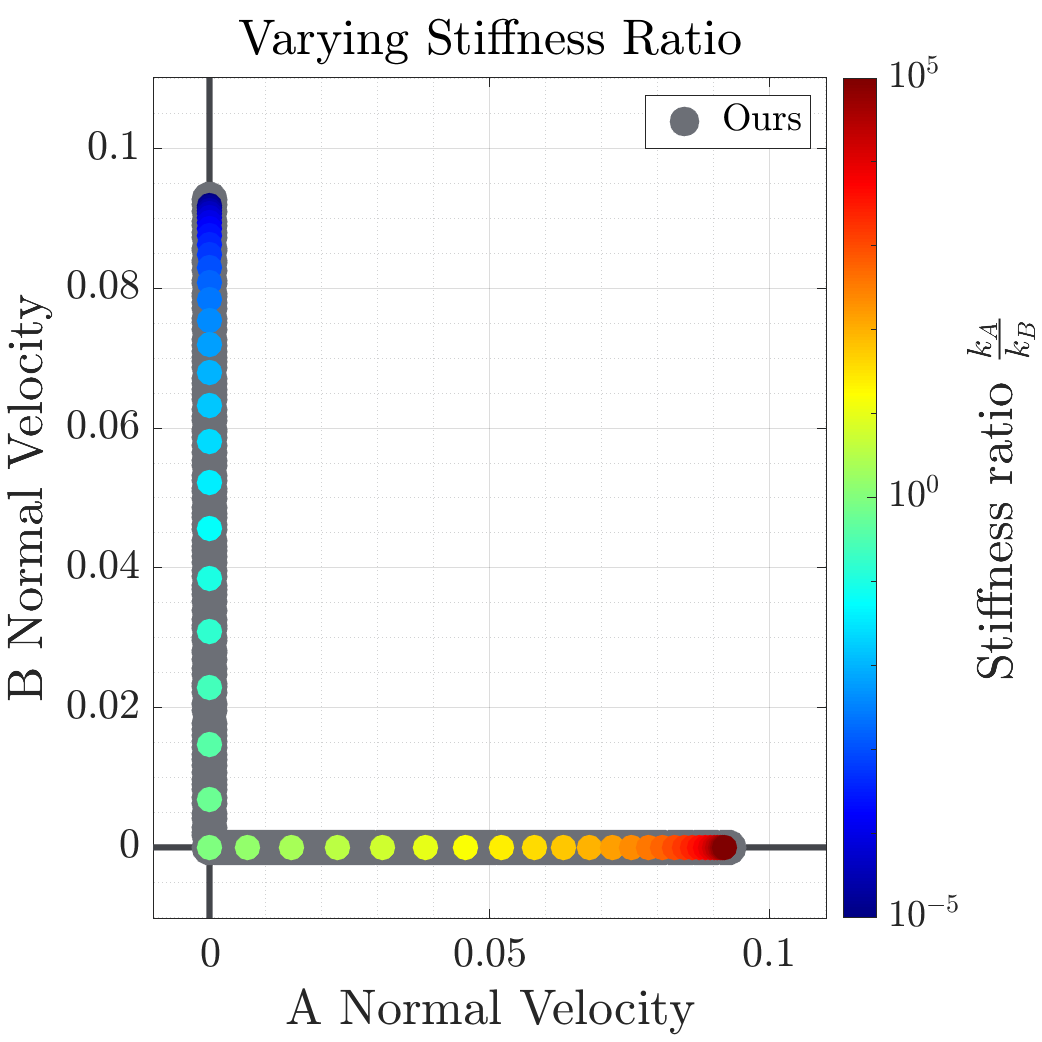}
    \includegraphics[width=0.8\linewidth]{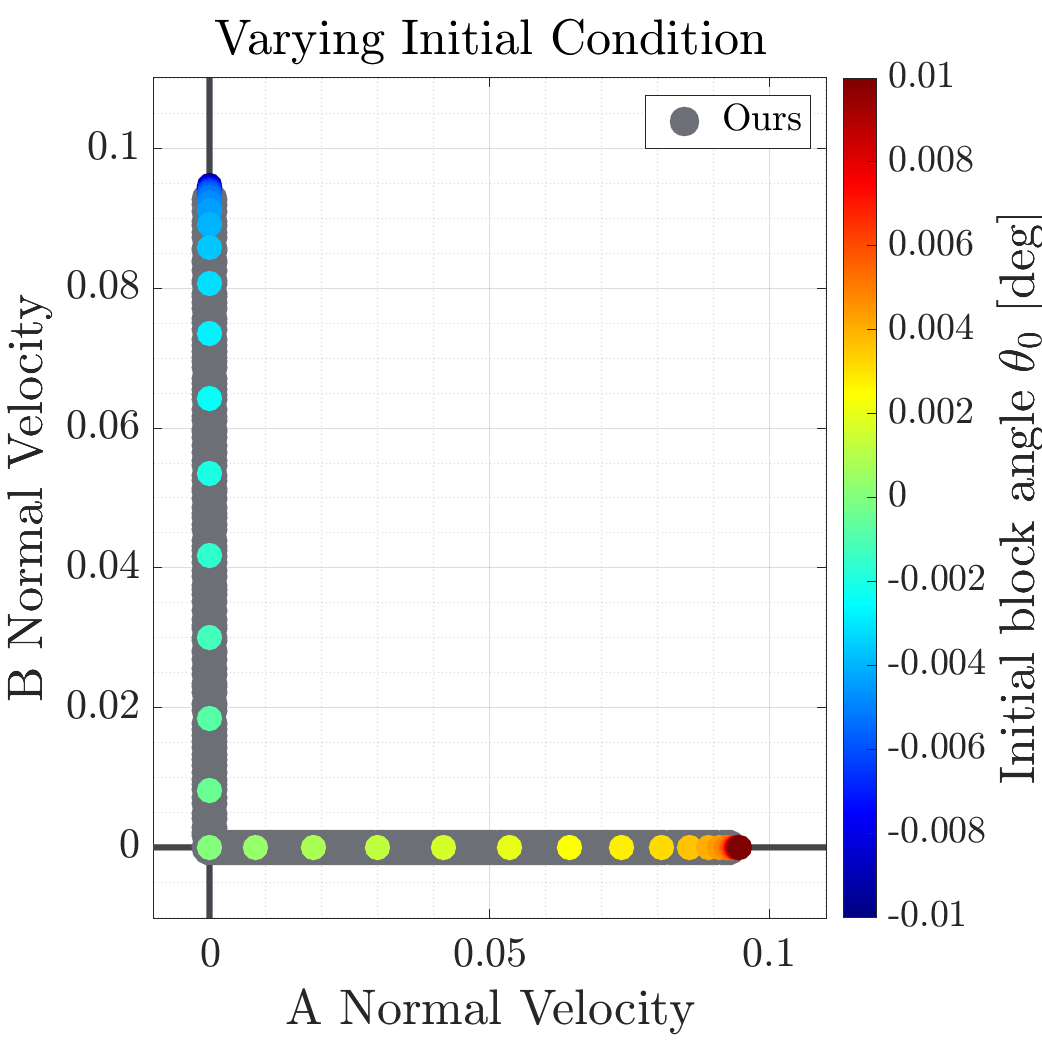}
    \caption{Post-impact velocities of perturbations to the rocking block impact under compliant contact are well-captured by our model's predictions.
    Velocities under perturbation are are plotted as circles of varying color, projected onto the contact normals.
    Top: the stiffness ratio between the contacts is varied from $10^{-5}$ to $10^5$. 
    The set of compliant-contact outcomes is properly contained in and nearly covers the entirety of our model's predictions.
    Bottom: the angle of the block before the impact is varied from \SI{-1e-2}{\deg} to \SI{1e-2}{\deg}.
    As before, nearly the entire set of our model's predictions is covered by these perturbations.
    However, the maximum post-impact separating velocity is slightly larger than our model's predictions.}
    \label{fig:block_sensitivity}
\end{figure}

We consider two phenomena in the compliant contact domain which may lead to differing outcomes: different stiffness ratios $\frac{k_A}{k_B}$ and slight changes in initial conditions.
In the former, we vary $\frac{k_A}{k_B}$ from $10^{-5}$ to $10^5$, while holding $\min(k_A,k_B) =\SI{1e6}{\newton\per\meter}$ and holding the pre-impact state the same as in \Cref{fig:phone_evolution};
$10^5$ was chosen at the maximum ratio due to numerical limitations of the ODE solver.
In the latter, we vary the initial angle $\theta$ at which the block hits the ground, such that the collision process will start at one corner first.
We vary $\theta_0$ from \SI{-1e-2}{\deg} to \SI{1e-2}{\deg}, which is enough that the impacts are nearly sequential in the extreme cases, with 99.9992\% of the total impulse accumulated at the first collision occuring before the second collision starts.
For each angle $\theta_0$, the initial high of the block is raised so one corner touches the ground, and the initial downward velocity is accordingly slightly decreased to maintain the same total mechanical energy (kinetic plus gravitational potential) across initial conditions.
For each of these two experiments, we additionally consider $49$ intermediate values for $\frac{k_A}{k_B}$ and $\theta_0$.

We plot outcomes of the concurrent impacts in \Cref{fig:block_sensitivity}, and find that the scaled-down predictions of our model in \Cref{fig:phone_evolution} can be attributed to differing orders of impulse accrual due to either changing stiffness ratio or initial condition.
In the case of varying stiffness ratio with the same initial condition as \Cref{fig:phone_evolution}, we find that the results, when projected onto the contact normals, are properly contained by and nearly cover the entirety of our model's predictions.
In the case of varying initial angle, we again see a tight match between the compliant collisions and our model's predictions, except that the maximum post-impact separating velocity for compliant collisions is slightly larger.

As our model's predictions were only calculated for $\theta_0 = 0$, we would expect that over the set of perturbations to $\theta_0$ we encounter a slightly larger set of post-impact velocities, due to slight shifts from the continuous-time evolution between states.

\subsubsection{Compass Gait}
We revisit the compass gait walker model taking a wide step (Fig. \ref{fig:compass_setup}).
Previously, we showed that the model of \citet{Anitescu97} predicts that the leading foot sticks (point A), while the trailing foot (point B) could slide, stick, or lift off.
Our model generates each of these outcomes, as well as various convex combinations of these results (Figure \ref{fig:compass_evolution}).
It also generates oscillatory behavior where impulses at points A and B alternate during the impact process.
This can potentially cause A to lift off, and B to remain on the ground instead.

\subsubsection{Box and Wall}
We examine our model's predictions on the scenario described in Figure \ref{fig:box}, where a box impacts a wall (at point B) while sliding along flat ground (at point A).
Simultaneous resolution with \citet{Anitescu97} predicted that the box came to rest, while sequential resolution predicts that A continues sliding and B lifts off.
As in the previous examples, our model reproduces both behaviors, as well as convex combinations of them (Figure \ref{fig:box_evolution}).
Additionally, some sequences allow A to slide even faster, while others allow B to slide instead of lifting or sticking.

\subsubsection{RAMone}
In this example, we examine a footstep of a more complex 5-link bipedal robot, RAMone, originally considered by \citet{Remy2017}.
As shown in Figure \ref{fig:ramone}, much like the compass gait example, \citet{Anitescu97} always predicts that the leading foot sticks, while the trailing foot can stick, slide, or lift.
Our model reproduces the same results, as well as ones where the final contact velocities are scaled down.
\begin{figure*}[ht]
    \center
    \begin{subfigure}[b]{\linewidth}
        \includegraphics[width=.25\hsize,trim={0mm 20mm 10mm 14mm},clip]{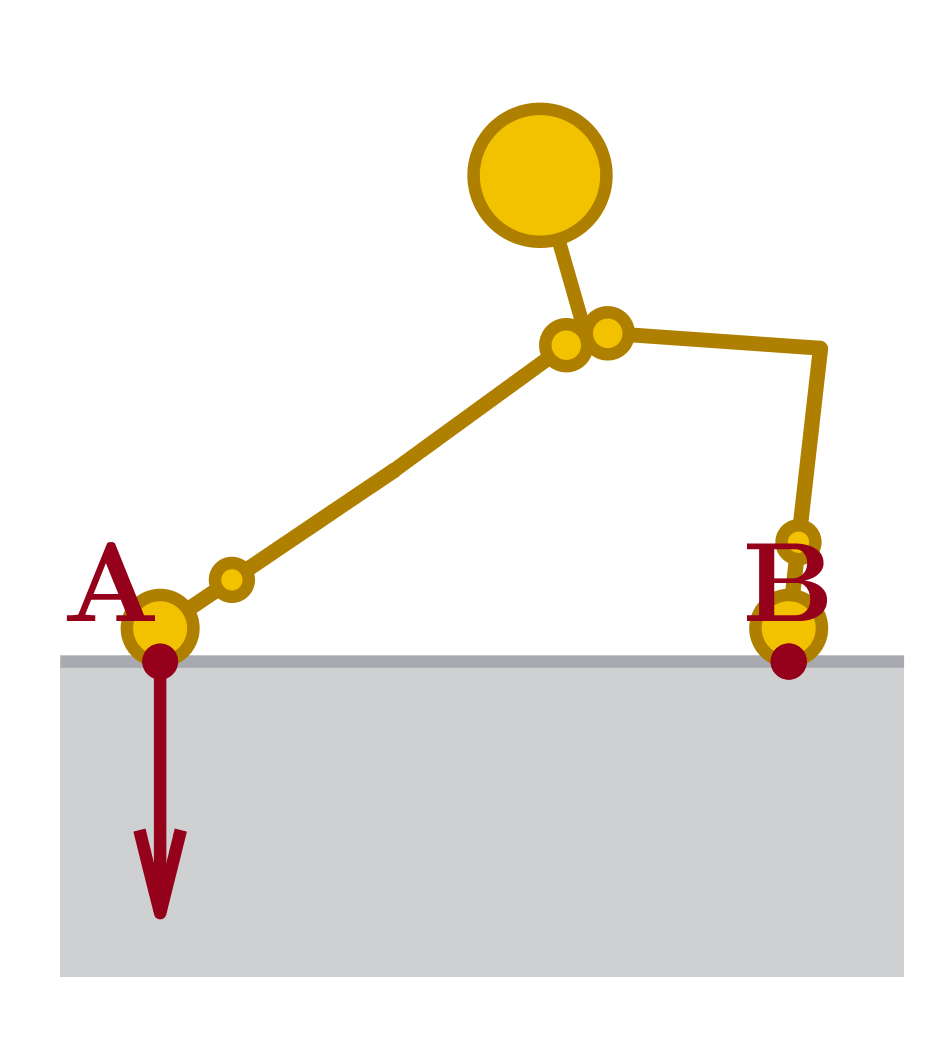}
        \centering
        \caption{\label{fig:ramone_diagram} RAMone initial condition}
    \end{subfigure}
    \begin{subfigure}[b]{\linewidth}
        \includegraphics[width=0.33\hsize,trim={2mm 2mm 15mm 12mm},clip]{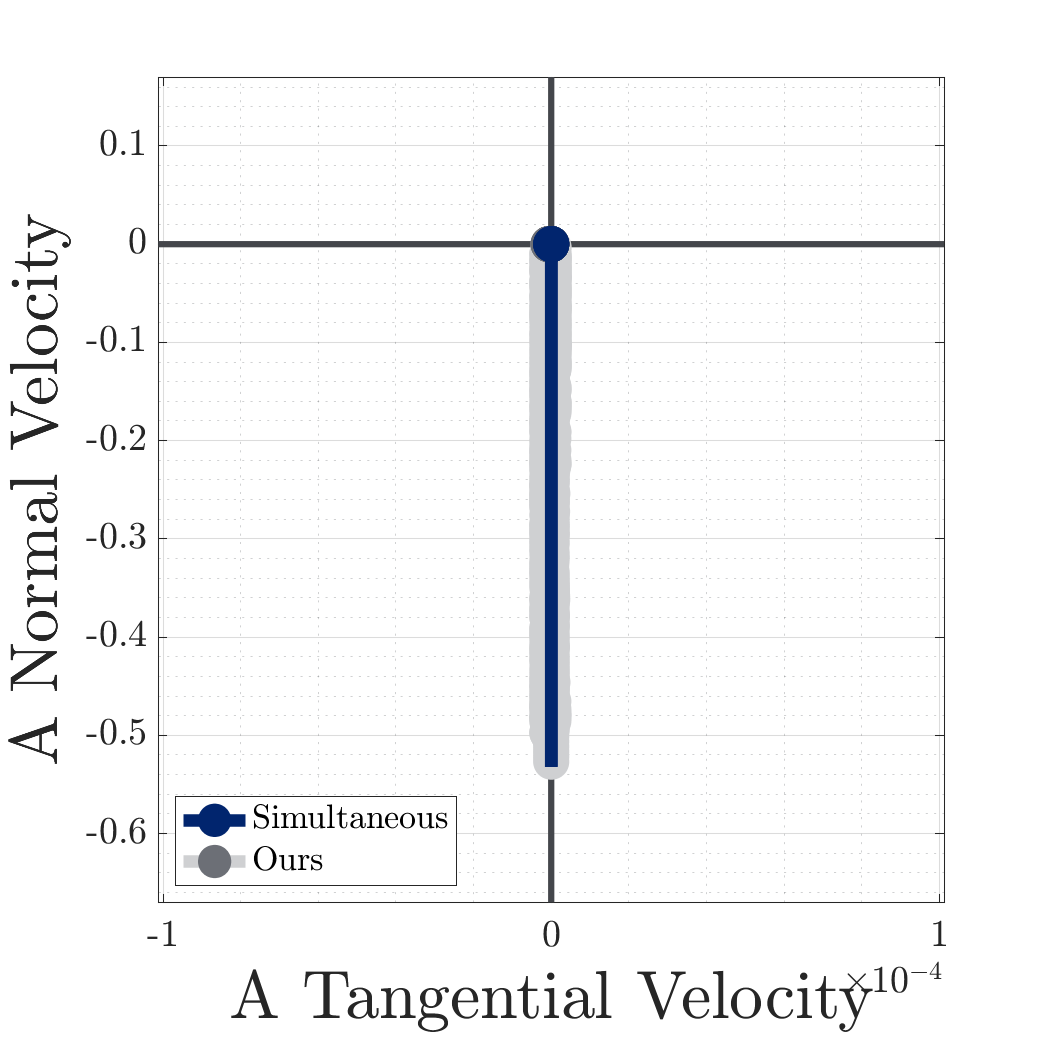}
    \includegraphics[width=0.33\hsize,trim={2mm 2mm 15mm 12mm},clip]{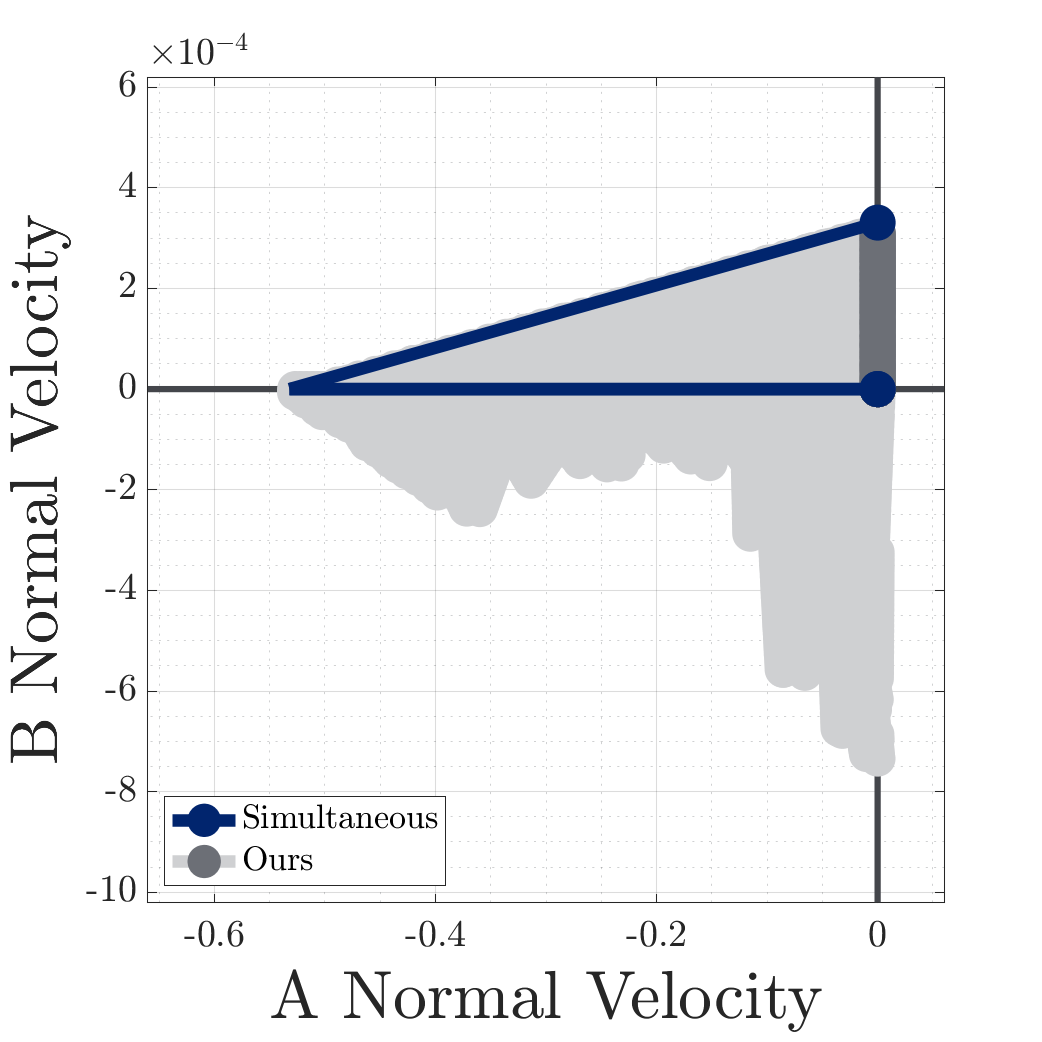}
    \includegraphics[width=0.33\hsize,trim={2mm 2mm 15mm 12mm},clip]{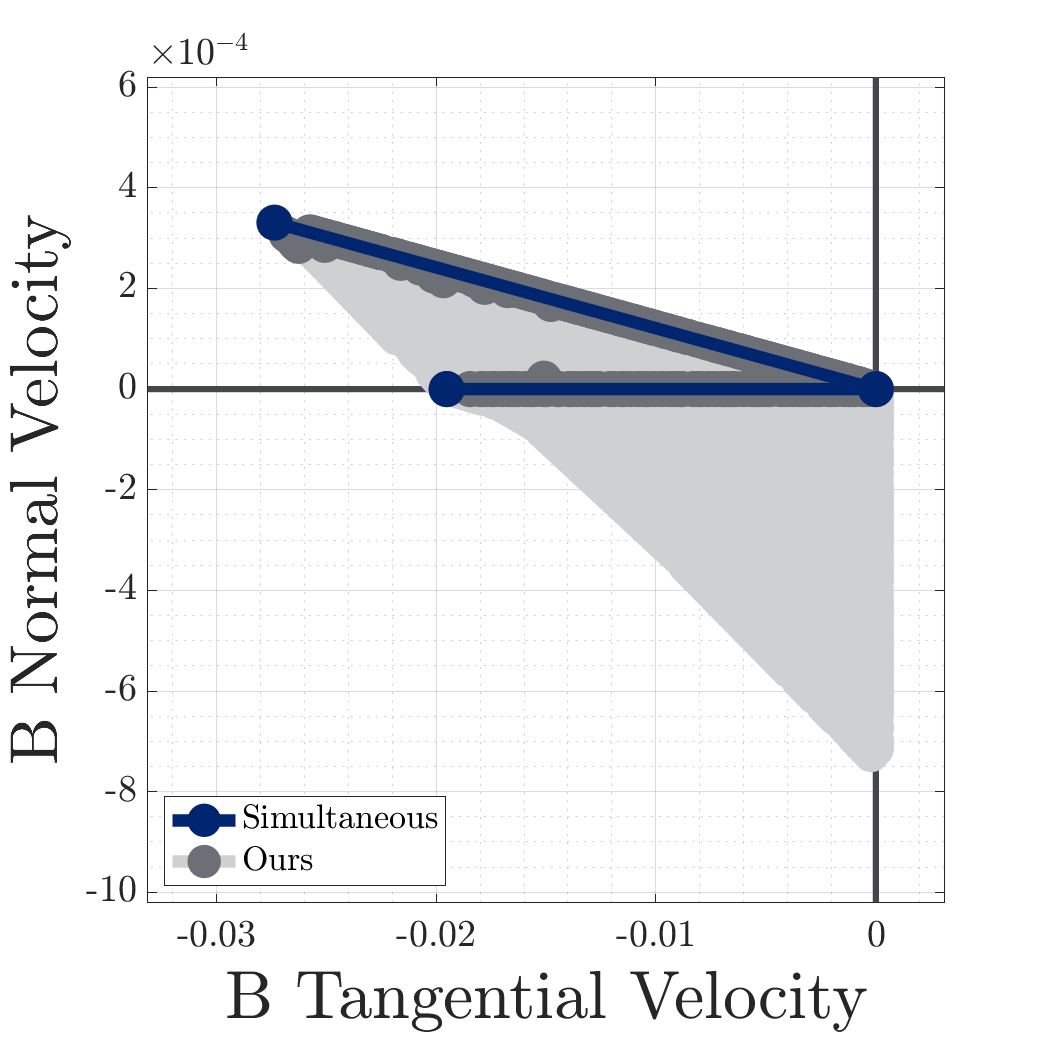}
        \centering
        \caption{\label{fig:ramone_evolution} Footfall impact resolution process}
    \end{subfigure}

    \caption{Evolution of a footfall (\subref*{fig:ramone_diagram}) of the RAMone robot. (\subref*{fig:ramone_evolution}) Similar to the compass gait example, \citet{Anitescu97} predicts that the leading foot, point A, comes to rest, while point B may come to rest, slide, or lift off. All results from our model produce intermediate outcomes between these three results, and point A remains in stiction for the entire duration of the impact.}
    \label{fig:ramone}
\end{figure*}
\begin{figure*}[ht]
    \center
    \begin{subfigure}[b]{.33\linewidth}
        \includegraphics[width=\hsize]{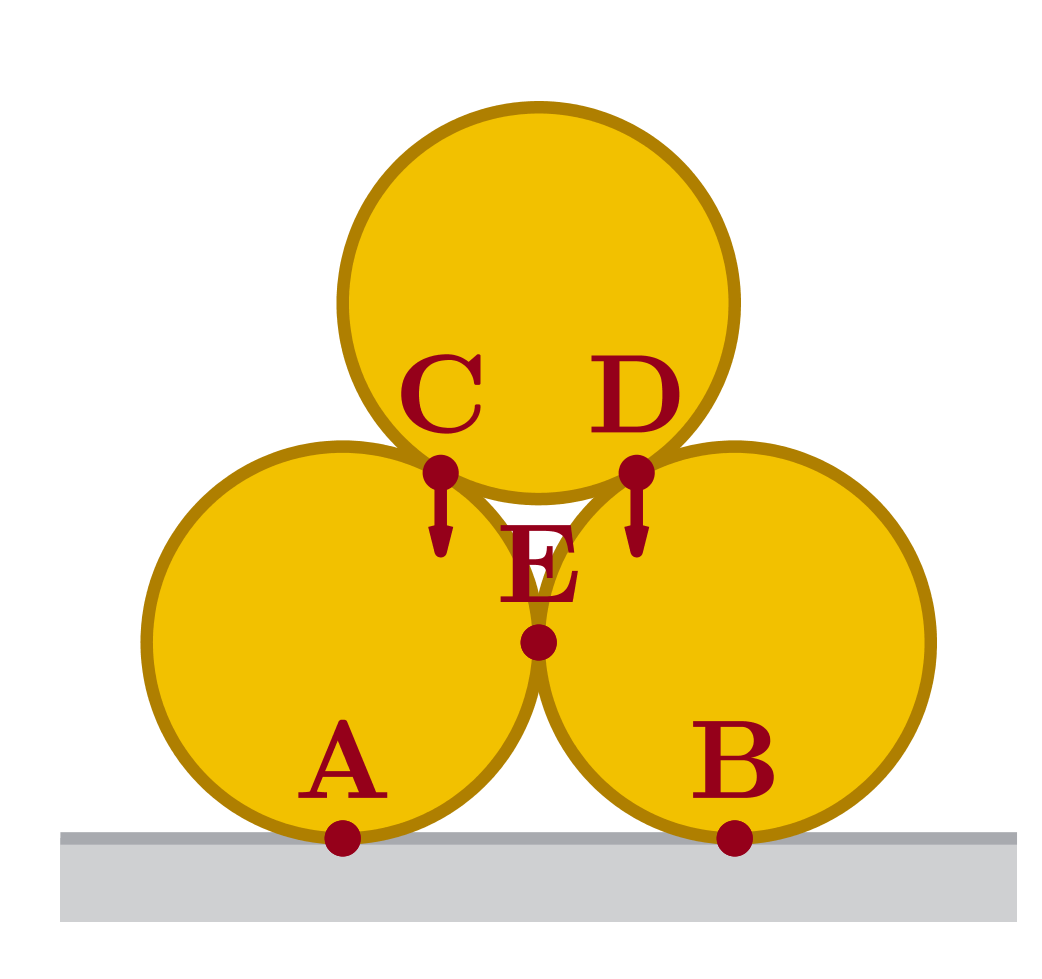}
        \centering
        \caption{\label{fig:stack_diagram} Disk stack initial condition}
    \end{subfigure}
    \begin{subfigure}[b]{.66\linewidth}
        \includegraphics[width=0.49\hsize]{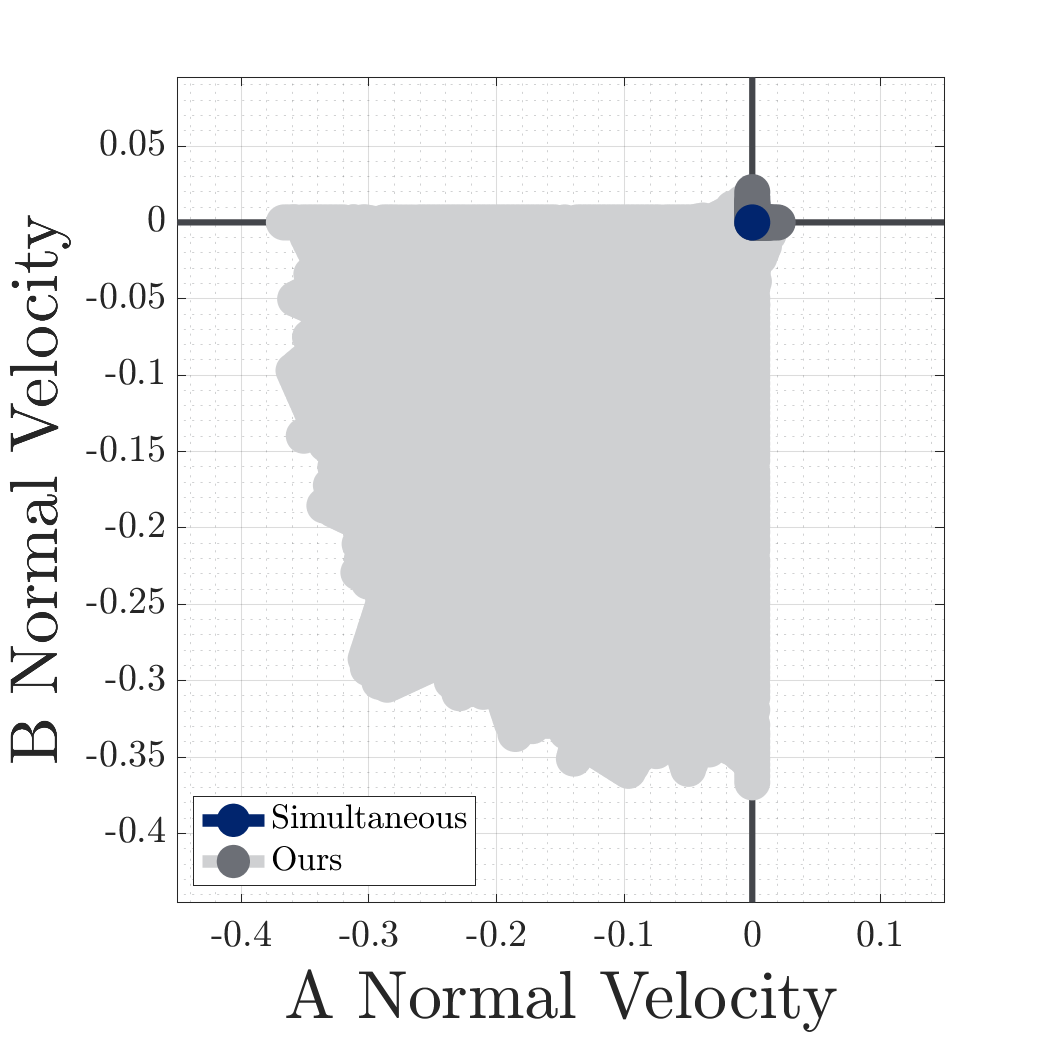}
    \includegraphics[width=0.49\hsize]{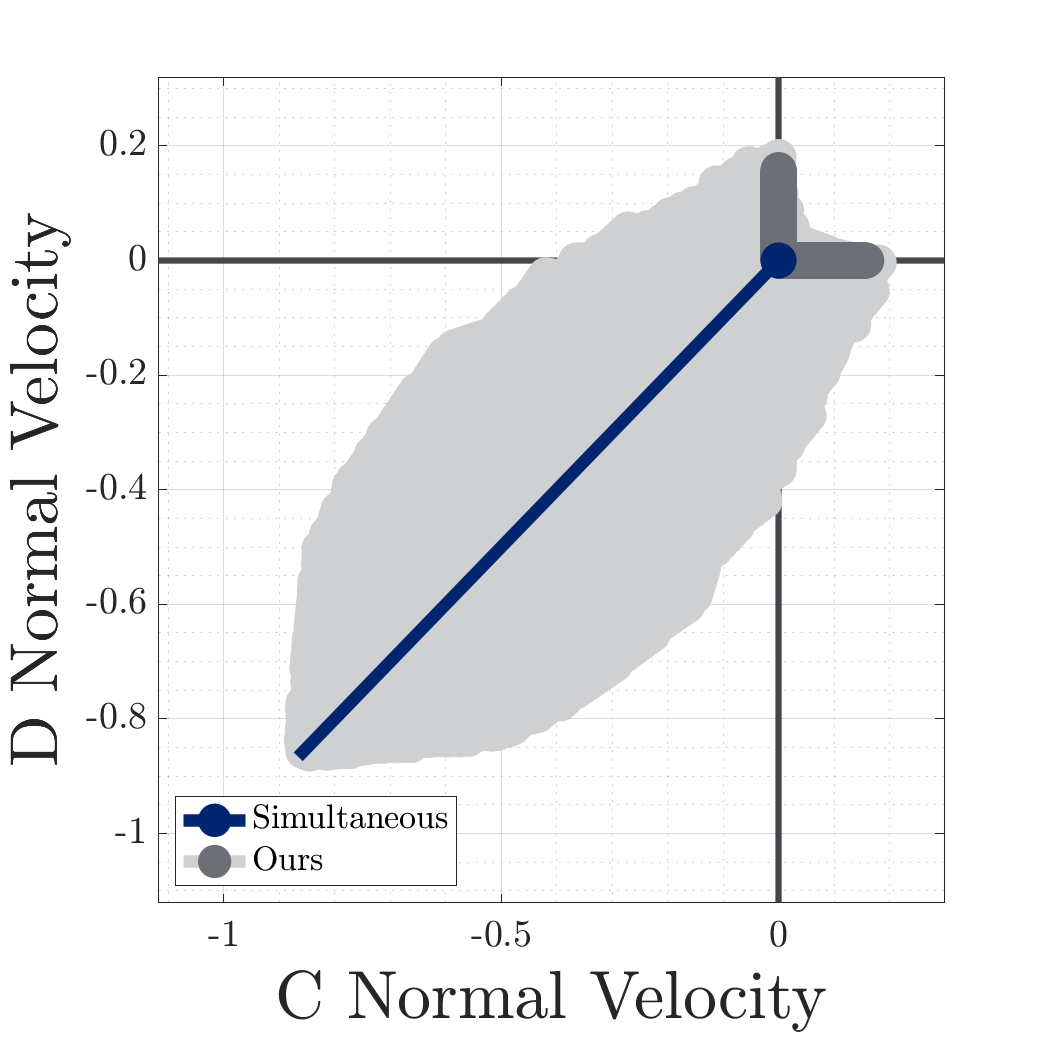}
        \centering
        \caption{\label{fig:stack_side_evolution} Normal impact resolution process comparing left and right sides}
    \end{subfigure}
    
    \begin{subfigure}[b]{\linewidth}
    	\centering
        \includegraphics[width=0.32\hsize]{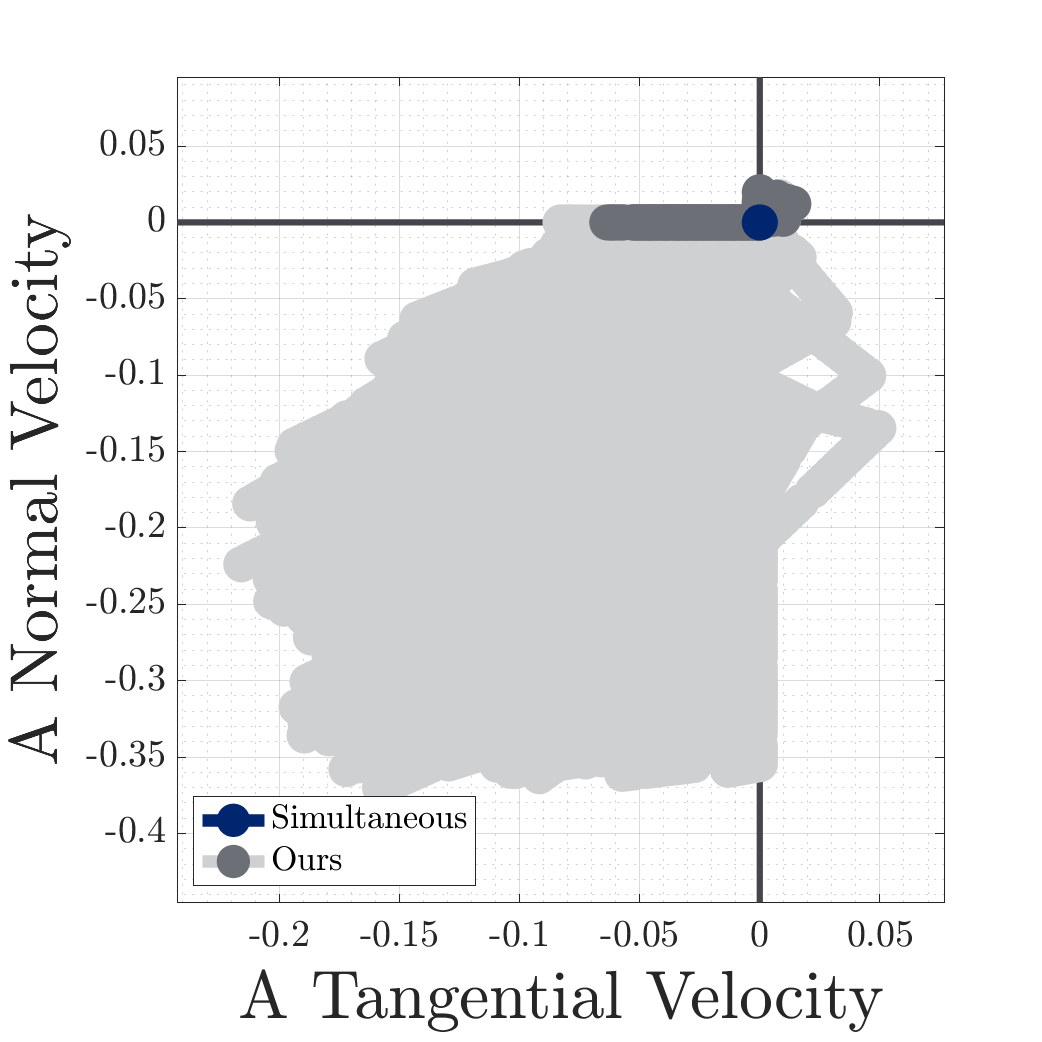}
    	\includegraphics[width=0.32\hsize]{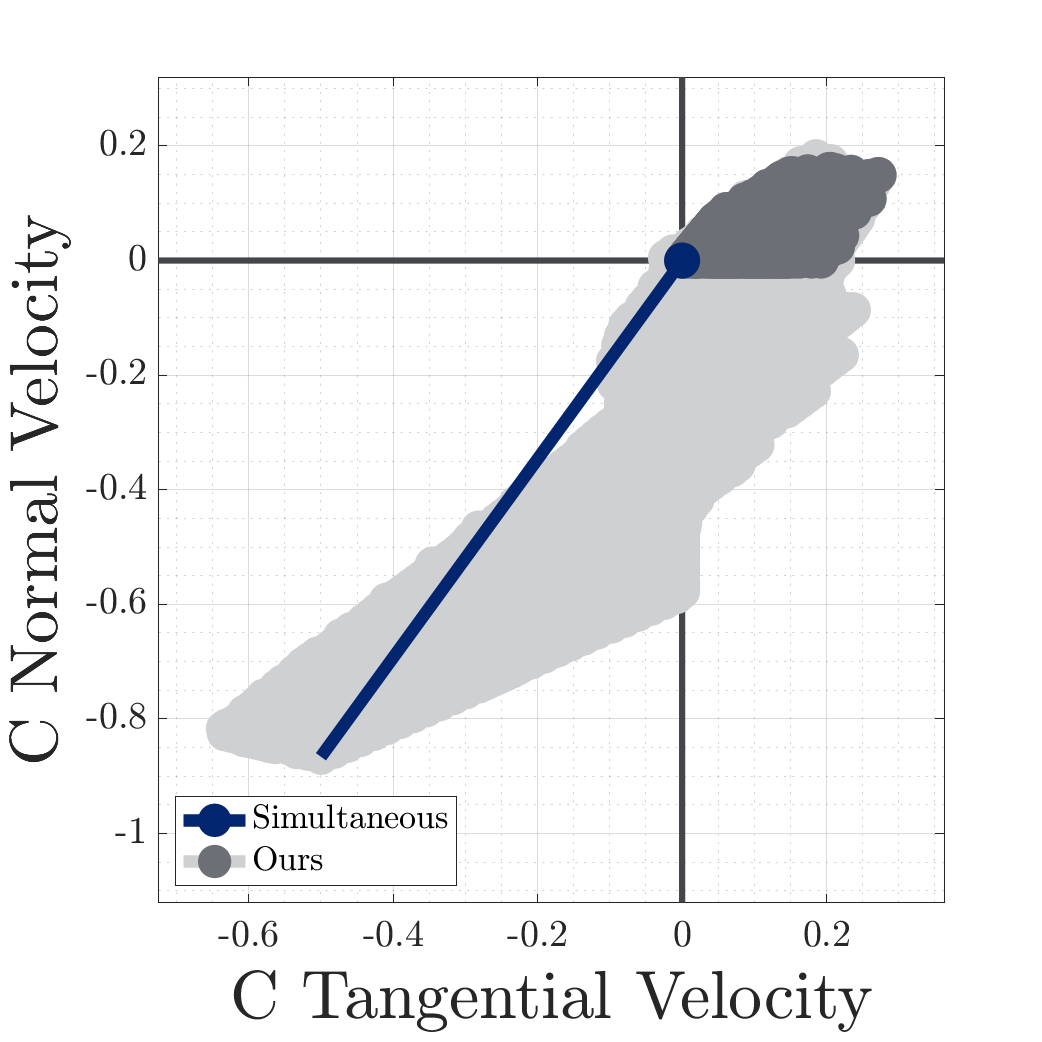}
    	\includegraphics[width=0.32\hsize]{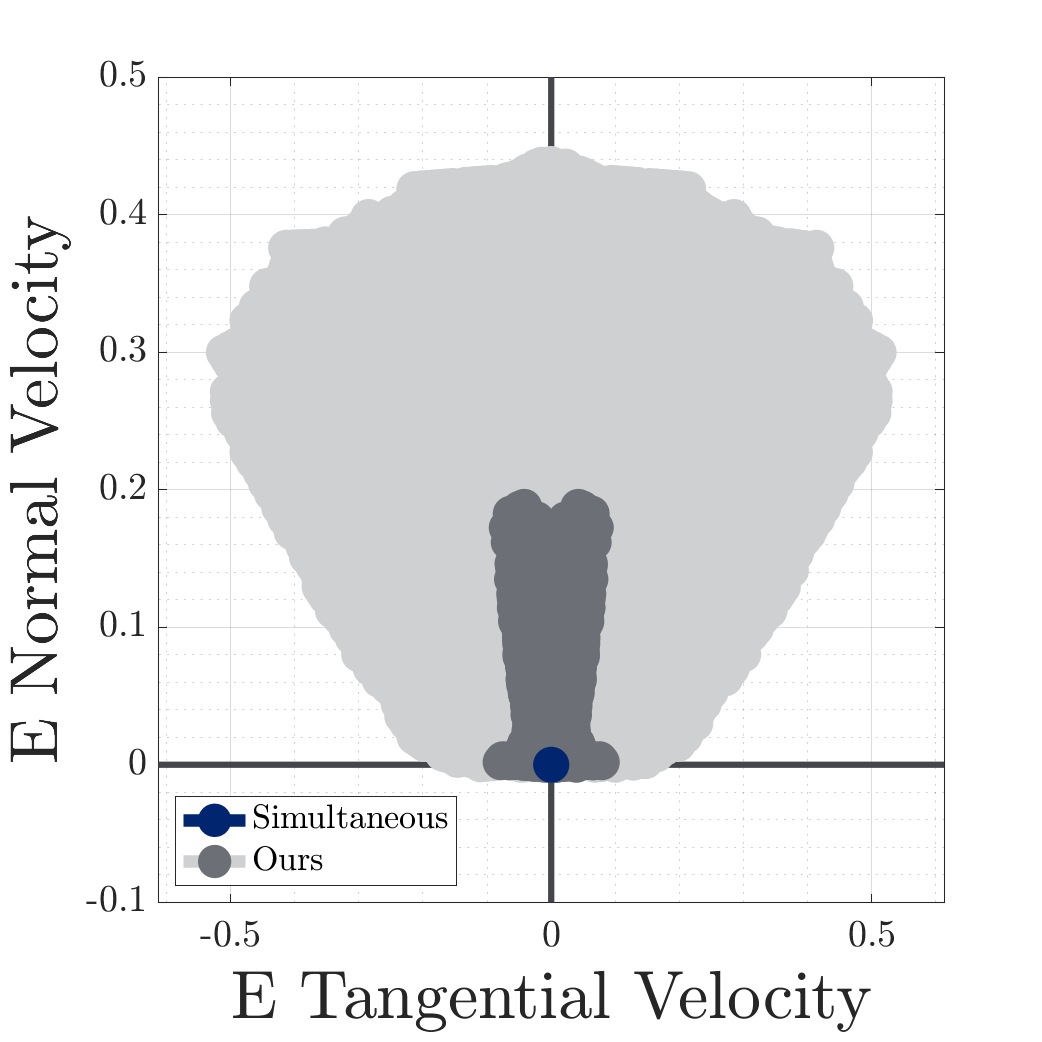}
        \caption{\label{fig:stack_individual_evolution} Impact resolution process at individual contact points}
    \end{subfigure}
    
    \caption{Evolution of a stack of disks (\subref*{fig:stack_diagram}) as the top disk fall on the bottom two. \citet{Anitescu97} only predicts that the entire system comes to rest. (\subref*{fig:stack_side_evolution}) Our method additional predicts several scenarios where the top disk remains in contact with only one of the bottom two disks, while the other may roll away or even lift off the ground slightly. (\subref*{fig:stack_individual_evolution}) various states of rolling, sliding, and lifting contact are shown for points A, C, and E; plots for B and D are omitted as they are symmetric with A and C, respectively.\label{fig:stack}}

\end{figure*}

\subsubsection{Disk Stacking}
In this example, we demonstrate our ability to generate non-unique results in a multi-object scenario motivated by manipulation: stacking disks.
A tower of 3 discs (Figure \ref{fig:stack}) is created by dropping a disk on two others, which rest on the ground.
The only prediction for this 5-contact collision offered by \citet{Anitescu97} is the entire tower coming to rest.
While we cannot be sure that the numerical results cover all possible outputs of our model, we are able to generate various outcomes in which the tower falls apart.
Figure \ref{fig:stack_side_evolution} shows how the post-impact normal velocities compare in the left and right sides of the tower.
The top ball always maintains contact with at least one of the left or right balls, and one of those balls always stays on the ground.
The contacts that are maintained may slide, while the ball on the opposite side may even lift off the ground (Figure \ref{fig:stack_individual_evolution}).

\section{Conclusion}
Non-unique behavior is a pervasive complexity that is present in both real-world robotic systems and common models capturing frictional impacts between rigid bodies---and thus accurate incorporation of such phenomena is an essential component of robust planning, control, and estimation algorithms.
Our model presents a state-of-the-art theoretical foundation for capturing these set-valued outcomes.
Despite the high versatility of allowing impacts to resolve at arbitrary relative rates, both the continuous-time formulation and simulation method have termination guarantees.

Future development of our model will focus on capturing a wider array of contact-driven behaviors; improved theoretical guarantees; and more efficient computational approaches.
For instance, while many models in robotics assume impacts are inelastic, capturing restitution would increase the accuracy of our model for some systems.
Several different approaches have been used to extend Routh's method to capture partially elastic impacts which are guaranteed to dissipate kinetic energy \citep{Liu2008,Liu2008_2,Mirtich1996}.
In particular, \citet{Mirtich1996} combines Routh's model with Stronge's energetic coefficient of restitution \citep{Stronge90}, which relates the quantity work done in compression and restitution phases of elastic impacts.
One possible variant of our models would allow for Stronge-type restitution to resolve at arbitrary relative rates in a similar fashion to the compression-only setting of our inelastic models.

Additionally, while approximation of the post-impact set is probabilistically complete, a straightforward process for computing the sample count $M$ and trajectory lengths $N$ required for particular approximation constants $\varepsilon, \delta$ is not known.
This proof thus did not inform useful computational settings in our examples.
Future research could instead develop outer approximations of the post-impact set via Lyapunov reachability and semidefinite programming \citep{Posa16}.
\label{sec:conclusion}

\begin{funding}
The authors disclosed receipt of the following financial support
for the research, authorship, and/or publication of this article: This work was supported by the National Science Foundation under Grant No. CMMI-1830218; and an NSF Graduate Research Fellowship under Grant No. DGE-1845298.
\end{funding}

\theendnotes
\Urlmuskip=0mu plus 1mu\relax
\bibliographystyle{SageH}
\raggedright
\bibliography{matt,library}
\justifying
\appendix

\section{Example Details}\label{adx:exampledetails}
Here, we list relevant details on the examples in Section \ref{sec:examples}.
The MATLAB codebase at \texttt{\url{https://github.com/mshalm/routh-multi-impact}} may be run via \texttt{Results()}.
The PATH LCP solver \citep{Dirske1995} must be installed, and \texttt{pathlcp.m} must be available from the MATLAB path.
Geometric, inertial, and simulation parameters of the examples are listed in Tables \ref{table:PhoneDetails}--\ref{table:DiskStackingDetails}, and the listed symbols match the variable names used in the codebase.
Unless otherwise stated, bodies have uniform density, and links are massless.
For the RAMone example, we refer the reader to \citet{Remy2017} for a full description of the system's inertial and geometric properties.
\begin{table}
\centering
	\caption{Rocking block parameters}
	\label{table:PhoneDetails}
	\centering
	\begin{tabular}{l c r}
		\toprule
		Parameter & Symbol & Value \\
		\midrule
		Block width & $w$ & \SI{1}{\meter} \\
		Block height & $h$ &\SI{2}{\meter} \\
		Block mass & $m$ & \SI{1}{\kilo\gram} \\
		Init. downward vel. & $v_0$ & \SI{.4429}{\meter\per\second} \\
		Friction coefficient & $\mu$ & \SI{1}{} \\
		Step size & $h$ & \SI{0.3}{\newton\second}\\
		Trajectory length & $N$ & 10\\
		Trajectory set size & $M$ & $2^{14}$\\
		\bottomrule
	\end{tabular}
\end{table}
\begin{table}
\caption{Compass gait parameters}
	\label{table:CompassGaitDetails}
	\centering
	\begin{tabular}{l c r}
		\toprule
		Parameter & Symbol & Value \\
		\midrule
		Leg length & $l$ & \SI{1}{\meter} \\
		Mass-to-foot length &$s_{\parallel}$ & \SI{0.5}{\meter} \\
		Leg mass & $m$ & \SI{1}{\kilo\gram} \\
		Trailing leg pitch & $\varphi_{tr}$ & \SI{78}{\degree} \\
		Leading leg pitch & $\varphi_{le}$ & \SI{-78}{\degree} \\
		Trailing leg init. ang. vel. & $\dot \varphi_{tr,0}$ & \hspace{-2mm}\SI{0.25}{\radian\per\second} \\
		Leading leg init. ang. vel. & $\dot \varphi_{le,0}$ & \hspace{-2mm}\SI{0.25}{\radian\per\second} \\
		Friction coefficient & $\mu$ & \SI{5}{} \\
		Step size  & $h$ & \SI{.25}{\newton\second}\\
		Trajectory length & $N$ & 5\\
		Trajectory set size &  $M$ & $2^{20}$\\
		\bottomrule
	\end{tabular}
\end{table}
\begin{table}
\centering
	\caption{Box and wall parameters}
	\label{table:BoxAndWallDetails}
	\centering
	\begin{tabular}{l c r}
		\toprule
		Parameter & Symbol & Value \\
		\midrule
		Box side length & $w$ & \SI{1}{\meter} \\
		Box mass & $m$ & \SI{1}{\kilo\gram} \\
		Angle from box to ground & $\theta$ & \SI{10}{\degree} \\
		Init. horizontal velocity & $v_0$ & \SI{1}{\meter\per\second} \\
		Friction coefficient & $\mu$ & \SI{1}{} \\
		Step size & $h$ & \SI{2}{\newton\second}\\
		Trajectory length & $N$ & 5\\
		Trajectory set size & $M$ & $2^{18}$\\
		\bottomrule
	\end{tabular}
\end{table}
\begin{table}
\centering
	\caption{RAMone parameters}
	\label{table:RAMoneDetails}
	\centering
	\begin{tabular}{l c r}
		\toprule
		Parameter & Symbol & Value \\
		\midrule
		Trunk pitch & $\Phi$ & \SI{16}{\degree} \\
		Leading hip angle & $\alpha_{le}$ & \SI{-70}{\degree} \\
		Trailing hip angle &$\alpha_{tr}$ & \SI{70}{\degree} \\
		Leading knee angle &$\beta_{le}$ & \SI{-2}{\degree} \\
		Trailing knee angle &$\beta_{tr}$ & \SI{-92.48}{\degree} \\
		Trunk init. $x$ vel. & $\dot x_0$ & \hspace{-2mm}\SI{-0.4114}{\meter\per\second} \\
		Trunk init. $y$ vel. & $\dot y_0$ &\hspace{-2mm}\SI{-0.2105}{\meter\per\second} \\
		Trunk init. ang. vel. & $\dot \Phi_0$ &\SI{1}{\radian\per\second} \\
		Leading hip init. vel. &$\dot \alpha_{le,0}$ & \SI{0}{\radian\per\second} \\
		Trailing hip init. vel. &$\dot \alpha_{tr,0}$ & \SI{0}{\radian\per\second} \\
		Leading knee init. vel. &$\dot \beta_{le,0}$ & \SI{0}{\radian\per\second} \\
		Trailing knee init. vel. &$\dot \beta_{tr,0}$ & \SI{0}{\radian\per\second} \\
		Friction coefficient & $\mu$ & $10^5$ \\
		Step size & $h$ & \SI{1}{\newton\second}\\
		Trajectory length & $N$ & 10\\
		Trajectory set size & $M$ & $2^{20}$\\
		\bottomrule
	\end{tabular}
\end{table}
\begin{table}
\centering
	\caption{Disk stacking parameters}
	\label{table:DiskStackingDetails}
	\centering
	\begin{tabular}{l c r}
		\toprule
		Parameter & Symbol & Value \\
		\midrule
		Disk radius & $R$ & \SI{1}{\meter} \\
		Disk mass & $m$ & \SI{1}{\kilo\gram} \\
		Initial vertical velocity & $v_0$ & \SI{-1}{\meter\per\second} \\
		Friction coefficient & $\mu$ & $\sqrt{3}$ \\
		Step size & $h$ & \SI{1}{\newton\second}\\
		Trajectory length & $N$ & 10\\
		Trajectory set size& $M$ & $2^{20}$\\
		\bottomrule
	\end{tabular}
\end{table}

\section{Additional Background and Proofs}\label{adx:notation}
In this appendix, we discuss aditional theoretical background and proofs of technical lemmas necessary for detailed understanding of the theoretical results in this thesis. 
We equip domains $\Domain \subseteq \Real^n$ with the Euclidean ($l_2$) norm; Euclidean metric $d(\State_1, \State_2) = \TwoNorm{\State_2 - \State_1}$, and the Lebesgue measure.
The total derivative $\dot{\vect f}(s)$ of an absolutely continuous function $\vect f(s)$ is taken in the Lebesgue sense.
We make use of the following theorem:

\begin{proposition}[Arzel\`a-Ascoli \citep{Rudin1991}]\label{thm:Rellich}
	Consider a uniformly bounded sequence $\Sequence{\vect f}{n}$ of $\Real^n$-valued functions on a compact interval $[a,b]$, where each function $\vect f_n$ is Lipschitz continuous with the same constant $L$. Then, there exists a subsequence $\Parentheses{{\vect f}_{n_k}}_{k \in \Natural}$ that converges uniformly.
\end{proposition}

\subsection{Proof of \Cref{prop:epsilonnet}}\label{adx:epsilonnetproof}
Let $g(x): \Real^n \to \Real^m$ be Lipschitz with constant $L$ and let $h>0$. Let $\StateSet = \Braces{x_1,\dots,x_N}$ be a set of uniform i.i.d. samples from $[0,h]^n$.
By Lipschitz continuity, $g(\StateSet)$ is a $\varepsilon$-net of $g([0,h]^n)$ if $\StateSet$ is an $\frac{\varepsilon}{L}$-net of $[0,h]^n$; we examine the latter.

Consider a regularly-spaced grid of cardinality $M^n$:
\begin{equation}\textstyle
	\StateSet' = \Braces{\frac{h}{2M},\frac{3h}{2M},\dots,\frac{(2M-1)h}{2M}}^n\,.
\end{equation}
$\StateSet'$ is a $\frac{h\sqrt{n}}{2M}$-net of $[0,h]^n$. 
Thus, setting $M = \left\lceil \frac{hL\sqrt{n}}{\varepsilon}\right\rceil$, $\StateSet'$ is an $\frac{\varepsilon}{2L}$-net of $[0,h]^n$.
Consider the case where for each $x \in \StateSet'$, $\StateSet$ contains an $x_i$ with
\begin{equation}\textstyle
	x_i \in x + \Brackets{-\frac{h}{2M},\frac{h}{2M}}^n \subseteq [0,h]^n\,,\label{eq:surroundingsquare}
\end{equation}
and thus $\TwoNorm{x_i-x}\leq\frac{\varepsilon}{2L}$.
Then by triangle inequality, $\StateSet$ is an $\frac{\varepsilon}{L}$-net of $[0,h]^n$ when \eqref{eq:surroundingsquare} holds for each $x_i$.
For a single $x \in \StateSet'$, as the elements of $\StateSet$ are chosen uniform i.i.d, the probability of \eqref{eq:surroundingsquare} \textit{not} holding is $(1 - M^{-n})^N$.
The probability of \eqref{eq:surroundingsquare} holding for every $x$ is at least
\begin{equation}
	1 - M^n(1 - M^{-n})^N\,,
\end{equation}
by union bound. The proof holds as $M^{-n} = \Omega$.

\subsection{Proof of \Cref{lem:singlefrictional}}\label{adx:singlestopproof}
	We may assume without loss of generality (WLOG) that $\Mass = \Identity$ by applying a coordinate transformation of $\Mass^{\frac{1}{2}}$ to $\Velocity$.
	Let $\vect R$ be a matrix with columns that constitute an orthogonal basis of $\RangeSpace{\J[i]^T}$. By equivalence of norms there exists $\varepsilon > 0$ such that
	\begin{equation}
		\Norm{\Jn[i]\Velocity} + \TwoNorm{\Jt[i]\Velocity} \geq \varepsilon\TwoNorm{\vect R^T\Velocity} .
	\end{equation}
	We will show that $\kappa = \Parentheses{\varepsilon\min\Parentheses{\FrictionCoeff[i],1}}^{-1}$ satisfies the claim.
	Let $V(s) = \TwoNorm{\vect R^T\Velocity(s)}^2$.
	Assume WLOG that $\Velocity(s)$ is a colling velocity ($\Velocity(s) \in \ActiveSet(\Configuration)$) at least until $s^* = \TwoNorm{\vect R^T\Velocity(0)}\kappa \leq \TwoNorm{\Velocity(0)}\kappa$.
	Then, on the interval $[0,s^*)$,
	\begin{align}\textstyle
		\dot V &= 2 \dot \Velocity^T \vect R \vect R^T \Velocity\,,\\
		&\in  2\Parentheses{\Jn[i] - \FrictionCoeff[i] \Unit\Parentheses{\Jt[i]\Velocity}^T\Jt[i]\ } \vect R \vect R^T \Velocity\,,\\
		&= - 2\Norm{\Jn[i]\Velocity} - 2\FrictionCoeff[i] \TwoNorm{\Jt[i]\Velocity}\,,\\
		&\leq -2\varepsilon\min\Parentheses{\FrictionCoeff[i],1}\sqrt{V}\,,\\
		&\leq -\frac{2}{\kappa}\sqrt{V}\,.
	\end{align}
	 The unique solution to the IVP $\dot x = -\frac{2}{\kappa}\sqrt{x}$,
	\begin{equation}
		x(s) = \Parentheses{\sqrt{x(0)} - \frac{s}{\kappa}}^2\,,
	\end{equation}
	therefore bounds $V$ from above on $[0,s^*)$.
Thus,
\begin{align}\textstyle
	V\Parentheses{s^*} &\leq \Parentheses{\sqrt{V(0)} - \frac{s^*}{\kappa}}^2\,,\\
	&= \Parentheses{\TwoNorm{\vect R^T\Velocity(0)} - \TwoNorm{\vect R^T\Velocity(0)}}^2\,,\\
	&= 0\,.
\end{align}
Therefore $\vect R^T\Velocity(s^*) = \ZeroVector$, $\Jn[i]\Velocity\Parentheses{s^*} = 0$, and $\Velocity\Parentheses{s^*} \not \in \ActiveSet(\Configuration)$.
\section{Impact Model Proofs and Lemmas}
\subsection{Proof of \Cref{lem:NonEmptyClosure}}\label{adx:NonEmptyClosureProof}
The final claim may be reached via direct application of \Cref{thm:Rellich}, as long as $\DerivativeMap[\Configuration](\Velocity)$ is non-empty, uniformly bounded, closed-valued, convex-valued, and u.s.c. We will demonstrate that each of these properties hold.

We first observe that the set of contacts $\ContactSet_{\Configuration}(\Velocity)$, used in the construction of $\DerivativeMap[\Configuration](\Velocity)$ in \eqref{eq:multicontact}, is non-empty by construction.
Furthermore, $\ContactSet_{\Configuration}(\Velocity)$ is u.s.c. in $\Velocity$, because it is constructed from non-strict inequalities of linear functions of $\Velocity$.
Next, we note that for each $i$, $\NetForce[i](\Configuration, \Velocity, 1)$ is non-empty, uniformly bounded, closed-valued, and u.s.c. as it is an affine transformation of $\Unit(\cdot)$.
Finally, we characterize $\DerivativeMap[\Configuration](\Velocity)$. $\DerivativeMap[\Configuration](\Velocity)$ is non-empty, uniformly bounded, and close-convex valued, by construction from the convex hull of a non-empty union of $\NetForce[i](\Configuration, \Velocity, 1)$.
Now, consider an arbitrary velocity $\Velocity_0$ and neighborhood $\dot \VelocitySet_0 \supset \DerivativeMap[\Configuration](\Velocity_0)$. As $\ContactSet_{\Configuration}(\Velocity)$ is u.s.c., we can select a neighborhood $\VelocitySet$ with $\ContactSet_{\Configuration}(\VelocitySet) \subseteq \ContactSet_{\Configuration}(\Velocity_0)$. Therefore on $\VelocitySet$,
\begin{equation}
	\DerivativeMap[\Configuration](\Velocity) \subseteq \DerivativeMap[0](\Velocity) = \Hull\Parentheses{\cup_{i\in\ContactSet_{\Configuration}(\Velocity_0)} \NetForce[i](\Configuration,\Velocity,1)}\,.
\end{equation}
$\DerivativeMap[0](\Velocity)$ is u.s.c. as the convex hull of u.s.c. functions, and furthermore $\DerivativeMap[\Configuration](\Velocity_0) =  \DerivativeMap[0](\Velocity_0)$. Therefore by definition of u.s.c. there exists a neighborhood $\VelocitySet_0$ of $\Velocity_0$ such that
\begin{equation}
	\DerivativeMap[\Configuration](\VelocitySet_0) \subseteq \DerivativeMap[0](\VelocitySet_0) \subseteq \dot \VelocitySet_0\,.
\end{equation}
$\DerivativeMap[\Configuration](\VelocitySet)$ is therefore by definition u.s.c. and the claim is satified.

\subsection{Homogeneity Lemma}\label{adx:homogeneityproof}
As the set of allowable contact forces are only dependent on the direction of $\Velocity$, $\DerivativeMap[\Configuration](\Velocity)$ is positively homogeneous in $\Velocity$, in the sense that $\forall k>0,\Velocity \in \VelocitySpace$, $\DerivativeMap[\Configuration](\Velocity) = \DerivativeMap[\Configuration](k\Velocity)$. Positive homogeneity induces a similar property on the solution set to the differential inclusion:
\begin{lemma}[Homogeneity]\label{lem:homogeneity}
	For all $\Configuration$, $k > 0$, and $[a,b]$ compact, if $\Velocity(s) \in \SolutionSet{\DerivativeMap[\Configuration]}[[a,b]]$, $k\Velocity(\frac{s}{k}) \in \SolutionSet{\DerivativeMap[\Configuration]}[[ka,kb]]$.
\end{lemma}
\begin{proof}
Consider a configuration $\Configuration \in \ConfigurationSpace$ and compact interval $[a,b]$.
We first demonstrate that the impact DI mapping $\DerivativeMap[\Configuration](\Velocity)$ is positively homogeneous in $\Velocity$ ($\DerivativeMap[\Configuration](\Velocity) = \DerivativeMap[\Configuration](k \Velocity)$ for $k>0$).
The DI mapping $\NetForce[i](\Configuration, \Velocity, 1)$ for contact $i$ is an affine transform of $\Unit(\Jt[i]\Velocity)$ and thus positively homogeneous.
Also, the the set of contacts $\ContactSet_{\Configuration}(\Velocity)$ used in the construction of $\DerivativeMap[\Configuration](\Velocity)$ in \eqref{eq:multicontact}, is also positively homogeneous in $\Velocity$. Therefore, $\DerivativeMap[\Configuration](\Velocity) = \Mass^{-1}\Hull\Parentheses{\cup_{i \in \ContactSet_{\Configuration}} \NetForce[i]}$ is positively homogeneous.

Now, consider a solution $\Velocity(s)$ to the impact DI $\dot \Velocity \in \DerivativeMap[\Configuration](\Velocity)$ over $[a,b]$, and $k > 0$. $k\Velocity(\frac{s}{k})$ is well-defined and absolutely continuous over the interval $[ka,kb]$, and has derivative equal to $\dot\Velocity(\frac{s}{k})$ a.e. on $[ka,kb]$. Then $\dot\Velocity(\frac{s}{k}) \in \DerivativeMap[\Configuration](\Velocity(\frac{s}{k})) = \DerivativeMap[\Configuration](k\Velocity(\frac{s}{k}))$ a.e., and $k\Velocity(\frac{s}{k}) \in \SolutionSet{\DerivativeMap[\Configuration]}[[ka,kb]]$.
\end{proof}

\subsection{Proof of \Cref{lem:dissipate}}\label{adx:dissipateproof}
Let $\Configuration \in \ConfigurationSet_A$, and let $[a,b]$ be a compact interval.
Consider a solution of the impact DI $\Velocity(s) \in \SolutionSet{\DerivativeMap[\Configuration]}[[a,b]]$ with non-separating velocity ($\Velocity([a,b]) \subseteq \Closure \ActiveSet(\Configuration)$).
We will show that $\Norm{\Velocity(s)}_{\Mass}$ is non-increasing by proving that $\dot \KineticEnergy(\Configuration,\Velocity(s))$ is non-positive almost everywhere.
Pick any $s \in [a,b]$ where $\dot \Velocity(s) \in \DerivativeMap[\Configuration](\Velocity(s))$.
By construction of $\DerivativeMap[\Configuration](\Velocity)$ \eqref{eq:multicontact} and the definition of the convex hull, there exists coefficients $c_i \geq 0$ such that
\begin{equation}
	\Mass\dot \Velocity(s) \in \sum_{i : \Jn[i] \Velocity(s) \leq 0} c_i\NetForce[i](\Configuration, \Velocity(s), 1)\,.
\end{equation}
We observe by chain rule that
\begin{equation}
	\dot K = \Velocity^T\Mass\dot\Velocity \in  
	\sum_{i : \Jn[i] \Velocity \leq 0} c_i\Velocity^T\NetForce[i](\Configuration, \Velocity, 1)\,.
\end{equation}
$\dot \KineticEnergy$ is then non-positive as each term in this sum is non-positive by construction of $\NetForce[i]$ and \eqref{eq:NormalDissipationGuaranteed}:
\begin{equation}
	\Velocity^T\NetForce[i](\Configuration, \Velocity, 1)=
	\Velocity^T\Jn[i]^T - \FrictionCoeff[i]\TwoNorm{\Jt[i]\Velocity}\,.
\end{equation}

\subsection{Strong Dissipation Lemma}\label{adx:stopproof}
\begin{lemma}[Strong Dissipation]\label{thm:stop}
	Let $\Configuration \in \ConfigurationSet_A$, and let $[a,b]$ be compact. If $\Velocity(s) \in \SolutionSet{\DerivativeMap[\Configuration]}[[a,b]]$ and $\Velocity\Parentheses{[a,b]} \subseteq \Closure \ActiveSet(\Configuration)$, $\Norm{\Velocity(s)}_{\Mass}$ constant implies $\Velocity(s)$ constant.
\end{lemma}
\begin{proof}
Let $\Configuration \in \ConfigurationSet_A$ be a configuration with active contact, and $\Velocity(s) \in \SolutionSet{\DerivativeMap[\Configuration]}[[a,b]]$ a solution to the associated impact DI with impacting velocity ($\Velocity([a,b]) \subseteq \Closure \ActiveSet(\Configuration)$).
Let $\Force(s)$ be the associated vector of force variables defined a.e..

Assume that $\Velocity(s)$ is non-constant.
We may now prove the claim by showing that $\Norm{\Velocity(b)}_{\Mass} < \Norm{\Velocity(a)}_{\Mass}$.
As $\Velocity(s)$ is continuous, we may select $a < s^* < b$ such that $\forall \delta > 0$, $\Velocity(s)$ is not constant on $[s^*, s^* + \delta]$.
Let $A = \Braces{i \in \ContactSet_A (\Configuration) : \Jn[i]\Velocity(s^*) \leq 0}$ be the set of non-separating contacts at $s = s^*$.
Let $B$ be the set of contacts $b \in A$ with zero contact velocity ($\J[b]\Velocity(s^*) = \ZeroVector$).
As $\Velocity(s)$ is continuous, $\exists \delta_\varepsilon > 0$ and $\varepsilon > 0$ such that $\forall s \in [s^*,s^* + \delta_\varepsilon] \subseteq [a,b]$,
	\begin{itemize}
		\item All $i \in \ContactSet_A \setminus A$ separate ($\Jn[i]\Velocity(s) > \varepsilon$)
		\item All $i \in A \setminus B$ move: $\Jn[i]\Velocity(s) < -\varepsilon$ or $\TwoNorm{\Jt[i]\Velocity(s)} > \frac{\varepsilon}{\FrictionCoeff[i]}$.
	\end{itemize}
	Select $s$ from $[s^*,s^* + \delta_\varepsilon]$ with $\Velocity(s) \neq \Velocity(s^*)$. By \Cref{lem:dissipate},
	\begin{align}
		0 &\geq \Norm{\Velocity(s)}_{\Mass}^2 - \Norm{\Velocity(s^*)}_{\Mass}^2,\\
						&= 2\Velocity(s^*)^T\Mass\Parentheses{\Velocity(s) - \Velocity(s^*)} + \Norm{\Velocity(s) - \Velocity(s^*)}_{\Mass}^2,\\
						&= 2\Parentheses{\J \Velocity (s^*)}^T
						\int_{s^*}^s \lambda(\sigma)\Differential \sigma +\Norm{\Velocity(s) - \Velocity(s^*)}_{\Mass}^2.\label{eq:dissipationNormInequality}
	\end{align}
	Therefore, there must exist a contact $a \in A \setminus B$ with $\int_{s^*}^s \NormalForce[a] > 0$ as \eqref{eq:dissipationNormInequality} is non-positive. Finally,
	\begin{align}
		K(\Velocity(s)) &= K(\Velocity(s^*)) + \int_{s^*}^{s} (\J \Velocity(\tau))^T\Force(\sigma) \Differential \sigma\,, \\
			   &\leq K(\Velocity(s^*)) - \varepsilon \int_{s^*}^s \NormalForce[a](\sigma) \Differential \sigma\,, \\  &< K(\Velocity(s^*))\,.
	\end{align}
	Therefore, as $\KineticEnergy$ is non-increasing, $\Norm{\Velocity(b)}_{\Mass} < \Norm{\Velocity(a)}_{\Mass}$.
\end{proof}
\subsection{Proof of \Cref{thm:strictdissipation}}\label{adx:strictdissipationproof}
As $\Velocity$ is never constant on $\Closure \ActiveSet(\Configuration)$ via Assumption \ref{assump:nondegenerate}, $\Norm{\Velocity(s)}_{\Mass}$ strictly decreases by \Cref{thm:stop} and \Cref{lem:dissipate}.
\subsection{Proof of \Cref{lem:exit}}\label{adx:exitproof}
Suppose not, so there exists a configuration $\Configuration \in \ConfigurationSet_A$, dissipation rate $\DissipationRate[\Configuration](s)$ such that $\dot\Velocity \in \DerivativeMap[\Configuration](\Velocity)$ is $\DissipationRate[\Configuration](s)$-dissipative, $s>0$ and $s^* > \Norm{\Velocity(0)}_{\Mass}\frac{s}{\DissipationRate[\Configuration](s)}$, and $\Velocity(s) \in \SolutionSet{\DerivativeMap[\Configuration]}[\Brackets{0, s^*}]$ with $\Velocity([0,s^*]) \subseteq \Closure\ActiveSet (\Configuration)$.
Assume WLOG by \Cref{lem:homogeneity} that $\Norm{\Velocity(0)}_{\Mass} = 1$.
By solution homogeneity (\Cref{lem:homogeneity}) we have for any $s_{k} = s_{k-1} + s\Norm{v(s_k)}_{\Mass}$,
\begin{equation}
		\frac{\Norm{\Velocity\Parentheses{s_k}}_{\Mass}}{\Norm{\Velocity\Parentheses{s_{k-1}}}_{\Mass}} \leq \Parentheses{1-\DissipationRate[\Configuration](s)}\,.
\end{equation}
Setting $s_0 = 0$, we thus have
\begin{align}
	\Norm{\Velocity\Parentheses{s_k}}_{\Mass} & \leq \Parentheses{1-\DissipationRate[\Configuration](s)}^k\,,\\
	s_k & \leq s\sum_{i=1}^k \Parentheses{1-\DissipationRate[\Configuration](s)}^{k-1}\,.
\end{align}
Therefore $s_\infty = \lim_{k\to\infty} s_k = \frac{s}{\DissipationRate[\Configuration](s)} < s^*$ and by continuity of $\Velocity$, $\Velocity\Parentheses{s_\infty} = \ZeroVector$.
But then by \Cref{thm:strictdissipation} $\Norm{\Velocity(s)}_{\Mass}$ must decrease below $0$ on $\Brackets{s_\infty,s^*}$, a contradiction.
\subsection{Proof of \Cref{thm:nondegeneratedissipation}}
\label{adx:nondegeneratedissipationproof}
Suppose not.
Then there exists a $\Configuration \in \ConfigurationSet_A \setminus \ConfigurationSet_P$, an $s_f > 0$, and a corresponding sequence of solutions $\Parentheses{\Velocity^j(s)}_{j \in \Natural}$, $\Velocity^j(s) \in \SolutionSet{\DerivativeMap[\Configuration]}[[0,s_f]]$, all starting with $\Norm{\Velocity^j(0)}_{\Mass} = 1$ and never exiting $\Closure \ActiveSet(\Configuration)$, which dissipate less and less energy:
\begin{equation}
	\lim_{j \to \infty} \Norm{\Velocity^j(s_f)}_{\Mass} = 1\,.\label{eq:LimitingNormIsOne}
\end{equation}
As $\DerivativeMap[\Configuration]$ is bounded and each solution $\Velocity^j(s)$ never exits $\Braces{\Norm{\Velocity^j(s)}_{\Mass} \leq 1}$ by dissipation (\Cref{lem:dissipate}), the sequence is equicontinuous and bounded. 
By \Cref{thm:Rellich}, a subsequence of $\Velocity^j(s)$ converges uniformly to a function $\Velocity_\infty(s)$, with $\Norm{\Velocity_\infty(s_f)}_{\Mass} = 1$ by \eqref{eq:LimitingNormIsOne}. Because kinetic energy is non-increasing, $\Norm{\Velocity_\infty(s)}_{\Mass} = 1$ for all $s \in [0,s_f]$.
By \Cref{lem:NonEmptyClosure} $\Velocity_\infty(s)$ is a solution to $\dot \Velocity \in \DerivativeMap[\Configuration](\Velocity)$.
Therefore as $\Velocity_\infty(s)$ does not dissipate kinetic energy, it is constant by \Cref{thm:stop}, and thus $\ZeroVector \in \DerivativeMap[\Configuration](\Velocity_\infty(s))$.
But as each $\Velocity^j(s) \in \Closure \ActiveSet(\Configuration)$, we must also have $\Velocity_\infty(s) \in \Closure \ActiveSet(\Configuration)$, which contradicts Assumption \ref{assump:nondegenerate}.

\subsection{Proof of \Cref{coro:uniform}}\label{adx:uniformproof}
	Let $\ConfigurationSet \subseteq \ConfigurationSet_A \setminus \ConfigurationSet_P $ be a compact set of non-penetrating configurations with active contact. We will construct a suitable $\DissipationRate[\ConfigurationSet]$ explicitly. Let $s_f>0$. Define the DI
\begin{equation}
	\begin{bmatrix}
		\dot \Configuration \\ \dot \Velocity
	\end{bmatrix} = \dot \State \in \DerivativeMap '(\State) = \begin{bmatrix}
		\ZeroVector \\ \DerivativeMap[\Configuration](\Velocity)
	\end{bmatrix}\,.\label{eq:UniformDissipationInclusion}
\end{equation}
As the set of active contacts $\ContactSet_A(\Configuration)$ is u.s.c. and $\Mass,\J$ are continuous, $D'$ is compact-convex, uniformly bounded, and u.s.c..
Now consider the sets
\begin{align}
	\StateSet_0 &= \Braces{\begin{bmatrix}
		\Configuration_0 \\ \Velocity_0
	\end{bmatrix} : \Configuration_0 \in \ConfigurationSet \wedge \Norm{\Velocity_0}_{\Mass(\Configuration_0)} = 1 }\,,\\
	\StateSet_{f} &= \Braces{\State(s_f) : \State(s) \in \IVP{\DerivativeMap '}{\StateSet_0}{[0,s_f]}}\,.
\end{align}
$\StateSet_0$ represents all initial conditions with configurations in $\ConfigurationSet$ and initial kinetic energy $\frac{1}{2}$, and $\StateSet_f$ is set of states reachable from $\StateSet_0$ via solutions to the dynamics \eqref{eq:UniformDissipationInclusion} for a duration $s_f$.
As $\StateSet_0$ is compact, $\IVP{\DerivativeMap '}{\StateSet_0}{[0,s_f]}$ and therefore $\StateSet_f$ is closed and non-empty by \Cref{prop:closure}.
Any solution $[\Configuration(s);\; \Velocity(s)] \in \IVP{\DerivativeMap '}{\StateSet_0}{[0,s_f]}$ must have constant $\Configuration(s) = \Configuration(0) \in \ConfigurationSet$, because the inclusion \eqref{eq:UniformDissipationInclusion} prescribes $\dot \Configuration = \ZeroVector$. Therefore, $\Velocity(s)$ must be a solution to the associated impact differential inclusion $\dot \Velocity \in \DerivativeMap[\Configuration(0)](\Velocity)$. Therefore, by \Cref{thm:nondegeneratedissipation},
\begin{equation}
	\DissipationRate[\ConfigurationSet](s_f) = 1 - \max_{[\Configuration_f;\; \Velocity_f] \in \StateSet_f} \Norm{\Velocity_f}_{\Mass(\Configuration_f)} \in (0,1]\,,
\end{equation}
where the fact that $\StateSet_f$ is closed implies the strict inequality $\DissipationRate[\ConfigurationSet](s_f) > 0$.
Setting $\DissipationRate[\ConfigurationSet](0) = 0$ and selecting a configuration $\Configuration \in \ConfigurationSet$, we now show that $\dot\Velocity \in \DerivativeMap[\Configuration](\Velocity)$ is $\DissipationRate[\ConfigurationSet](s)$-dissipative.
Let $s_f>0$, $\Norm{\Velocity_0}_{\Mass(\Configuration)} = 1$, and $\Velocity(s) \in \IVP{\DerivativeMap[\Configuration]}{\Velocity_0}{[0,s_f]}$.
By Definition \ref{def:inclusionsolution}, $\State(s) = [\Configuration;\; \Velocity(s)] \in \IVP{\DerivativeMap '}{\StateSet_0}{[0,s_f]}$ and thus $\Norm{\Velocity(s)}_{\Mass(\Configuration)} \leq 1 - \DissipationRate[\ConfigurationSet](s) < 1$ for all $s \in [0,s_f]$.
\section{Continuous-time Model Proofs}
\subsection{Proof of \Cref{thm:ContinuousTimeSolutions}}\label{adx:ContinuousTimeSolutionsProof}
Let $[a,b]$ and $\bar \StateSet$ be compact. As $\DerivativeMap(\bar \State)$ neither depends on $t(\bar \State)$ nor $s$, WLOG $[a,b]=[0,s_f]$ and $t(\bar \State)=0$ for each $\bar \State \in \bar \StateSet$. We will prove that $\IVP{\DerivativeMap}{\bar \State}{[0,s_f]}$ has the claimed properties in the following manner:
\begin{enumerate}
	\item We will bound kinetic energy growth (via Assumption \ref{assump:EnergyIncreasesSlowly}), which will guarantee that solutions starting in $\bar \StateSet$ remain in a larger open bounded set, $\bar \StateSet'$.
	\item We will show that, restricted to $\bar \StateSet'$, $\dot{\bar \State} \in \DerivativeMap(\bar \State)$ is equivalent to another DI $\dot{\bar \State} \in \tilde \DerivativeMap(\bar \State)$ which is compatible with \Cref{prop:closure}.
\end{enumerate}

First, we construct a suitable $\bar \StateSet'$. As $\bar \StateSet$ is compact, we may pick $c>0$ such that $\bar \StateSet \subseteq \Ball_c$. Let $\bar \State(s) \in \IVP{\DerivativeMap}{\bar \StateSet}{[0,s_f]}$. We begin by establishing a bound on $\Velocity(\bar \State(s))$ over $[0,s_f]$. Let $\KineticEnergy(\bar \State) = \KineticEnergy(\Configuration(\bar \State),\Velocity(\bar \State))$. By Assumption \ref{assump:EnergyIncreasesSlowly}, $\exists c_\KineticEnergy >0$ such that for all $\bar \State$,
\begin{equation}
	\PartialDiff{K}{\bar \State}\DerivativeMap[C](\bar \State)  = \Velocity^T\NetForce[s](\State,\InputSet(\bar\State)) \leq \sqrt{2} c_\KineticEnergy  \Norm{\Velocity}_{\Mass}\,.
\end{equation}
As the impact dynamics dissipate energy (\Cref{lem:dissipate}),
\begin{equation}
\dot \KineticEnergy(\bar \State(s)) \in \PartialDiff{K}{\bar \State}\DerivativeMap(\bar \State) \leq 2 c_\KineticEnergy \sqrt{\KineticEnergy (\bar \State)} \,.\label{eq:KineticEnergyGrowthInequality}
\end{equation}
Similar to the argument in Appendix \ref{adx:singlestopproof}, we compare \eqref{eq:KineticEnergyGrowthInequality} to the differential equation $\dot x = 2 c_\KineticEnergy \sqrt{x}$, and upper bound $\KineticEnergy$ as
	\begin{equation}
		\KineticEnergy(\Configuration(s),\Velocity(s)) \leq \Parentheses{\sqrt{\KineticEnergy(\Configuration(0),\Velocity(0))} + c_\KineticEnergy s}^2 \,.
	\end{equation}
Thus, picking $c_{\Mass}$ with $c_{\Mass}^{-1}\Norm{\Velocity}_{\Mass} \leq \sqrt 2 \TwoNorm{\Velocity} \leq c_{\Mass}\Norm{\Velocity}_{\Mass}$,
	\begin{align}
		\TwoNorm{\Velocity(s)} & \leq c_{\Mass}\sqrt{\KineticEnergy(\Configuration(s),\Velocity(s))}\,,\\
		& \leq c_{\Mass}\Parentheses{\sqrt{\KineticEnergy(\Configuration(0),\Velocity(0))} + c_\KineticEnergy s}\,,\\
		& \leq c_{\Mass}^2\TwoNorm{\Velocity(0)}+ c_{\Mass}c_\KineticEnergy s\,.\label{eq:VelocityNormBound}
	\end{align}
Now, we bound $\Configuration(\bar \State(s))$.
As $\TwoNorm{\dot \Configuration} \leq \Norm{\GeneralizedVelocityJacobian}_F\TwoNorm{\Velocity}$, selecting $c_{\GeneralizedVelocityJacobian} = \sup_{\Configuration} \Norm{\GeneralizedVelocityJacobian}_F$, we can apply the triangle inequality as:
\begin{equation}
	\TwoNorm{\Configuration(s)} \leq \TwoNorm{\Configuration(0)} + c_{\GeneralizedVelocityJacobian} s \max_{s'\in[0,s]} \TwoNorm{\Velocity(s')}\,.
\end{equation}
Finally, we bound $\Norm{t(s)} \leq s$ from $\dot t \leq 1$. As $\TwoNorm{\bar \State(0)} < c$,
\begin{align}
\TwoNorm{\bar \State(s)} &\leq \TwoNorm{\Configuration(s)} + \TwoNorm{\Velocity(s)} + \Norm{t(s)}\,,\\
&< c + \Parentheses{ c_{\GeneralizedVelocityJacobian}s_f + 1}\Parentheses{ c_{\Mass}^2c + c_{\Mass}c_\KineticEnergy s_f} + s_f\,.\label{eq:CPrime}
\end{align}
As \eqref{eq:CPrime} is constant, $\bar \State(s)$ remains in a bounded open set $\bar \StateSet'$.

Now, we relate $\dot{\bar \State} \in \DerivativeMap(\bar \State)$ to another DI which can be analyzed via \Cref{prop:closure}.
First we show that $\DerivativeMap(\bar \State)$ is u.s.c.
For any $\Configuration$ and separating velocity $\Velocity \in \InactiveSet(\Configuration)$, we can pick an open neighborhood $\ConfigurationSet \times \VelocitySet$ of $[\Configuration;\; \Velocity ]$ which also consists solely of separating velocities by continuity of $\Gap$ and $\Jn$.
Therefore, the set of separating-velocity states $\SeparatingStateSet$ is open.
Furthermore, each of $\SeparatingStateDerivative(\bar \State)$ (Assumption \ref{assump:ConvexCompactUSCInputs}) and $\ImpactingStateDerivative(\bar \State)$ (\Cref{coro:uniform}) are u.s.c..
$\DerivativeMap(\bar \State)$ must then be u.s.c., because it is constructed from two u.s.c. functions on disjoint open sets, and their convex hull on the remainder of the space.

By Assumption \ref{assump:DerivativeMapCompactImages}, $\DerivativeMap(\bar \StateSet')$ is bounded. As $\bar \StateSet'$ is open and bounded, we can construct a bounded, non-empty, compact-convex valued u.s.c. function $\tilde \DerivativeMap (\bar \State)$ defined over $\Real^n$ such that $\tilde \DerivativeMap |_{\bar \StateSet'} = \DerivativeMap |_{\bar \StateSet'}$.
In particular,
\begin{equation}
	\tilde \DerivativeMap (\bar \State) = \begin{cases}
		\DerivativeMap (\bar \State) & \bar \State \in \bar \StateSet'\,,\\
		\Hull\Parentheses{\Closure\DerivativeMap (\bar \StateSet')} & \Otherwise\,.
	\end{cases}
\end{equation}
Therefore, by \Cref{thm:Rellich}, $\IVP{\tilde \DerivativeMap}{\bar \State}{[0,s_f]}$ is non-empty, closed under uniform convergence, and u.s.c. on $\bar \StateSet$.
As $\DerivativeMap$ and $\tilde \DerivativeMap$ are locally equivalent,
$\IVP{\tilde \DerivativeMap}{\bar \State}{[0,s_f]} = \IVP{\DerivativeMap}{\bar \State}{[0,s_f]}$ on $\bar \StateSet$ and the claim is proven.
\subsection{Proof of \Cref{thm:ContinuousTimeDoesNotPenetrate}}\label{adx:ContinuousTimeDoesNotPenetrateProof}
Suppose not.
Then there exists a non-penetrating initial state $\bar \State_0 = [\Configuration_0;\;\Velocity_0;\; t_0] \not \in \bar \StateSet_P$, compact interval $[0,s_f]$, and corresponding solution $\bar \State(s) = [\Configuration(s);\Velocity(s);t(s)] \in \IVP{\DerivativeMap}{\bar \State_0}{[0,s_f]}$ that penetrates at some $s_P \in (0,s_f]$ ($\bar \State(s_P) \in \bar \StateSet_P$). 
Thus some contact $i \in \ContactsSpace$ penetrates at $s_P$ ($\Gap_i(\Configuration(s_P)) < 0$).
By the intermediate value theorem, we may select $s_A \in [0,s_P)$ such that $\Gap_i(\Configuration(s_P)) < \Gap_i(\Configuration(s_A)) < 0$ and contact $i$ penetrates on the entire interval $[s_A,s_P]$. But then by the definition of $\DerivativeMap(\bar \State)$, $\Configuration(s)$ and therefore $\Gap_i$ must be constant on $[s_A,s_P]$, a contradiction.
\subsection{Proof of \Cref{thm:ContinuousTimeImpactSelection}}\label{adx:ContinuousTimeImpactSelectionProof}
Suppose not.
Then there exists a compact interval $[a,b]$; solution $\bar \State(s) \in \SolutionSet{ \DerivativeMap}[[a,b]]$ with $\bar \State(s)$ impacting but not penetrating, $\bar \State([a,b]) \subseteq \ImpactingStateSet \setminus \bar \StateSet_P $; and set $\IntegrationSet= \Braces{s : \dot {\bar \State}(s) \in \DerivativeMap({\bar \State}(s)) \setminus \ImpactingStateDerivative ({\bar \State}(s))}$ with positive measure.
Furthermore, $\dot t(s)|_\IntegrationSet > 0$ and $\dot \Configuration(s) = \GeneralizedVelocityJacobian(\Configuration(s))\Velocity(s)\dot t(s)$.

We will now show that allowing $\dot t(s)|_\IntegrationSet > 0$ must lead to penetration and therefore a contradiction with \Cref{thm:ContinuousTimeDoesNotPenetrate}.
By Lebesgue's density theorem, we may select a point of density $a < s_1 < b$, i.e., for all $\delta > 0$, $[s_1, s_1 + \delta] \cap \IntegrationSet$ has non-zero measure.
As $\bar \State(s)$ remains in $\ImpactingStateSet$, by continuity of $\J(\Configuration)$ and $\bar \State(s)$ we may select $\delta > 0$ and a contact $i$ that is active $\Gap_i(\Configuration(s)) = 0$ with negative time derivative $\Jn[i]\Velocity(s) < 0$ on $[s_1, s_1 + \delta] \subseteq [a,b]$. Let $\dot \Gap_{\max} = \max_{s\in [s_1,s_1 + \delta]} \Jn[i]\Velocity(s) < 0$. Then
\begin{align}
	\Gap_i(s_1 + \delta) & = \int_{[s_1,s_1 + \delta]} \Jn[i]\Velocity(s)\dot t(s)\Differential s\,,\\
	& \leq \dot \Gap_{\max} \int_{[s_1,s_1 + \delta] \cap \IntegrationSet} \dot t(s) \Differential s\,,\\
	& < 0\,,
\end{align}
and thus $\bar \State(s_1 + \delta) \in \bar \StateSet_P$, a contradiction.
\subsection{Proof of \Cref{thm:ContinuousTimeMinimumAdvancement}}
\label{adx:ContinuousTimeMinimumAdvancementProof}
Let $\bar \StateSet \subseteq \bar \StateSet_P^c$ be compact. By \Cref{coro:uniform} there exists a dissipation rate $\DissipationRate[\bar \StateSet](s)$ such that the impact differential inclusion $\dot\Velocity \in \DerivativeMap[\Configuration](\Velocity)$ for each configuration $\Configuration \in \Configuration(\bar \StateSet)$ is $\DissipationRate[\bar \StateSet](s)$-dissipative. 
Let $\bar K = \max_{\Bar \StateSet} \Norm{\Velocity}_{\Mass(\Configuration)}$.

Suppose the claim is not true.
Then, for some $s_f > s^*(\bar \StateSet) = \frac{\bar K}{\alpha_{\bar X}(1)}$, there must exist a sequence of solutions $\Parentheses{\bar \State_j(s)}_{j \in \Natural}$, $\bar \State_j(s) \in \IVP{\DerivativeMap}{\bar \StateSet}{[0,s_f]}$, for which the elapsed times grows arbitrarily small: $t_j(s_f) - t_j(0) \to 0$.
By \Cref{thm:ContinuousTimeSolutions}, $\IVP{\DerivativeMap}{\bar \StateSet}{[0,s_f]}$ is compact, and therefore by Assumption \ref{assump:DerivativeMapCompactImages}, the derivatives $\dot{\bar \State}_j(s)$ are uniformly bounded. Therefore $\Parentheses{\bar \State_j(s)}_{j \in \Natural}$ is equicontinuous.
Thus by \Cref{thm:Rellich}, a subsequence of $\bar \State_j(s)$ converges uniformly to some $\bar \State_\infty(s)$ with $t_\infty([0,s_f]) = t_\infty(0)$.
As $\IVP{\DerivativeMap}{\bar \StateSet}{[0,s_f]}$ is closed (\Cref{thm:ContinuousTimeSolutions}), $\bar \State_\infty(s)$ must also solve the initial value problem.

We now show a contradiction arises because $\bar \State_\infty(s)$ follows impact dynamics longer than $\frac{\bar K}{\alpha_{\bar X}(1)}$.
As $t_\infty(s)$ is constant, $\dot t_\infty(s) = 0$, and thus $\bar \State_\infty(s)$ is following only impact dynamics, $\bar \State_\infty(s) \in \IVP{\ImpactingStateDerivative }{\bar \StateSet}{[0,s_f]}$.
In order for $\dot{\bar \State}_\infty(s)$ to be selected from $\ImpactingStateDerivative$, we must have $\bar \State_\infty(s) \not \in \SeparatingStateSet$ a.e., and thus $\Velocity_\infty(s) \not\in \InactiveSet(\Configuration_\infty(s))$ a.e.
Additionally, as $\bar \State_\infty(s)$ only follows impact dynamics, the configuration is constant, i.e. $\Configuration_\infty([0,s_f]) = \Configuration_\infty(0) = \Configuration_\infty$.
Therefore $\Velocity_\infty(s)$ is a solution of $\dot \Velocity \in \DerivativeMap[\Configuration_\infty](\Velocity)$,
and $\Velocity_\infty(s) \in \Closure \ActiveSet (\Configuration_\infty)$.
Therefore $\ActiveSet (\Configuration_\infty)$ is non-empty and therefore has active contact ($\Configuration_\infty \in \ConfigurationSet_A$).
As, as $\bar \StateSet$ is closed, $\Configuration_\infty \in \Configuration(\bar \StateSet)$ and thus $\dot \Velocity \in \DerivativeMap[\Configuration_\infty](\Velocity)$ is $\DissipationRate[\bar \StateSet](s)$-dissipative.
Finally, by \Cref{lem:exit}, $s_f < \frac{\Norm{\Velocity_\infty(0)}_{\Mass}}{\DissipationRate[\bar \StateSet](1)} \leq s^*(\bar \StateSet)$, a contradiction.
\subsection{Proof of \Cref{coro:AggregateAdvancement}}
\label{adx:AggregateAdvancementProof}
As $\bar \StateSet (s_f)$ is non-empty and compact for all $s_f>0$, $t_f(s_f)$ is well-defined.
Then, $\lim \inf_{s_f \to \infty} \frac{t_f(s_f)}{s_f} \in [0,1]$ as the DI \eqref{eq:ContinuousTimeModel} enforces $\dot t(s) \in [0, 1]$. 
Consider a particular $s_f>0$, and let $\bar \State(s) \in \IVP{\DerivativeMap}{\bar \StateSet}{[0,s_f]}$.
By \Cref{thm:ContinuousTimeMinimumAdvancement}, $t(s)$ increases by $t^*(\bar \StateSet)$ over each interval of duration $s^*(\bar \StateSet)$, bounding
\begin{equation}
	\frac{t_f(s_f)}{s_f} \geq \frac{t^*(\bar \StateSet)}{s_f}\left\lfloor \frac{s_f}{s^*(\bar \StateSet)} \right\rfloor \geq \frac{t^*(\bar \StateSet)}{s^*(\bar \StateSet)} - \frac{t^*(\bar \StateSet)}{s_f} \,.
\end{equation}
Therefore, $\lim \inf_{s_f \to \infty} \frac{t_f(s_f)}{s_f} \geq \frac{t^*(\bar \StateSet)}{s^*(\bar \StateSet)}$.

\section{Simulation Proofs}
\subsection{Proof of \Cref{thm:TimesteppingSolutionExistence}}
\label{adx:TimesteppingSolutionExistenceProof}
Consider some state $[\Configuration;\;\Velocity]$ and normal impulse $\NormalForce[max] \geq \ZeroVector$.
Let $\LCPVariables = [\SlackForce;\; \bar \Force;\; \SlackVelocity]$
Then we have
\begin{align}
	\LCPVariables^T \LCPMatrix_{\Configuration} \LCPVariables &= \frac{1}{2} \LCPVariables^T \Parentheses{\LCPMatrix_{\Configuration} + \LCPMatrix_{\Configuration}^T} \LCPVariables\,,\\
	&= \Norm{\bar \J^T \bar \Force}_{\Mass^{-1}}^2 + \NormalForce^T\FrictionCoeff\SlackVelocity\,,\\
	&\geq 0\,,
\end{align}
where the final inequality holds because $\FrictionCoeff$ has positive entries and $\Mass \succ 0$. Therefore, $\Mass_{\Configuration}$ is copositive.

Suppose further that $\LCPVariables \in \LCP{\LCPMatrix_{\Configuration}}{\ZeroVector}$, thus $\LCPMatrix_{\Configuration} \LCPVariables \geq \ZeroVector$ and $\LCPVariables^T \LCPMatrix_{\Configuration} \LCPVariables = 0$. $\LCPMatrix_{\Configuration} \LCPVariables \geq \ZeroVector$ implies by construction that
\begin{align}
	\NormalForce &\leq \ZeroVector\,, &
	\OneVectorMatrix^T \FrictionBasisForce &\leq \FrictionCoeff \NormalForce \leq \ZeroVector\,.
\end{align}
Therefore as $\NormalForce, \FrictionBasisForce \geq \ZeroVector$, $\NormalForce = \ZeroVector$ and $\FrictionBasisForce = \ZeroVector$. Finally, as $\NormalForce[max]$ and $\SlackForce$ are non-negative,
\begin{equation}
	\LCPVariables^T \LCPVector_{\Configuration}(\Velocity,\NormalForce[max]) = \SlackForce^T \NormalForce[max] \geq \ZeroVector\,.
\end{equation}
Therefore by \Cref{prop:LCPCopExist}, $\LCP{\LCPMatrix_{\Configuration}}{\LCPVector_{\Configuration}(\Velocity,\NormalForce[max])}$ is non-empty.

\subsection{Proof of \Cref{thm:TimesteppingDissipation}}
\label{adx:TimesteppingDissipationProof}
Consider a state $[\Configuration;\; \Velocity]$, normal impulse increment $\NormalForce[max] \geq \ZeroVector$, and solution to the impact LCP $\LCPVariables = \Brackets{\SlackForce;\;\bar\Force;\;\SlackVelocity} \in \LCP{\LCPMatrix_{\Configuration}}{\LCPVector_{\Configuration}(\Velocity,\NormalForce[max])}$. Let $\Velocity' = \Velocity + \Mass^{-1}\bar\J^T\bar \Force$.
Then from the complementarity condition we have
\begin{align}
	0 &= \LCPVariables^T \Parentheses{ \LCPMatrix_{\Configuration} \LCPVariables + \LCPVector_{\Configuration}(\Velocity,\NormalForce[max])}\,,\\
	& = \Parentheses{\bar\Force^T\bar \J}\Velocity' + \NormalForce^T\FrictionCoeff\SlackVelocity + \SlackForce^T\NormalForce[max] \,, \\
	& = \Parentheses{\Velocity' - \Velocity}^T\Mass \Velocity' + \NormalForce^T\FrictionCoeff\SlackVelocity + \SlackForce^T\NormalForce[max] \,, \\
	& = \Norm{\Velocity'}_{\Mass}^2 - \Velocity^T \Mass \Velocity' + \NormalForce^T\FrictionCoeff\SlackVelocity + \SlackForce^T\NormalForce[max] \,,
\end{align}
As $\NormalForce^T\FrictionCoeff\SlackVelocity + \SlackForce^T\NormalForce[max] \geq 0$, $\Norm{\Velocity'}_{\Mass}^2 \leq \Velocity^T \Mass \Velocity'$.
Cauchy-Schwartz then gives $\Norm{\Velocity'}_{\Mass}^2 \leq \Norm{\Velocity}_{\Mass}\Norm{\Velocity'}_{\Mass}$, and thus $\KineticEnergy(\Configuration,\Velocity') - \KineticEnergy(\Configuration,\Velocity) \leq 0$.
\subsection{Impulse Advancement Lemma}
\label{adx:NonzeroLambdamaxNonzeroLambdaProof}
If $\NormalForce = \ZeroVector$ were allowed by the simulation LCP \eqref{eq:SimulationLCP} at a penetrating velocity $\Velocity \in \ActiveSet (\Configuration)$, then $\Velocity = \Velocity'$ could be selected in an infinite loop, and Algorithm \ref{alg:ImpactSimulation} might never terminate. The structure of the normal impulse constraints \eqref{eq:ImpactSimulationSlackForceComplementairty} and \eqref{eq:ImpactSimulationNormalForceComplementairty} prevents this behavior by design for $\NormalForce[max] > \ZeroVector$:

\begin{lemma}[Impact Advancement (Appendix \ref{adx:NonzeroLambdamaxNonzeroLambdaProof})]\label{lem:NonzeroLambdamaxNonzeroLambda}
	Let $[\Configuration;\; \Velocity]$ be colliding ($\Velocity \in \ActiveSet ( \Configuration)$), and $\NormalForce[max] > \ZeroVector$. Let $\bar \Force = [\NormalForce;\;\FrictionBasisForce] $ be an impulse generated by $\LCP{\LCPMatrix_{\Configuration}}{\LCPVector_{\Configuration}(\Velocity,\NormalForce[max])}$. Then either some contact $i$ activates fully ($\NormalForce[i] = \NormalForce[{max_i}]$), or all contacts terminate ($\Jn\Velocity'(\bar \Force) \geq \ZeroVector$).
\end{lemma}
\begin{proof}
Let $[\Configuration;\; \Velocity]$ be an impacting state ($\Velocity \in \ActiveSet ( \Configuration)$), and let $\NormalForce[max] > \ZeroVector$ be a normal impulse. Consider an impact LCP solution
$$\Brackets{\SlackForce;\;\NormalForce;\;\FrictionBasisForce;\;\SlackVelocity} \in \LCP{\LCPMatrix_{\Configuration}}{\LCPVector_{\Configuration}(\Velocity,\NormalForce[max])}\,.$$
such that
	\begin{equation}
		\NormalForce < \NormalForce[{max}]\,.
	\end{equation}
	Therefore for each contact $i$, the complementary equation \eqref{eq:ImpactSimulationSlackForceComplementairty} yields $\SlackForce[i] = 0$ as $\NormalForce[{max_i}] - \NormalForce[i] > 0$. Then from complementarity equation \eqref{eq:ImpactSimulationNormalForceComplementairty}, $\Jn[i]\Velocity' \geq 0$.
\end{proof}
	
\subsection{Proof of \Cref{lem:NetForceBoundedByIndividualForces} }
\label{adx:NetForceBoundedByIndividualForcesProof}
Consider a configuration $\Configuration \in \ConfigurationSet_A \setminus \ConfigurationSet_P$ and $\bar \Force = [\NormalForce;\; \FrictionBasisForce]$ obeying \eqref{eq:SimulationSlackVelocityComplementarity}. By construction, $\bar \J^T \bar \Force \in \LinearFrictionCone(\Configuration)$. Let
\begin{equation}
\mathcal F = \LinearFrictionCone \Parentheses{\Configuration} \cap \bar \J^T\Braces{[\NormalForce;\;\FrictionBasisForce] : \Norm{\NormalForce}_1 = 1}\,.
\end{equation}
As $\LinearFrictionCone \Parentheses{\Configuration}$ is a convex cone, $\vect r$ satisfies the claim if $\mathcal F \cdot \Mass^{-1}\vect r > 1$. 
As $\LinearFrictionCone \subseteq \FrictionCone$, by Assumption \ref{assump:nondegenerate}, $\ZeroVector \not \in \mathcal F$. $\mathcal F$ is compact, non-empty, and convex polyhedron.
Therefore, 
by \citet[Theorem 11.4]{Rockafellar1970} there exists $\tilde {\vect r}$ such that
\begin{align}
	\varepsilon = \min_{\vect F \in \mathcal F} \vect F \cdot \tilde {\vect r} > \max_{\vect F \in -\mathcal F} \vect F \cdot \tilde {\vect r} = - \varepsilon\,.
\end{align}
Setting $\vect r(\Configuration) = \frac{\Mass(\Configuration)\tilde {\vect r}}{\varepsilon}$ satisfies the claim.

\subsection{Proof of \Cref{thm:ExponentialSimulationDecay}}
\label{adx:ExponentialSimulationDecayProof}
Let $\Configuration_0 \in \ConfigurationSet_A \setminus \ConfigurationSet_P$ be a pre-impact configuration; let $\Velocity_0 \in \ActiveSet(\Configuration_0)$ be a pre-impact velocity; and let $h>0$ be a step size.
As each $\NormalForce[{max}]$ is selected from the uniform distribution over the $h$-width box, we have that
\begin{equation}
	c_p = \Expectation{\underset{i}{\min}\, \NormalForce[{max_i}]}[\NormalForce[max] \sim h \cdot p] = \frac{h}{m+1} \,.
\end{equation}
We assume WLOG that $p$ is supported on the interior of the unit box $(0,1)^\Contacts$, as the probability of being on the boundary is $0$.
Let $\sigma = \sigma_{min}\Parentheses{\Mass = \Mass(\Configuration_0)}$, and therefore $
	\sqrt{\sigma} \TwoNorm{\Velocity} \leq \Norm{\Velocity}_{\Mass}$.
Now, select $\vect r$ for $\Configuration_0$ as defined in \Cref{lem:NetForceBoundedByIndividualForces}. We will now show that the existence of $\vect r$ in conjunction with dissipation (\Cref{thm:TimesteppingDissipation}), allows us to create a useful sufficient condition for impact termination.

Consider any execution of Algorithm \ref{alg:ImpactSimulation} with initial state $[\Configuration_0;\;\Velocity_0]$, and let $\NormalForce[max]^k$, $\bar\Force^k = [\NormalForce^k;\;\FrictionBasisForce^k]$ and $\Velocity_k$ be the maximum normal impulse; selected impulse; and velocity computed on lines \ref{line:LambdaMaxSelection}--\ref{line:VelocityUpdate} on the $k$th iteration of the loop.
If the loop has not terminated after $K$ steps, then for all loop iterations $k\in \Braces{1,\dots,K}$, $\Velocity_k \in \ActiveSet(\Configuration_0)$. By \Cref{thm:TimesteppingDissipation} and \Cref{lem:NonzeroLambdamaxNonzeroLambda,lem:NetForceBoundedByIndividualForces}, we have that
\begin{align}
	\Norm{\Velocity_0}_{\Mass}  &\geq \Norm{\Velocity_K}_{\Mass}\,,\\
	&\geq \sqrt{\sigma} \TwoNorm{\Velocity_K}\,,\\
	&\geq \sqrt{\sigma} \frac{\vect r}{\TwoNorm{\vect r}} \cdot \Velocity_K \\
	&\geq \frac{\sqrt{\sigma}}{\TwoNorm{\vect r}} \Parentheses{\vect r \cdot \Velocity_0 +  \sum_{k=1}^K \Norm{ \NormalForce^k }_1}\,,\\
	&\geq  - \sqrt{\sigma} \TwoNorm{\Velocity_0} + \frac{\sqrt{\sigma}}{\TwoNorm{\vect r}} \sum_{k=1}^K \Norm{ \NormalForce^k }_1\,,\\
	&\geq - \Norm{\Velocity_0}_{\Mass}+ \sum_{k=1}^K\frac{\sqrt{\sigma}}{\TwoNorm{\vect r}} \underset{i}{\min}\, \NormalForce[max_i]^k \,.
\label{eq:KineticEnergyDecreasesWithNormalForce} 
\end{align}
For this inequality to hold, and thus for $\Velocity_K$ to remain in $\ActiveSet(\Configuration_0)$, it must be true that the summation in \eqref{eq:KineticEnergyDecreasesWithNormalForce} is no greater than $2\Norm{\Velocity_0}_{\Mass}$.
Therefore, termination of the impact within $K$ steps (i.e. $Z(h,\Configuration_0,\Velocity_0) \leq K$) is implied by $Z_K > c_Z\Norm{\Velocity_0}_{\Mass}$, where
\begin{align}
	c_Z &= \frac{2\TwoNorm{\vect r}}{\sqrt{\sigma}}\,,\\
	Z_K &= \sum_{k=1}^K \underset{i}{\min}\, \NormalForce[{max_i}]^k \,.
\end{align}
Given that the $\NormalForce[{max}] \sim h\cdot p$ are selected i.i.d. we have that $\Expectation{Z_K} = Kc_p$. Thus we would expect an impact to terminate proportional to
\begin{equation}
K^* = \left\lceil \frac{c_Z}{c_p} \right\rceil \left\lceil \Norm{\Velocity_0}_{\Mass} \right\rceil\,.
\end{equation}
We now bound the termination time $Z$ using Hoeffding's inequality, applied below in \eqref{eq:hoeffding}; for $k\in\Integer^+$ and $K = 2K^*+k$,
\begin{align}
	P\Parentheses{Z \geq K} & \leq P\Parentheses{Z_K \leq c_Z\Norm{\Velocity_0}_{\Mass}}\,,\\
	& \leq P\Parentheses{Z_K \leq K^*c_p}\,,\\
	&= P\Parentheses{Z_K - Kc_p \leq -\Parentheses{K^* + k}c_p}\,,\\
	&\leq \exp\Parentheses{-\frac{2}{K}\Parentheses{K^* + k}^2\frac{c_p^2}{h^2}} \,,\label{eq:hoeffding} \\
	&\leq \exp\Parentheses{-\Parentheses{K^* + k}\frac{c_p^2}{h^2}} \,,\\
	&\leq \exp\Parentheses{-\frac{k}{(m+1)^2}} \,.
\end{align}
Thus the claim is satisfied.

\subsection{Proof of \Cref{lem:ApproximateTermination}}\label{adx:ApproximateTerminationProof}
Suppose not. Then there exists a configuration $\Configuration \in \ConfigurationSet_A \setminus \ConfigurationSet_P$, velocity $\Velocity$, and $\varepsilon > 0$, such that for all $N \in \Natural$, there exists a $\Velocity_N$, $\Jn\Velocity_N \geq -\frac{1}{N}$, $\Norm{\Velocity_N}_{\Mass} \leq \Norm{\Velocity}_{\Mass}$, and yet $\Velocity_N' = \vect f_{\Configuration}(\Velocity_N,\varepsilon\OneVector ) \in \ActiveSet (\Configuration)$.

Due to energy dissipation (\Cref{thm:TimesteppingDissipation}) and the boundedness of $\Velocity_N$, the sequence $\Velocity_N'$ is bounded as well.
Without loss of generality we can therefore assume that $\Velocity_N\to \Velocity_\infty$ and $\Velocity_N'\to \Velocity_\infty'$.
As $\Jn\Velocity_N \geq -\frac{1}{N}$, it must be that $\Jn \Velocity_\infty \geq 0$.
Therefore, $\Velocity_\infty' = \vect f_{\Configuration}(\Velocity_\infty,\varepsilon\OneVector_\Contacts) = \Velocity_\infty$ via \Cref{lem:VelocityStationaryConditions}.
As $\Velocity_N$ and $\Velocity_N'$ converge to each other, there exists an $N^*$, with LCP-selected force $\bar \Force_{N^*} = [\NormalForce;\; \FrictionBasisForce]$ such that
	\begin{equation}
		\TwoNorm{\Parentheses{\Velocity_{N^*}' - \Velocity_{N^*}}} = \TwoNorm{\Mass^{-1} \bar \J^T \bar \Force_{N^*}} < \frac{\varepsilon}{\TwoNorm{\vect r(\Configuration)}}\,,
	\end{equation}
	where $\vect r(\Configuration)$ comes from \Cref{lem:NetForceBoundedByIndividualForces}.
	However, by \Cref{lem:NonzeroLambdamaxNonzeroLambda}, as $\Velocity_{N^*}' \in \ActiveSet (\Configuration)$, at least one contact must fully activate, and thus $\Norm{\NormalForce}_1 \geq \varepsilon$. But then again by \Cref{lem:NetForceBoundedByIndividualForces}, $ \TwoNorm{\Mass^{-1}\bar \J^T \bar \Force_{N^*}} \geq \frac{\varepsilon}{\TwoNorm{\vect r(\Configuration)}}$, a contradiction.
\subsection{Proof of \Cref{thm:VelocitySetApproximation}}\label{adx:VelocitySetApproximationProof}
First we show that generating an $\varepsilon$-net of $\VelocitySet_\infty(\State_0,h) \setminus \ActiveSet(\Configuration_0)$ can be reduced to generating an $\varepsilon'$-net of $\VelocitySet_N(\State_0,h)$ for a suitable $(\varepsilon',N)$. We then show that $\VelocitySet_N(\State_0,h)$ is the image of a box under a Lipschitz function, and apply \Cref{prop:epsilonnet}.

Select an initial $\State_0 = [\Configuration_0;\;\Velocity_0] \in (\ConfigurationSet_A \setminus \ConfigurationSet_P) \times \VelocitySpace$; step size $h>0$; and constants $\varepsilon,\delta > 0$. Define $\psi$ as on line \ref{line:psi} of Alg. \ref{alg:Approximate}. Select
\begin{equation}
	\varepsilon' = \min\Parentheses{\frac{\varepsilon}{3},\frac{\delta\Parentheses{\frac{\varepsilon}{3\psi},\Velocity_0}}{2\sigma_{max}(\Jn)}}\,,\label{eq:EpsilonPrime}
\end{equation}
where $\delta\Parentheses{\frac{\varepsilon}{3\psi},\Velocity_0}$ comes from \Cref{lem:ApproximateTermination}.
Via \Cref{lem:VPlusNConvergence}, select $N$ such that $\VelocitySet_N(\State_0,h)$ is an $\varepsilon'$-net of $\VelocitySet_\infty(\State_0,h)$.
Consider a run of $\mathrm{Approximate}(h,\State_0,\epsilon,N,M)$ for some $M>0$.
Suppose that the $M$ samples generated on line \ref{line:ApproximationSimulation} of Alg. \ref{alg:Approximate} constitute a $\varepsilon'$ net of $\VelocitySet_N(\State_0,h)$.
Consider a post-impact velocity $\Velocity_1 \in \VelocitySet_\infty(\State_0,h) \setminus \ActiveSet(\Configuration_0)$. Then there exists a $\Velocity_2 \in \VelocitySet_N(\State_0,h)$ with $\TwoNorm{\Velocity_1 - \Velocity_2} < \varepsilon'$.
Pick the closest $\Velocity_3$
to $\Velocity_2$ selected on line \ref{line:ApproximationSimulation} of Alg. \ref{alg:Approximate}.
From \eqref{eq:EpsilonPrime}, we know that $\TwoNorm{\Velocity_3-\Velocity_2}\leq \frac{\varepsilon}{3}$, and $\Jn\Velocity_3 \geq -\delta\Parentheses{\frac{\varepsilon}{3\psi},\Velocity_0}$.
$\Velocity_3$ is used to generate $\Velocity_4 = \vect f_{\Configuration_0}(\Velocity_3,\frac{\varepsilon}{3\psi}\OneVector_\Contacts)$ on line \ref{line:FinalStep}, and $\Jn\Velocity_4 \geq \ZeroVector$ via \Cref{lem:ApproximateTermination}.

$\Velocity_4$ is thus in the post-impact set $\VelocitySet_\infty(\State_0,h) \setminus \ActiveSet(\Configuration_0)$ and is output by $\mathrm{Approximate}(h,\State_0,\epsilon,N,M)$.
Suppose that $\bar \Force = [\NormalForce;\;\FrictionBasisForce]$ was the LCP-selected force in the calculation of $\Velocity_4$; we then have that $\TwoNorm{\Velocity_4-\Velocity_3}
	\leq \frac{\varepsilon}{3}$ by construction of $\psi$ on line \ref{line:psi} of Algorithm \ref{alg:Approximate}.
Thus, $\TwoNorm{\Velocity_4-\Velocity_1}$ is smaller than
\begin{equation}
\TwoNorm{\Velocity_2-\Velocity_1} + \TwoNorm{\Velocity_3-\Velocity_2} + \TwoNorm{\Velocity_4-\Velocity_3} \leq\varepsilon \,.
\end{equation}

Therefore, the claim is true if the samples from $\VelocitySet_N(\State_0,h)$ generated on line \ref{line:ApproximationSimulation} of Algorithm \ref{alg:Approximate} are a $\varepsilon'$ net of $\VelocitySet_N(\State_0,h)$ with probability $1 -\delta$; we now calculate a $M$ that guarantees this property.

Consider the sequence of functions
\begin{align}
	{\vect f}^1(\NormalForce^1) &= \vect f_{\Configuration_0}(\Velocity_0,\NormalForce^1)\,,\\
	{\vect f}^k(\NormalForce^1,\dots,\NormalForce^k) &= \vect f_{\Configuration_0}({\vect f}^{k-1}(\NormalForce^1,\dots,\NormalForce^{k-1}),\NormalForce)\,.
\end{align}
Examining \eqref{eq:ReachabilityStep}, we see that
\begin{equation}
	\VelocitySet_N(\State_0,h) = {\vect f}^N([0,h]^{N\Contacts})\,.
\end{equation}
Furthermore, if $f_{\Configuration_0}$ has Lipschitz constant $L$, then ${\vect f}^N$ is Lipschitz with constant no more than $L^N$ by the composition rule for Lipschitz functions.
Under Assumption \ref{assump:uniqueLCPoutcome}, $\mathrm{Sim}(h,\State_0,N)$ yields a uniform sample of $[0,h]^{N\Contacts}$ mapped under ${\vect f}^N$. Therefore, the claim holds, with $M$ given by \Cref{prop:epsilonnet}:
\begin{align}
	M &\geq \frac{\ln(\delta\Omega)}{\ln(1-\Omega)}\,, &
	\Omega &= \left\lceil \frac{hL^N\sqrt{N\Contacts}}{\varepsilon'} \right\rceil^{-N\Contacts}\,.
\end{align}

\end{document}